\documentclass[12pt,a4paper]{article}
\usepackage[utf8]{inputenc}
\usepackage[T1]{fontenc}
\usepackage{graphicx}
\usepackage{amsmath, amsthm, amsfonts, amssymb}
\usepackage{comment}
\usepackage{mathrsfs}  
\usepackage[a4paper, margin=2.7cm]{geometry}
\usepackage[colorlinks=true,linkcolor=blue,citecolor=blue,urlcolor=blue,breaklinks]{hyperref}
\usepackage{dsfont}
\usepackage{caption}
\usepackage[round]{natbib}
\usepackage{cleveref}
\usepackage{stmaryrd}
\usepackage{algorithm} 
\usepackage{algpseudocode} 
\usepackage{framed}
\usepackage{lipsum}
\usepackage{color}
\usepackage[toc,page]{appendix}
\usepackage{alltt}
\usepackage{multirow}
\usepackage{subcaption}
\usepackage{caption}
\usepackage{float} 
\usepackage{booktabs}
\usepackage[export]{adjustbox}

\newcommand{\depth}{k}

\def\R{\mathbb{R}}

\def\E{\mathbb{E}}

\def\U{\mathbb{U}}
\def\P{\mathbb{P}}
\def\v{{\rm Var}}
\def\cov{{\rm Cov}}

\newtheorem{thm}{Theorem}[section]

\newtheorem{cor}[thm]{Corollary}
\newtheorem{lem}[thm]{Lemma}
\newtheorem{prp}[thm]{Proposition}
\newtheorem{rem}[thm]{Remark}

\theoremstyle{definition}
\newtheorem{dfn}[thm]{Definition}
\theoremstyle{example}

\newcommand{\1}{\mathds{1}}
\newcommand{\bm}[1]{\mathbf{#1}}

\usepackage{snaptodo}
\snaptodoset{block rise=2em}
\snaptodoset{margin block/.style={font=\tiny}}

\title{Asymptotic Normality of Infinite Centered Random Forests - Application to Imbalanced Classification}

\author{Moria Mayala, Erwan Scornet,  Charles Tillier and  Olivier Wintenberger }

\date{}
\AtEndDocument{\bigskip{\footnotesize
		\noindent
		\textsc{Sorbonne Université, CNRS, Laboratoire de Probabilités, Statistique et Modélisation
			F-$75005$ Paris, France}
		\textit{E-mail address}: \texttt{\href{mailto:email@example.com}{moria.mayala@sorbonne-universite.fr}}
		\par
		\addvspace{\medskipamount}
		\noindent
		\textsc{Sorbonne Université, CNRS, Laboratoire de Probabilités, Statistique et Modélisation
			F-$75005$ Paris, France}
		\textit{E-mail address}: \texttt{\href{mailto:olivier.wintenberger@sorbonne-universite.fr}{olivier.wintenberger@sorbonne-universite.fr}}
		\par
		\addvspace{\medskipamount}
		\noindent
		
		\textsc{Université Paris-Saclay, Laboratoire de Mathématiques de Versailles, 78000, Versailles.}
		\textit{E-mail address}: \texttt{\href{mailto:charles.tillier@gmail.com}{charles.tillier@gmail.com}}
}}

\begin{document}
	\maketitle
	\begin{abstract}
 Many classification tasks involve imbalanced data, in which a class is largely underrepresented. Several techniques consists in creating a rebalanced dataset on which a classifier is trained. In this paper, we study theoretically such a procedure, when the classifier is a Centered Random Forests (CRF).
We establish a Central Limit Theorem (CLT) on the infinite CRF with explicit rates and exact constant. We then prove that the CRF trained on the rebalanced dataset exhibits a bias, which can be removed with appropriate techniques. Based on an importance sampling (IS) approach, the resulting debiased estimator, called IS-ICRF, satisfies a CLT centered at the prediction function value. For high imbalance settings, we prove that the IS-ICRF estimator enjoys a variance reduction compared to the ICRF trained on the original data. Therefore, our theoretical analysis highlights the benefits of training random forests on a rebalanced dataset (followed by a debiasing procedure) compared to using the original data. 
Our theoretical results, especially the variance rates and the variance reduction, appear to be valid for Breiman's random forests in our experiments.\\
		
	\end{abstract}

	\section{Introduction} 
	\subsection{Motivation and context}
In binary classification, the target variable \( Y \in \{0,1\} \) is predicted based on a covariate \( X \) taking values in the feature space \( \mathcal{X} \) and the objective is to estimate accurately the regression function 
\begin{align}\label{eq:mu_intro}
\mu (\bm{x}) := \mathbb{P}(Y=1 | X=\bm{x}), \quad \bm{x} \in \mathcal{X}.
\end{align}
Class imbalance occurs when one class has substantially more instances than the other, typically referred to as the majority and minority classes, respectively. Under such conditions, standard machine learning algorithms, including decision trees, neural networks, Bayesian networks, k-nearest neighbors, and support vector machines, often exhibit poor performance when predicting minority class instances \citep[][]{krawczyk2016learning}. 

In many real-world applications, the minority class corresponds to rare but critical events. Although collected datasets continue to grow in size, these rare events remain severely underrepresented, thereby exacerbating the challenges associated with class imbalance. Common examples include predicting manufacturing equipment failures \citep[][]{lee2021early}, detecting spam emails \citep[][]{liu2017addressing}, identifying fraudulent credit card transactions \citep[][]{dhankhad2018supervised}, diagnosing rare medical conditions \citep[][]{mena2006machine}, and detecting cyber-attacks \citep[][]{okutan2018forecasting}. 

Introduced in  \cite{breiman2001random}, Random Forests (RF) are a powerful and widely used ensemble learning method that combine multiple decision trees, where each tree is trained on a randomly sampled subset of the data. 
Among their key advantages are their ability to handle high-dimensional data, capture complex patterns, provide feature importance and handle a variety of supervised learning tasks. Due to their versatility and strong performance across various domains, random forests have become a fundamental tool in machine learning applications \citep[see, e.g.,][]{fernandez2014we}. 
There exist several Central Limit Theorems (CLT) in the literature that have been derived for random forests, drawing from the connection between random forests and U-statistics \citep[][]{lee2019u,hoeffding1948class}. In this line of work, to facilitate statistical analysis, bootstrap is often replaced by subsampling  without replacement, and the number of trees in the forest is assumed to tend to infinity, thus leading to the 
\emph{subsampling Infinite Random Forest} estimator, defined from i.i.d.\ training data \( \bm Z_n = (Z_1, \ldots, Z_n) \) by
\begin{align}\label{eq:intro}
	\widehat{\mu}^{s}(\bm x):= \binom{n}{s}^{-1} \sum_{S \subset \{ 1,\ldots,n\}, |S|=s} \mathbb{E}[T^s \mid \bm Z_S], \qquad \bm x \in \mathcal{X},
\end{align}
where \( T^s := T^s(\bm{x}; \U; \bm Z_S) \) denotes the prediction at point $\bm x$ of a tree trained on a subsample \( \bm Z_S \), drawn uniformly without replacement from \( \bm Z_n \), and where \( \U \) governs the randomization of the tree structure. Note that $S \subset \{ 1,\ldots,n\}$ is a random subset of indices of cardinality $s$, corresponding to the observations selected uniformly without replacement during the subsampling step. Within this framework, most existing CLTs in the literature do not provide explicit convergence rates. This limitation stems from their dependence on the quantity
\[
V_1^s = \v\big( \mathbb{E}[T^s \mid Z_1] - \mathbb{E}[T^s] \big),
\]
whose explicit characterization is challenging for Breiman's random forests due to their intricate construction.

\citet{wager2018estimation} and \citet{mentch2016quantifying} were the first to use the connection between random forests estimates and U-statistics to establish central limit theorem. CLT established in \citet{mentch2016quantifying} hold for different asymptotic regimes in the number of trees with a subsample size $s=o(\sqrt{n})$. Similar results have been established for subsampling with replacement \citep[see][]{zhou2021v}.
\citet{wager2018estimation} modify the construction of Breiman's forests by considering honest trees (for which typically one part of the data set is used to construct the trees and the other part to estimate the prediction inside each cell) with a positive probability of splitting along each direction and enforcing that at least a fraction $\alpha$ of observations fall into both children for each split. Assuming that the subsample size $s = n^\beta$, where $\beta$ tends to one when the number of variables tends to infinity, \citet{wager2018estimation} prove a CLT centered at the regression function for the modified forest. However, the rate of consistency depends on the forest variance, which is not explicit.
\citet{peng2022rates} investigate the rate of convergence of CLT for random forests and establish upper bound on the convergence rate of forests whose tree construction is independent of the labels and whose number of observations per leaf is controlled. 

In this work, we adopt a different perspective by analyzing a simplified random forest model, namely the Centered Random Forests (CRF) introduced by \citet{Breiman2004CONSISTENCYFA} (see also \citet{biau2012analysis}). More specifically, we focus on subsampling infinite CRFs (ICRF) as defined in \Cref{eq:intro} where individual trees \( T^s \) are constructed as follows: at each node, a feature is selected uniformly at random, and the split occurs at the midpoint along the chosen feature. Compared to Breiman's original forests, CRFs are more amenable to theoretical analysis since the tree construction is independent of both the covariates and the target variables. We refer the reader to the related work section below for a detailed review of CRFs.

The remainder of the paper is organized as follows. We conclude this introductory section by outlining our main contributions and briefly reviewing the relevant literature in \Cref{subsec:contribution} and \Cref{subsec:relatedwork}, respectively. In \Cref{section:3}, we formalize the problem setting, introduce the centered random forest model, and establish the asymptotic normality of the subsampling ICRF estimator in \Cref{th2}. Building on additional assumptions regarding the random forest hyperparameters, we derive in \Cref{th: TCL} a CLT centered at the target quantity $\mu(\bm{x})$. \Cref{sec: underICRF} focuses on ICRFs applied to imbalanced classification problems. We first define the rebalanced data setting, and then derive a CLT for the ICRF trained on this dataset in \Cref{th: TCLB}, along with a corresponding result for its debiased version in \Cref{cor: TCLB}. This section concludes with a comparison of the asymptotic variances of ICRF and its debiased version in \Cref{thm_comparison_variances}. Numerical experiments illustrating our theoretical results are presented in \Cref{sec:numerical}. The proofs of \Cref{th2} and \Cref{th: TCL} are provided in \Cref{sec:Proof_th_TCL}, while those of \Cref{th: TCLB} and \Cref{cor: TCLB} are deferred to \Cref{sec: proof_TCLB}. The proof of \Cref{thm_comparison_variances} is given in \Cref{sec:proof_th_compare}.

\subsection{Main contributions}\label{subsec:contribution} 

We start by establishing the asymptotic normality of the subsampling ICRF estimator under mild assumptions, specifically assuming that \( s^2/n = o(2^{\depth} \depth^{-(d-1)/2}) \), where \( s \) denotes the subsample size, \( n \) the training sample size, \( d \) the dimension of the covariate \( X \), and \( \depth \) the tree depth. In this first CLT, we obtain an explicit convergence rate of \( O\left( n^{1/2} \depth^{(d-1)/4} / 2^{\depth/2} \right) \). However, the subsampling ICRF estimator may be biased, since the centering term does not necessarily coincide with the true value of the regression function \( \mu \).

Under stronger conditions on the hyperparameters, specifically by choosing the tree depth to satisfy \( 2^{\depth} \depth^{- \frac{d-1}{2 }} n^{- \frac{d \log 2}{1 + d \log 2}} \to \infty \), we show that the bias vanishes asymptotically. In this regime, the subsampling ICRF estimator becomes an asymptotically unbiased estimator of \( \mu(\bm{x}) \), assumed to be a Lipschitz continuous function over $\mathcal X$. To the best of our knowledge, this is the first result establishing a CLT for random forests with an explicit convergence rate (including explicit constants) under clearly stated assumptions on the forest parameters.

Then, to address class imbalance, we consider an ICRF estimator trained on a rebalanced dataset obtained through a sampling procedure. Under this setting, we establish a second CLT for the rebalanced ICRF estimator. However, due to the distributional shift introduced by rebalancing, this estimator is biased with respect to the original regression function \( \mu \). To correct for this bias, we apply a debiasing procedure proposed by \citet{mayala2024infinite} and establish a third CLT for the resulting debiased estimator, called IS-ICRF. The main difficulty across these three CLTs lies in obtaining an explicit expression for the variance term \( V_1^s \) and verifying the associated Lindeberg condition \( n V_1^s \to \infty \) as $n\to \infty$. For high imbalance settings, we prove that the IS-ICRF estimator enjoys a variance reduction compared to the ICRF trained on the original data. Therefore, our theoretical analysis highlights the benefits of training random forests on a rebalanced data set (followed by a debiasing procedure) compared to using the original data.  
 
\subsection{Related Work}\label{subsec:relatedwork}
\paragraph{Centered Random Forests.} Several works study Centered Random Forests (CRF), which were introduced by \cite{Breiman2004CONSISTENCYFA} and rigorously analyzed in  an influential paper of \citet{biau2012analysis}. These works primarily examine the mean squared prediction error of CRF but do not address its asymptotic normality. For example, \citet{biau2012analysis} demonstrated a mean squared prediction error of  $O(n^{- 1/(d (4/3)\log 2 + 1)})$. \citet{scornet2016random} modified the RF definition to create Centered Kernel Random Forests,  where trees are grown according to the same selection and splitting procedure as CRF, retrieving the same bound as \citet{biau2012analysis}. Recently,  \cite{klusowski2021sharp} improved the error bound of CRF  to  the order
 $O(n^{- (1+\delta_d)/(d\log 2 + 1)})$, up to a $\log n$ factor, where $\lim_{d \to \infty} \delta_d =0$. The interested reader can refer to \citet{biau2016random} and \citet{scornet2025theory} for a more comprehensive understanding of random forests, from a theoretical perspective.

 \paragraph{Imbalanced learning.} Several strategies have been proposed to address the challenges posed by imbalanced data, ranging from algorithm-level to data-level approaches. Algorithm-level methods involve modifying the learning process to handle class imbalance. This includes assigning higher misclassification costs to the minority class (cost-sensitive learning), adjusting decision thresholds, or using specialized models such as ensemble methods (e.g., boosting techniques like AdaBoost) that focus on difficult-to-classify minority class samples. In contrast, data-level approaches focus on modifying the dataset itself to balance class distribution. Standard approaches include oversampling (repeatedly drawing  samples from the minority class) or undersampling (subsampling the majority class) so that both classes have the same number of elements. Both methods have drawbacks: oversampling can cause overfitting, while undersampling may lose valuable majority class information. 
One of the most famous and commonly used method is SMOTE \citep[Synthetic Minority Over-sampling Technique][]{chawla2002smote} that generates synthetic samples for the minority class  by interpolating between existing observations. For a comprehensive review of methods used in imbalanced classification, we refer to  \cite{ramyachitra2014imbalanced,krawczyk2016learning}. 

\paragraph{Imbalanced learning with random forests.} Many studies address imbalanced classification using RFs. \citet{chen2004using} propose two ways to deal with imbalanced data. For the first one, each tree is built by  bootstrapping the minority class and subsampling the same number of samples from the majority class, a method called Balanced Random Forests (BRFs). For the second one,  cost-sensitive learning is applied by reweighting each sample according to its class probability. \citet{o2019random} propose a density-based Random Forests quantile classifier, showing a connection to BRFs despite not being a data-level method. Additionally, \citet{lee2022downsampling} explore downsampling with active learning to select the most informative samples, mitigating class imbalance effects. Other recent approaches involving deep learning generative techniques such as  CTGAN \citep[][]{xu2019modeling}, transformers \citep[][]{hegselmann2023tabllm} or diffusions \citep[][]{kotelnikov2023tabddpm} have been proposed to solve class imbalance problems. However, they have not yet achieved state-of-the-art performances \citep[see e.g.][]{dsouza2025synthetictabulardatageneration} and often lack theoretical guarantees. 
\newline
 
\noindent \textbf{Notation}: The indicator function of the set $A$ is denoted by $\1\{A\}$ and its cardinality 
is denoted by $|A|$. 
We denote by $\log$ the natural logarithm. For any vector $z \in \mathds{R}^n$ and any subset $J \subset \{1, \hdots, n\}$, $z_J$ denotes the subvector of $z$ composed of the components of $z$ indexed by $J$. We denote by $\mathcal{N}(\mu, \sigma^2)$ the Gaussian distribution with mean $\mu$ and variance $\sigma^2$,  $\Gamma(a)$ the gamma function evaluated at $a$, where $a$ is  a positive constant, $ U([a,b])$ the Uniform distribution on $[a,b]$, $a<b$. 

\section{Infinite Centered Random Forests (ICRF)}\label{section:3}

\subsection{Preliminaries}\label{subsec:Notation}
Let \( Z_i := (X_i, Y_i) \) be  independent and identically (i.i.d. for short) copies of a random pair \( (X, Y) \), where \( X \in \mathcal{X}=[0,1]^d \) (with \( d \geq 2 \)) denotes the covariate vector used to predict the binary response variable \( Y \in \{0,1\} \). We consider a binary classification setting in which the class distribution is imbalanced: the label \( 0 \) corresponds to the majority class, and the imbalance ratio is defined as \( \textup{IR} = n_0 / n_1 \gg 1 \), where \( n_0 \) and \( n_1 \) denote the number of observations in class \( 0 \) and class \( 1 \), respectively. Given a training sample \( \mathbf{Z}_n = (Z_1, \ldots, Z_n) \), this paper focuses on the estimation of the \textit{regression function}
\begin{equation}\label{eq:px}
    \mu (\bm{x}):=\mathbb{P}(Y=1 | X=\bm{x}) ,\, \qquad \bm{x}\in[0,1]^d,
    \end{equation} 
    using random forests. 

\citet{breiman2001random} Random Forests (RF) are the most widely used RF algorithms. Unfortunately, they are notoriously difficult to analyze theoretically notably due to their tree construction which involves a data-dependent splitting criterion. In order to establish CLT with explicit rates of convergence, we consider instead the Centered Random Forests (CRF). CRF were originally proposed by \citet{Breiman2004CONSISTENCYFA} and later studied by \citet{biau2012analysis} and \citet{klusowski2021sharp}. They fall within the broader class of \textit{Purely Random Forests}, which do not rely on the dataset for building individual trees.  Centered Random Forests are built from Centered Random Trees (CRT), whose construction is detailed below. 

\vspace{0.5cm}
\noindent \textbf{Centered Random Trees algorithm:}
\vspace{0.3cm}
\begin{enumerate}
    \item[(i)] Select uniformly at random without replacement $s$ observations among the $n$ original ones. Denote by $S \subset \{1, \hdots, n\}$ the corresponding set of indices.
    \item[(ii)] Initialize with $[0,1]^d $ as the root.
     \item[(iii)] At each node, a coordinate among $\{1, \hdots, d\}$ is chosen  uniformly at random.
     \item[(iv)] Split the node at the midpoint of the interval along the direction of the selected coordinate in step (iii).
     \item[(v)] Repeat steps (iii) and (iv) for the two daughter nodes until each node has been split exactly $\depth$ times, where $\depth$ is an integer. 
     \item[(vi)] The prediction of the tree at a point ${\bf x}$ is computed as the average of the labels of the points that are selected in step (i) \textit{and} that fall in the leaf containing ${\bf x}$. This prediction is denoted $T^s(\bm{x}; \U; \bm Z_S)$ where $\U$ stands for the randomization of the eligible direction for splitting (step (iii)). 
\end{enumerate}

\begin{rem}\label{rem:trees}
 Since step (v) of the CRT algorithm is repeated $\depth$ times, the resulting tree contains $2^{\depth}$ leaves (terminal nodes). 
 \end{rem} 

\begin{rem}
Note that the CRT algorithm defined above relies on subsampling without replacement, rather than the standard bootstrap (sampling with replacement), to facilitate statistical analysis for several reasons. First, subsampling reduces dependencies between trees, as each subset is drawn independently, whereas bootstrap samples induce correlations through repeated observations. Thus, variance decomposition becomes more tractable. Besides, the i.i.d. nature of the training samples resulting from subsampling eases the application of classical limit theorems, such as CLT.  
\end{rem}

We now formalize the prediction of a CRT. Let \( S = \{i_1, \hdots, i_s\} \subset \{1, \ldots, n\} \) denote the random set of indices selected during sampling step (i) of the CRT algorithm, with \( |S| = s \), and define \( \bm Z_S := (Z_{i_1}, \ldots, Z_{i_s}) \) and \( \bm X_S := (X_{i_1}, \ldots, X_{i_s}) \) as the corresponding subsamples and covariate vectors, respectively.
 The CRT prediction at a point \( \bm{x} \in [0,1]^d \) based on the subsample \( \bm Z_S \), denoted by \( T^s(\bm{x}; \U; \bm Z_S) \), is given by

\begin{equation}\label{eq:T_ICRF}
		T^s(\bm{x}; \U; \bm Z_S):= \frac{\sum_{i \in S} Y_{i}\1{\{X_{i} \in L_{\U}(\bm{x}) \}}}{N_{L_{\U}(\bm{x})}( \bm X_S)} , 
	\end{equation}
where \( L_\U(\bm{x}) \) is the leaf of the tree built with randomness \( \U \) that contains \( \bm{x} \in [0,1]^d \) and 
	$N_{L_{\U}(\bm{x})}(\bm X_{S}):= 
	\sum_{i \in S} \1{\{X_{i} \in L_{\U}(\bm{x}) \}} \,$ 
is the number of selected observations (indexed by $S$) falling into the leaf $L_{\U}(\bm{x})$. By convention,  the prediction of the tree $T^s(\bm{x}; \U; \bm Z_S)$ is arbitrarily set to zero when there is no observation falling into the leaf $L_{\U}(\bm{x})$. The prediction of $B$ different trees built with subsamples indexed by $S_1, \hdots, S_B$ and split randomization $\U_1, \hdots, \U_B$ can then be averaged to form the forest prediction:
\begin{align}\label{eq: random_sub}
	\widehat{\mu}_{B,s}( \bm{x}; \mathbf{Z}_n):= \frac{1}{B} \sum_{b=1}^B T^s(\bm{x}; \U_{b};\bm Z_{S_b}),  \qquad  \bm{x} \in [0,1]^d.
\end{align}

We are now ready to define the Infinite Centered Random Forests (ICRF) estimator. 
\begin{dfn}[Infinite Centered Random Forests (ICRF)]\label{def:ICRF}
The ICRF at point $\bm x\in [0,1]^d$ is defined as 	\begin{align}
 \label{def:subs_classif}
		\widehat{\mu}_{s}^{\textup{ICRF}}(\bm x):= \binom{n}{s}^{-1} \sum_{S \subset \{1, \hdots, n\}, |S|=s} \mathbb{E}[T^s(\bm{x}; \U;\bm Z_S) \mid \bm Z_S]\,,
	\end{align}
 where $T^s(\bm{x}; \U; \bm Z_S)$ are the tree predictions at point  $\bm x\in [0,1]^d$, as defined in  \Cref{eq:T_ICRF}. 
 \end{dfn}

The term \textit{Infinite} refers to the theoretical setting where the number of trees \( B \) in \Cref{eq: random_sub} tends to infinity.
By the Law of Large Numbers, the forest prediction converges to its expectation, removing the additional Monte Carlo variability.

In this setting, the infinite forest takes the form described in Equation~\eqref{def:subs_classif} (see \citet{mayala2024infinite} for more details), where the summation extends over all subsets \( S \subset \{1, \hdots, n\} \) of size \( |S| = s \), thus turning into 
a complete \textit{U-statistic}, for which convenient tools such as the Hoeffding decomposition and exchangeability can be applied. Although real-world implementations rely on a finite number of trees, the \( B \to \infty \) regime provides a foundational framework for understanding the statistical properties of such RF. We refer to \cite{scornet2016asymptotics, wager2018estimation, peng2022rates} for more details on infinite random forests.
 
\subsection{Asymptotic normality of ICRF}\label{subsec:ICRF}

In the sequel, we need the following three conditions.

\begin{enumerate}
\item [\textbf{(H0)}] \label{hyp3}   \textbf{Covariate Condition:} The covariate vector $X$ is uniformly distributed on $[0,1]^d$.  

\item [\textbf{(H1)}] \label{hyp1} \textbf{Smoothness Condition:}
The regression function $\mu$ is a $L$-lipschitz function with respect to  the max norm.
\item [\textbf{(G1)}] \textbf{Tree complexity Condition:} The subsample size $s$ and the tree depth $\depth$ tend to infinity and satisfy $s/(\depth 2^{\depth})\to \infty$,\, as $n\to \infty$.
\end{enumerate}

The next theorem establishes a CLT for the ICRF estimator, when the sample size $n$, the subsample size $s$ and the tree depth $\depth$ tend to infinity. 

\begin{thm}\label{th2}
	Let  $\bm x\in [0,1]^d$, $d\ge 2$ and $\widehat{\mu}_{s}^{\textup{ICRF}}({\bf x})$ be the ICRF estimator at point  $\bm x$ (see \Cref{def:ICRF}). Assume \textbf{(H0)}, \textbf{(H1)} and \textbf{(G1)} hold and  
	\begin{align}
	    \frac{ n 2^{\depth}}{s^2 \depth^{(d-1)/2}} \to \infty, \label{condition_th1}
	\end{align}
as $n,s, \depth \to \infty$.  Then, for $\bm x\in [0,1]^d$ we have
	$$
\sqrt{\dfrac {n \depth^{(d-1)/2}}{2^{\depth}}} \left( \widehat{\mu}_{s}^{\textup{ICRF}}(\bm x)-\E[\widehat{\mu}_{s}^{\textup{ICRF}}(\bm x)]\right)\overset{d}{\underset{n \to \infty}{\longrightarrow}} \mathcal N(0,C(d)\mu(\bm x)(1-\mu(\bm x)))\,,
$$
with
\begin{align}
C(d)&=\dfrac{2\Gamma(d-1)}{(\log 2)^{d-1}\Gamma((d-1)/2)}\E\left[\left( \frac{\|(\bm{N}-\overline{\bm{N}} \1)\|_2}{\|(\bm{N}-\overline{\bm{N}} \1)\|_1}\right)^{d-1} \right]\,,
\end{align}
where $\bm N = ( N_1, \hdots,  N_d)$ with $ N_1, \hdots,  N_d$ independent Gaussian random variables $\mathcal{N}(0,1)$ and $\overline{\bm N} = (1/d) \sum_{j=1}^d  N_j$. 
\end{thm}  
 The proof of \Cref{th2} is provided in  \Cref{sec:Proof_th_TCL}. To the best of our knowledge,  \Cref{th2} is the first result to establish a CLT for random forests, with an exact rate of convergence, an expression of the asymptotic variance and an explicit condition on the sample size, subsample size and tree depth. 
 Note that the constant $C(d)$ is made explicit in \Cref{th2} and upper and lower bounds are provided in Equation~\eqref{eq_encadrement_Cd} in \Cref{sec:Proof_th_TCL}.
 
 Condition \eqref{condition_th1} imposes that the subsample size is negligible compared to the sample size $n$. Indeed, in order to avoid obtaining empty cells, we usually set  the tree depth $k$ so that  $2^{\depth} \leq s$. In this case, in order to satisfy \eqref{condition_th1}, it is necessary that $s=o(n)$. 
 
 However, one can choose $k = \log_2 (s) - 2 \log_2 ( \log_2(s))$,
which satisfies {\bf (G1)}. In this case, \eqref{condition_th1} turns into $s (\log_2(s))^{(d+3)/2} = o(n)$, which allows to choose a  large subsample size $s$. Nevertheless, such a choice may damage the rate of convergence of the CLT which is $O(n (\log_2(s))^{(d+3)/2}/s)$. Note that results for large subsample size were also obtained by \citet{wager2018estimation} for a modified Breiman forest construction or by \citet{peng2022rates} for generic estimators.

Assumption \textbf{(H0)} is quite common in the random forest literature \citep[see, e.g.,][]{biau2010layered,genuer2012variance, wager2018estimation}. 

Assumption \textbf{(H1)}, which postulates that the regression function $\mu$ is Lipschitz,  is standard in the nonparametric statistics literature, and particularly in the theoretical analysis of random forests  \citep[see, e.g.,][]{wager2018estimation, klusowski2021sharp}. This assumption,  allows us to control the variation of the regression function inside any cell, by controlling the cell diameter. Thus, bounds on cell diameter can be used to obtain a fine control on the bias and variance of the random forest estimate. The cell diameter decreases with tree depth, but increasing tree depth increases the tree variance, as it decreases the number of observations per leaf. Thus, 
Assumption \textbf{(G1)} prevents the number of leaves $2^{\depth}$ from growing too fast compared to the subsample size \( s \), or equivalently ensures that all leaves contain on average a large number of observations. In this setting, the variance of the tree-based estimator can be controlled.

It is worth noting that the CLT stated in \Cref{th2} is centered not at the target quantity \( \mu(\bm{x}) \), but at the expectation \( \mathbb{E}[\widehat{\mu}_{s}^{\textup{ICRF}}(\bm{x})] \). Consequently, it characterizes the fluctuations of the forest prediction around its mean, without accounting for the potential bias between \( \mathbb{E}[\widehat{\mu}_{s}^{\textup{ICRF}}(\bm{x})] \) and \( \mu(\bm{x}) \). Thus, asymptotic confidence interval for $\mu({\bf x})$ cannot be built using \Cref{th2}. 
With an additional condition on tree depth,  \Cref{th: TCL} below establishes a CLT for ICRF centered at the quantity of interest $\mu({\bf x})$.

\begin{thm}\label{th: TCL}
  Let $\bm x\in [0,1]^d$, $d\ge 2$, and $\widehat{\mu}_{s}^{\textup{ICRF}}({\bf x})$ be the ICRF estimator at point  $\bm x$ (see \Cref{def:ICRF}). Assume \textbf{(H0)}, \textbf{(H1)} and \textbf{(G1)} hold and  
  \begin{align}
	    \frac{ n 2^{\depth}}{s^2 \depth^{(d-1)/2}} \to \infty, \quad \textrm{and} \quad 2^{\depth} \depth^{- \frac{d-1}{2 }} n^{- \frac{d \log 2}{1 + d \log 2}}\to \infty, \label{condition_th2}
	\end{align}
as $n,s, \depth \to \infty$. 
  Then, for $\bm x\in [0,1]^d$ we have
$$
\sqrt{\dfrac {n \depth^{(d-1)/2}}{2^{\depth}}}( \widehat{\mu}_{s}^{\textup{ICRF}}(\bm x)-\mu(\bm x))\overset{d}{\underset{n \to \infty}{\longrightarrow}} \mathcal N(0,C(d)\mu(\bm x)(1-\mu(\bm x)))\,,
$$
where $C(d)$ is defined in \Cref{th2}. 
\end{thm}

The proof of the \Cref{th: TCL} is provided in \Cref{sec:Proof_th_TCL}. To the best of our knowledge, \Cref{th: TCL} is the first result to establish a CLT for random forests, centered at the regression function $\mu({\bf x})$, with explicit convergence rate, explicit assumption on the forest parameters (subsample size, tree depth) and an expression of the asymptotic variance. 

According to \Cref{th: TCL}, the  ICRF estimator $\widehat{\mu}_{s}^{\textup{ICRF}}(\bm x)$ is an asymptotically unbiased estimator for $\mu(\bm{x})$.  Assumption \textbf{(H1)} is crucial for controlling the bias of the estimator \( \widehat{\mu}_{s}^{\textup{ICRF}}(\bm{x}) \). Indeed, with this assumption, a control of cell diameters is sufficient to control the variations of $\mu$ inside cells. 
 
More precisely, under the assumptions of \Cref{th: TCL}, we prove that
\[
\mathbb{E}[\widehat{\mu}_{s}^{\textup{ICRF}}(\bm{x})] = \mu(\bm{x}) + O(\alpha_1^k),
\]
where \( \alpha_1 = 1 - \frac{1}{2d} \) (see \Cref{lem:bias_Ts} in \Cref{sec:Proof_th_TCL}). 

Under Condition \textbf{(G1)}, the tree depth \( k \) tends to infinity, which entails \( \alpha_1^k \to 0 \), $k\to \infty$ and therefore an exponential decay of the bias. Our analysis shows that the bias of $\widehat{\mu}_{s}^{\textup{ICRF}}(\bm x)$ is negligible compared to $(2^{\depth}/(n \depth^{(d-1)/2}))^{1/2}$, which allows us to replace $\E[\widehat{\mu}_{s}^{\textup{ICRF}}(\bm x)]$ by $\mu({\bf x})$ in \Cref{th2} with the additional second condition in Equation~\eqref{condition_th2}. 

In order to discuss assumption {\bf (H0)}, let us fix $S=\{1,\ldots,s\}$, and introduce the following three key quantities that appear in the proof of \Cref{th: TCL}
\begin{align}
T^s_1 & =\E[T^s(\bm x;\U;\bm Z_{S})\mid Z_1]-\E[T^s(\bm x;\U;\bm Z_{S})], \\
V_1^s & =\v(T_1^s), \\
p_{\depth,\U}(\bm{x})& =\mathbb{P}(X \in L_\U(\bm x) | \U). \label{def:psU}
\end{align}
The quantity \( p_{\depth, \U}(\bm{x}) \) represents the probability that a random element \( X \) falls into the leaf \( L_{\U}(\bm{x}) \), given the random partition \( \U \), and plays a central role throughout the analysis. Under Assumption \textbf{(H0)}, this quantity simplifies to \( p_{\depth, \U}(\bm{x}) = 2^{-\depth} \), becoming independent of both \( \bm{x} \) and \( \U \). This simplification enables an explicit characterization of the convergence rate in the CLT, as well as of the limiting variance. In particular, the constant \( C(d) \) admits a closed-form expression \eqref{eq_bis: const}.  
One could relax Assumption \textbf{(H0)} by assuming instead that \( X \) has a density bounded above and below. In this case,  a CLT with an explicit convergence rate can be derived, but the limiting constant is no longer explicit (see also the discussion following \Cref{th: TCLB}).

In \Cref{th: TCL}, we impose $ 2^{\depth} \depth^{- \frac{d-1}{2 }} n^{- \frac{d \log 2}{1 + d \log 2}} \to \infty$, ensuring that $2^{\depth}$ grows fast to infinity with $n$. As discussed above, this condition guarantees that the bias of $\widehat{\mu}_{s}^{\textup{ICRF}}(\bm x)$ is negligible in the CLT. 
 
As we usually impose $2^{\depth} \leq s$, this condition constraints the subsample size $s$ to be large enough. On the other hand, the first condition in Equation~\eqref{condition_th2} imposes that the subsample size $s$ does not grow too fast to infinity with $n$.  
Thus the two conditions in Equation~\eqref{condition_th2} delineate a suitable asymptotic regime for the subsample size and tree depth. Assuming that $s=n^{\alpha}$ and $2^{\depth} = n^{\beta}$, with $\beta < \alpha$, the two conditions in \Cref{th: TCL} can be rewritten as $\alpha < (1+\beta)/2$ and $\beta > (d \log 2) / ( 1 + d \log 2)$. Thus, any choice 
\begin{align}
    \frac{d \log 2}{1 + d \log 2} < \alpha < 1  \quad \textrm{and} \quad \max\left( \frac{d \log 2}{1 + d\log 2} , 2\alpha-1 \right) < \beta < \alpha,
\end{align}
satisfies the assumptions of \Cref{th: TCL}.

Now, assume that the number of leaves $2^\depth$ is of the same order as the subsample size $s$ so that each leaf of each tree contains on average one observation.
 
The bias of such forests can thus be compared with the bias of the bagged one nearest neighbour (1-NN). \Cref{table: 1} compares the bias, variance and number of leaves adapted according to subsample size $s$ of three estimators namely CRF introduced  in  \citet{biau2012analysis}, ICRF studied in this paper, and that of bagged 1-NN presented in \citet{mayala2024infinite}. 

The bias of bagged 1-NN is of order $O(s^{-1/d})$, which is smaller than that of CRF and  ICRF, and thus leads to faster rates of convergence for the CLT. Thus, nearest neighbours appear to be more efficient from a theoretical perspective. This is not surprising as centered forests do not use the data to construct tree partition: the good empirical performances of random forests may result from their capacity to adapt to the data, via a data-dependent choice of splitting directions. Studying such random forests, and in particular deriving CLT with close-form expression of convergence rate is very challenging and outside the scope of this paper.

\begin{table}[h!]
\centering
\renewcommand{\arraystretch}{1.5} 
\begin{tabular}{l c c c}
\toprule
& \textbf{Bias } & \textbf{Variance} & \textbf{Optimal choice for s} \\
\midrule
\textbf{CRF} 
& $s^{-0.75/(d\log 2)}$ & $s(\log(s))^{-1/2}/n$ & $n^ {d /(d+0.75/ \log 2)} $ \\
\textbf{ICRF} & $s^{\log_2 (1-1/(2d))}$ & $s/(n(\log_2 s)^{(d-1)/2})$ & $n^ {d /(d + 1/ \log 2)} $ \\
\textbf{1-NN} 
& $s^{-1/d}$ & $s/n$ &  $n^{d/(d+2)}$\\
\bottomrule
\end{tabular}
\caption{Comparison of convergence rates of CRF \citep[][]{biau2012analysis}, ICRF for $2^{\depth} = s$ (\Cref{th: TCL} and \Cref{lem:bias_Ts}) and 1-NN \citep[][]{mayala2024infinite}.}
    \label{table: 1}
\end{table}

\section{ICRFs for imbalanced classification}\label{sec: underICRF}

In this section, we focus on the problem of class imbalance. Following the notation of \Cref{subsec:Notation}, we denote by $Y=1$ the minority class and assume that the imbalance ratio satisfies  \( IR = n_0/n_1 \gg 1 \). 

\subsection{Rebalanced ICRF}\label{subsec: datasample}

Standard methods to handle imbalanced data works by rebalancing the dataset so that the proportion of minority samples is similar to that of majority samples. For example, \textit{Random Under Sampling} subsamples the majority class, \textit{Random Over Sampling} replicates observations from the minority class, while \textit{class-weight} assigns a weight to each sample so that, by default, the total weight of minority samples equals that of majority samples. 

Denoting by $(X',Y')$ the new data distribution, all these methods aim at preserving the conditional distribution of $X|Y$ while giving more weight to the minority distribution. Consequently, letting $f_{X|Y=j}$ (resp.\ $f_{X'|Y'=j}$) be the conditional density of $X$ (resp. of $X'$) given $Y=j$ (resp. given $Y'=j$) for all $j \in \{0,1\}$, we assume that the new data distribution is characterized by, for all ${\bf x } \in [0,1]^d$, 
\begin{align}
 f_{X'|Y'=1}({\bf x}) & = f_{X|Y=1}({\bf x}), \nonumber\\
  f_{X'|Y'=0}({\bf x}) & = f_{X|Y=0}({\bf x}), \nonumber \\
  \P (Y' = 1) & = p' \in (0,1), \label{definition_xprime}
\end{align}
with $p' > p$. \Cref{lem_density_Xprime} in \Cref{sec: proof_TCLB} establishes the density of the rebalanced distribution with density $f_{X'}$. In the rest of our analysis, we assume that we are given samples 
$Z_i'=(X_i', Y_i')$, for $i\in I:=\{1,\ldots,n'\}$ distributed as $(X',Y')$. As mentioned above, many resampling methods can lead to such rebalanced samples. 
However, in our mathematical analysis, we need to assume that the samples $Z_i'$ are i.i.d. which excludes methods that generate replicates (as Random Over Sampling or SMOTE). Concretely, one could draw  $n'\leq n_1$ samples from the original samples without replacement, with probability $p'$ of drawing a majority sample and $1-p'$ of drawing a minority sample. 

Now, we can apply the tree-based methodology described in the previous section to the new samples $Z_i'=(X_i', Y_i')$, for $i\in I:=\{1,\ldots,n'\}$. Again denote \( S = \{i_1, \hdots, i_s\} \subset \{1, \ldots, n'\} \)  with \( |S| = s\) the set of indices selected during the sampling step  and \( \bm Z_S^{'} := (Z_{i_1}^{'}, \ldots, Z_{i_s}^{'}) \) and \( \bm X_S^{'} := (X_{i_1}^{'}, \ldots, X_{i_s}^{'}) \) as the corresponding subsamples and covariate vectors, respectively.
 Accordingly, let   $T^s(\bm{x}; \U; \bm Z_S^{'})$
 be the centered tree evaluated at point $\bm x$ trained on the $s$-subsample $\bm Z_S ^{'}$ that takes the form
 \begin{equation}\label{def:base_two_sampl}
			T^s(\bm{x}; \U; \bm Z_S^{'}):= \frac{\sum_{i \in S} Y_{i}^{'}\1{\{X_{i}^{'} \in L_{\U}(\bm{x}) \}}}{N_{L_{\U}(\bm{x})}( \bm X_{S}^{'})}  \,,
	\end{equation}
with 
	$$N_{L_{\U}(\bm{x})}(\bm X_{S}^{'}) =  \sum_{i \in S} \1{\{X_{i}^{'} \in L_{\U}(\bm{x}) \}} \,.$$
We refer to the random forest obtained by aggregating the individual trees described in \Cref{def:base_two_sampl} trained on a dataset  $ \bm Z_S^{'}$, distributed as defined in Equation~\eqref{definition_xprime},
as the \textit{Rebalanced Infinite Centered Random Forest (RB ICRF)}.

\begin{dfn}[The Rebalanced ICRF]\label{def:RB_ICRF}
	The Rebalanced ICRF estimator at point  $\bm x\in [0,1]^d$ is defined as
	\begin{align}\label{est: under_classif} 
		\widehat{\mu}_{\textrm{RB},s}^{\textup{ICRF}}(\bm x):= \binom{n^{'}}{s}^{-1} \sum_{S \subset \{1, \hdots, n'\}, |S|=s} \mathbb{E}[T^s(\bm{x}; \U; \bm Z_S^{'}) \mid \bm Z_S^{'}]\,
	\end{align}
	where $\mathbb{E}[T^s(\bm{x}; \U; \bm Z_S^{'}) \mid \bm Z_S^{'}]$ are the tree predictions at point  $\bm x\in [0,1]^d$, as defined in  \Cref{def:base_two_sampl}.
 
\end{dfn}

\subsection{Asymptotic properties of the rebalanced ICRF}\label{subsec: asymp}

  In the following, we denote   $$\mu'(\bm x)= \P(Y'=1\mid X'=\bm x), \qquad \bm{x} \in [0,1]^d$$ the  regression function associated to the distribution $(X',Y')$  evaluated at point $\bm{x}$. Accordingly to  \Cref{subsec:ICRF}, denote the following three key quantities 
  \begin{align}
    {T^s}'&=T^s(\bm x, \U, \bm Z_{s}'),\\
    T_{1,s}'&=\E[{T^s}'\mid Z_1]-\E[{T^s}'],\\
    V_{1,s}'& :=\v(T_{1,s}').
\end{align}
The next theorem provides the central limit theorem for the Rebalanced ICRF estimator, when the sample size $n'$, the subsample size $s$ and the tree depth $\depth$ tend to infinity.

\begin{thm}\label{th: TCLB} Let $\bm x\in [0,1]^d$, $d\ge 2$ and $\widehat{\mu}_{\textup{RB},s}^{\textup{ICRF}}(\bm x)$ be the rebalanced ICRF estimator  at point  $\bm x $ (see \Cref{def:RB_ICRF}).
  Assume \textbf{(H0)}, \textbf{(H1)} and \textbf{(G1)} hold and
 \begin{align}
	    \frac{ n' 2^{\depth}}{s^2 \depth^{(d-1)/2}} \to \infty, \quad \textrm{and} \quad 2^{\depth} \depth^{- \frac{d-1}{2 }} n'^{- \frac{d \log 2}{1 + d \log 2}}  \to \infty
	    \label{conditions_th_rebalanced_hyperparameters}
	\end{align}
  as $n',s, \depth \to \infty$. Then, for  $\bm x\in [0,1]^d$, we have
\begin{align}
    \sqrt{\dfrac {n'}{s^2 V_{1,s}'}}( \widehat{\mu}_{\textup{RB},s}^{\textup{ICRF}}(\bm x)-\mu'(\bm x))\overset{d}{\underset{n \to \infty}{\longrightarrow}} \mathcal N(0,1)\,, 
\end{align}
with, for all $\depth$ large enough, 
\begin{align}
    \frac{c_1'({\bf x}) 2^{\depth}}{s^2 \depth^{(d-1)/2}} \leq \frac{V_{1,s}'}{\mu'(\bm x)(1-\mu'(\bm x))} \leq \frac{c_2'({\bf x}) 2^{\depth}}{s^2\depth^{(d-1)/2}},
    \label{ineq_th_rb_variance}
\end{align}
for some constant $c_1'({\bf x})$ and $c_2'({\bf x})$, which are made explicit in \Cref{sec: proof_TCLB} (\Cref{prop:v1prime}).
\end{thm}

The proof of \Cref{th: TCLB} is given in \Cref{sec: proof_TCLB}. \Cref{th: TCLB} establishes a CLT for a forest trained on a rebalanced data set. The conditions in Equation~\eqref{conditions_th_rebalanced_hyperparameters} are related to the hyperparameters of the forest (subsample size and tree depth) and are the same as that of \Cref{th: TCL}, except for the fact that $n$ is replaced by $n'$. Besides, inequalities~\eqref{ineq_th_rb_variance} gives the exact rate of convergence of the CLT, which is of order $\sqrt{n' \depth^{(d-1)/2}/2^{\depth}}$, the exact same rate as in \Cref{th: TCL}. Unfortunately, \Cref{th: TCLB} differs from \Cref{th: TCL}  in two aspects: the asymptotic variance is unknown and the centering term is not $\mu({\bf x})$.

Let us discuss this first point. 
When the covariates are uniformly distributed on $[0,1]^d $, a CLT can be established with exact limiting constant (\Cref{th: TCL}). Indeed, under {\bf (H0)}, we have 
\begin{align}
    p_{\depth,\U}(\bm{x})& =\mathbb{P}(X \in L_\U(\bm x) | \U) = 2^{- \depth}. 
\end{align}
The fact that $p_{\depth,\U}(\bm{x})$ is deterministic and independent of ${\bm x}$ allows us to perform exact calculations throughout the proof. Unfortunately, as we rebalance the original data set, the distribution of the resulting inputs $X'$ is not uniform anymore (see \Cref{lem_density_Xprime} in \Cref{sec: proof_TCLB}), and the quantity $p_{\depth,\U}'(\bm{x})$ takes the form (see \Cref{lem: pU} in \Cref{sec: proof_TCLB})
\begin{align}
p_{\depth,\U}'(\bm{x}) & = \mathbb{P}(X' \in L_\U(\bm x) | \U) \\
& = \dfrac{c'(\bm x)}{2^{\depth}} \left(1+ \alpha'({\bf x}) \varepsilon'_{\U}({\bf x}) \textup{Diam}(L_{\U}(\bm{x}))\right) \qquad \text{a.s.}
\end{align}
where $\alpha'({\bf x})$ and $c'(\bm x)$ depend on $p,p', \mu'(\bm x)$, and where $\varepsilon'_{\U}({\bf x})$ is a random variable. Since  $p_{\depth,\U}'(\bm{x})$ now depends on ${\bf x}$ and is random (since it involves a dependence on $\U$ via $\varepsilon'_{\U}$), we are not able to adapt the exact calculation of the proof of \Cref{th: TCL}, and thus we do not have access to the exact asymptotic constant.

Regarding the second point, we observe that the centering term in the CLT of \Cref{th: TCLB} is $\mu'({\bf x})$, which is the value of the regression function of the rebalanced data set $\P(Y'=1|X'={\bf x})$. One can show that $\mu'({\bf x}) \neq \mu({\bf x})$ when $p'\neq p$ which is the case here. Thus, the rebalanced ICRF prediction $\widehat{\mu}_{\textup{RB},s}^{\textup{ICRF}}(\bm x)$ is asymptotically biased. In the next section, we present a debiased estimator based on an importance sampling procedure and prove its asymptotic normality.

\subsection{Importance Sampling ICRF}

The rebalanced ICRF introduced in \Cref{subsec: datasample} may benefit from enhanced predictive performance, particularly for the minority class, as it operates on a rebalanced data set. However, it also introduces an estimation bias, which is illustrated by the central term in \Cref{th: TCLB}, which differs from $\mu({\bf x})$. To address this issue, following \citet{mayala2024infinite}, a correction based on an importance sampling approach can be applied. Indeed, simple calculations (see \cref{lem_mu_prime_expression} in \Cref{sec: proof_TCLB}) show that, for all $\bm x \in [0,1]^d$, it holds that

\begin{align} 
\mu(\bm x)=\dfrac{p(1-p')\mu'(\bm x)}{p'(1-p)(1-\mu'(\bm x))+(1-p')p\mu'(\bm x)}.\label{eq_expr_mu_muprime}
\end{align}

Thus the function 
\begin{align}
    g: u \mapsto \dfrac{p(1-p')u}{p'(1-p)(1- u)+(1-p')pu},
\end{align}
can be used to obtain a debiased prediction $\mu({\bf x})$ from a biased prediction $\mu'({\bf x})$. While the probability $p'$ is chosen by the user (usually set to $1/2$ to achieve a balance between the two classes), this is not the case for $p$, that we need to estimate. This leads to the following estimate, referred to as Importance Sampling ICRF. 

\begin{dfn}[Importance Sampling ICRF]\label{def:IS_ICRF}
    The importance sampling ICRF (IS-ICRF) estimator at point  $\bm x\in [0,1]^d$ is defined by
    \begin{equation}\label{eq:def_IS_ICRF}
	\widehat{\mu}_{\textup{IS},s}^{\textup{ICRF}}(\bm x):= \frac{n_1 (1-p')\widehat{\mu}_{\textup{RB},s}^{\textup{ICRF}}(\bm x)}{p' n_0(1- \widehat{\mu}_{\textup{RB},s}^{\textup{ICRF}}(\bm x))+n_1(1-p') \widehat{\mu}_{\textup{RB},s}^{\textup{ICRF}}(\bm x)}\,, 
\end{equation}
\noindent where $\widehat{\mu}_{\textup{RB},s}^{\textup{ICRF}}(\bm x)$ is the rebalanced ICRF estimator at point $\bm x$ as defined in \Cref{def:RB_ICRF}. 
\end{dfn}
We are now ready to state the central limit theorem for the importance sampling ICRF estimator. 

\begin{cor}\label{cor: TCLB}
  Let $\bm x\in [0,1]^d$, $d\ge 2$ and $\widehat{\mu}_{\textup{IS},s}^{\textup{ICRF}}(\bm x)$ be the importance sampling ICRF  estimator (see \Cref{def:IS_ICRF}). Let $p\neq 0$, $p'\neq 1$ and assume \textbf{(H0)}, \textbf{(H1)} and \textbf{(G1)}  hold and
  \begin{align}
	    \frac{ n' 2^{\depth}}{s^2 \depth^{(d-1)/2}} \to \infty, \quad \textrm{and} \quad 2^{\depth} \depth^{- \frac{d-1}{2 }} n'^{- \frac{d \log 2}{1 + d \log 2}}  \to \infty
	\end{align}
as $n',s, \depth \to \infty$. Then, for  $\bm x\in [0,1]^d$, we have
\begin{align}
    \frac{1}{g'(\mu'({\bf x}))}\sqrt{\dfrac {n'}{s^2 V_{1,s}'}}( \widehat{\mu}_{\textup{IS},s}^{\textup{ICRF}}(\bm x)-\mu(\bm x))\overset{d}{\underset{n \to \infty}{\longrightarrow}} \mathcal N(0,1)\,, 
\end{align}
with 
\begin{align}
    c_1'({\bf x})\mu'(\bm x)(1-\mu'(\bm x))  \leq \dfrac{s^2 \depth^{(d-1)/2}}{2^{\depth}}  V_{1,s}' \leq c_2'({\bf x})\mu'(\bm x)(1-\mu'(\bm x)),
\end{align}
where $c_1'({\bf x}), c_2'({\bf x}) $ are made explicit in \Cref{sec: proof_TCLB} (\Cref{prop:v1prime}) and with  
\begin{align}
    g'(\mu'({\bf x})) & = \frac{p'(1-p)}{p(1-p')} \left(\frac{\mu({\bf x})}{\mu'({\bf x})}\right)^2.
\end{align}
\end{cor}
The proof of \Cref{cor: TCLB} is given in \Cref{sec: proof_TCLB}.

\subsection{Comparison of ICRF and Importance Sampling ICRF for small $p$.}\label{subsec:comp}

We need to slightly modify our framework to compare the asymptotic variances of $\widehat{\mu}_{s}^{\textup{ICRF}}$ and $\widehat{\mu}_{\textup{IS},s}^{\textup{ICRF}}$ respectively defined in \Cref{def:ICRF} and \Cref{def:IS_ICRF} for small values of $p$. We assume that we observe i.i.d.\ data, distributed as $Z = ( X, Y)$ with $\P(Y=1) = p$ and whose conditional densities $f_{X|Y=1}$ and $f_{X|Y=0}$ satisfy the following assumption.  
\begin{enumerate}
\item [\textbf{(H3)}] \label{hyp7}  There exists $p''\in (0,1)$ such that $f_{X''}({\bf x}) = p'' f_{X|Y=1}({\bf x}) + (1-p'') f_{X|Y=0}({\bf x})$ is the uniform density. Besides, the densities $f_{X|Y=1}$ and $f_{X|Y=0}$ are $L$-Lipschitz and  there exist $0 < b_1 < b_2 < \infty$ such that, for all ${\bf x} \in [0,1]^d$,
\begin{align}
b_1 \leq f_{X|Y=0}({\bf x}), f_{X|Y=1}({\bf x}) \leq b_2.   
\end{align}
\end{enumerate}
We also consider the rebalanced distribution ${  Z'} = ( X', Y')$ with
\begin{align}
 f_{X'|Y'=1}({\bf x}) & = f_{X|Y=1}({\bf x}), \nonumber\\
  f_{X'|Y'=0}({\bf x}) & = f_{X|Y=0}({\bf x}), \nonumber \\
  \P (Y' = 1) & = p' \in (0,1).
\end{align}

In this framework, we consider the two estimators ICRF and IS-ICRF. We assume that both estimators are built using $n$ observations. This is the case for example if observations are reweighted in order to build the RB-ICRF and then the IS-ICRF. Note however that this scenario departs from our theoretical analysis, as we do not allow repetitions in the observations. In this scenario, the following \Cref{thm_comparison_variances} compares the asymptotic variances of the two estimators. 

\begin{thm}\label{thm_comparison_variances}
Assume \textbf{(H3)} and \textbf{(G1)} hold. Let $d\ge 2$ and $n,s, \depth \to \infty$ such that 
 \begin{align}
	    \frac{ n 2^{\depth}}{s^2 \depth^{(d-1)/2}} \to \infty, \quad \textrm{and} \quad 2^{\depth} \depth^{- \frac{d-1}{2 }} n^{- \frac{d \log 2}{1 + d \log 2}}  \to \infty.
	\end{align}
Fix $\bm x\in [0,1]^d$. Then, the ICRF estimator satisfies 
\begin{align}
    \sqrt{\dfrac {n}{s^2 V_{1,s}}}( \widehat{\mu}_{s}^{\textup{ICRF}}(\bm x)-\mu(\bm x))\overset{d}{\underset{n \to \infty}{\longrightarrow}} \mathcal N(0,1)\,, 
\end{align}
and the Importance Sampling ICRF estimator satisfies
\begin{align}
    \frac{1}{g'(\mu'({\bf x}))}  \sqrt{\dfrac {n}{s^2 V_{1,s}'}}( \widehat{\mu}_{\textup{IS},s}^{\textup{ICRF}}(\bm x)-\mu(\bm x))\overset{d}{\underset{n \to \infty}{\longrightarrow}} \mathcal N(0,1)\,, 
\end{align}
with, for all $\depth$ large enough, 
\begin{align}
    \dfrac{V_{1,s}'}{V_{1,s}} g'(\mu'({\bf x}))^2 = O(p).
\end{align}
\end{thm}

The proof of \Cref{thm_comparison_variances} is given in \Cref{sec:proof_th_compare}. \Cref{thm_comparison_variances} establishes that the Importance Sampling ICRF enjoys a variance reduction compared to the standard ICRF, built on the original imbalanced dataset. This variance reduction occurs for highly imbalanced dataset, for which $p$ is small. This variance reduction is all the more important as $p$ is small. This phenomenon is illustrated via experiments in the next section.

\section{Numerical illustrations}\label{sec:numerical}

In this section, we present numerical illustrations to validate and test theoretical results established in \Cref{th: TCL},  \Cref{th: TCLB}, \Cref{cor: TCLB} and \Cref{thm_comparison_variances}.

We begin by generating a dataset composed of $n$ i.i.d.\ pairs $(X_i, Y_i)$ distributed as $(X,Y)$ satisfying $X\sim U([0,1]^2)$ and where $Y \in \{0,1\}$ satisfies  
 \begin{align*}
     \P(Y=1|X=x) = \mu(x) = \frac{1}{1+\exp(-(\beta_0 + 3x_1+ 2x_2))}.
 \end{align*}
We consider an Imbalance Scenario  \textbf{(ImB-Sc)}, in which  $\beta_0$ is chosen such that $\P(Y=1) = 0.1$.

  We compare the three estimates defined in the previous sections: the ICRF estimate $\widehat{\mu}_{s}^{\textup{ICRF}}(\bm x)$, the rebalanced  ICRF $\widehat{\mu}_{\textup{RB},s}^{\textup{ICRF}}(\bm x)$ and the  importance sampling ICRF  $\widehat{\mu}_{\textup{IS},s}^{\textup{ICRF}}(\bm x)$ respectively defined in \Cref{def:ICRF}, \Cref{def:RB_ICRF} and \Cref{def:IS_ICRF}. To do so, we use the \texttt{R} package  \texttt{ranger}  (see \citet{wright2017ranger}), a fast implementation of random forests. In order to implement the rebalanced  ICRF, letting $n_1$ be the number of original minority samples, we use the parameter $$\texttt{sample.fraction} = \left(\frac{s}{n}\right) \left[\frac{n_1}{n} \frac{1-p'}{p'}, \frac{n_1-1}{n}\right],$$ which allows us to subsample a fraction $s/n$ of the minority class, and the rest among the majority class, so that the final proportion of $1$ equals $p'$. 
  
   All  simulations use a subsample size $s = n^{\alpha}$, with a maximal depth $\texttt{max.depth} = \beta \log_2 n$, which corresponds to a number of terminal nodes $\texttt{maxnodes} = n^{\beta}$ if trees are balanced. All other parameters are set to their default values (in particular, we use $500$ trees).  Different values for the parameters $\alpha$ and $\beta$ are considered in the experiments. 
  Note that according to \Cref{th: TCL}, the limiting value for $\alpha$ is $$\frac{d \log 2}{1 + d \log 2} \simeq 0.58$$ (since $d=2$ in our setting). 
  For each sample size $n$, and for each value of $(\alpha, \beta)$, the following procedure is repeated $B=1000$ times: 
  \begin{enumerate}
      \item A dataset is created with the specified sample size, following the distribution described above.
      \item A random forest is trained on this data set with the pre-specified subsample size and maximal depth (which is equivalent to the choice of $\alpha$ and $\beta$).
      \item The forest prediction is evaluated at an arbitrary point, namely ${\bf x} = (0.7, 0.7)$. 
  \end{enumerate}
  For each value of $(n, \alpha, \beta)$, we average the $1000$ predictions to estimate $\E[\widehat{\mu}_{s}^{\textup{ICRF}}(\bm x)]$.  We then use this estimation to compute   \begin{align}
       &  \E[\widehat{\mu}_{s}^{\textup{ICRF}}(\bm x)] -  \mu({\bf x}) \quad 
  \textrm{and}   \quad  \log \left( \E \left[ \left( \widehat{\mu}_{s}^{\textup{ICRF}}(\bm x)-\E[\widehat{\mu}_{s}^{\textup{ICRF}}(\bm x)] \right)^2  \right] \right). \label{quantities_left_hand_side_tcl_experiments}
  \end{align}
  According to our theoretical results, we should observe that these quantities are controlled \textit{via} linear function of $\log n$ with coefficients that correspond respectively to the bias rate and the variance rate. Note that a major difference between our theoretical analysis and our experimentation resides in forest algorithms: our theory is valid for centered random forests but experiments are performed with Breiman's random forests. This is due to the fact that theory for Breiman forests is notoriously difficult, and centered forests are not implemented. 
  
  \Cref{fig:expect} and \Cref{fig:var} display respectively the two terms in Equation~\eqref{quantities_left_hand_side_tcl_experiments} as a function of the logarithm of the sample size for different values of $(\alpha, \beta)$.
  
  \begin{figure}
     \centering
     \includegraphics[scale=0.4]{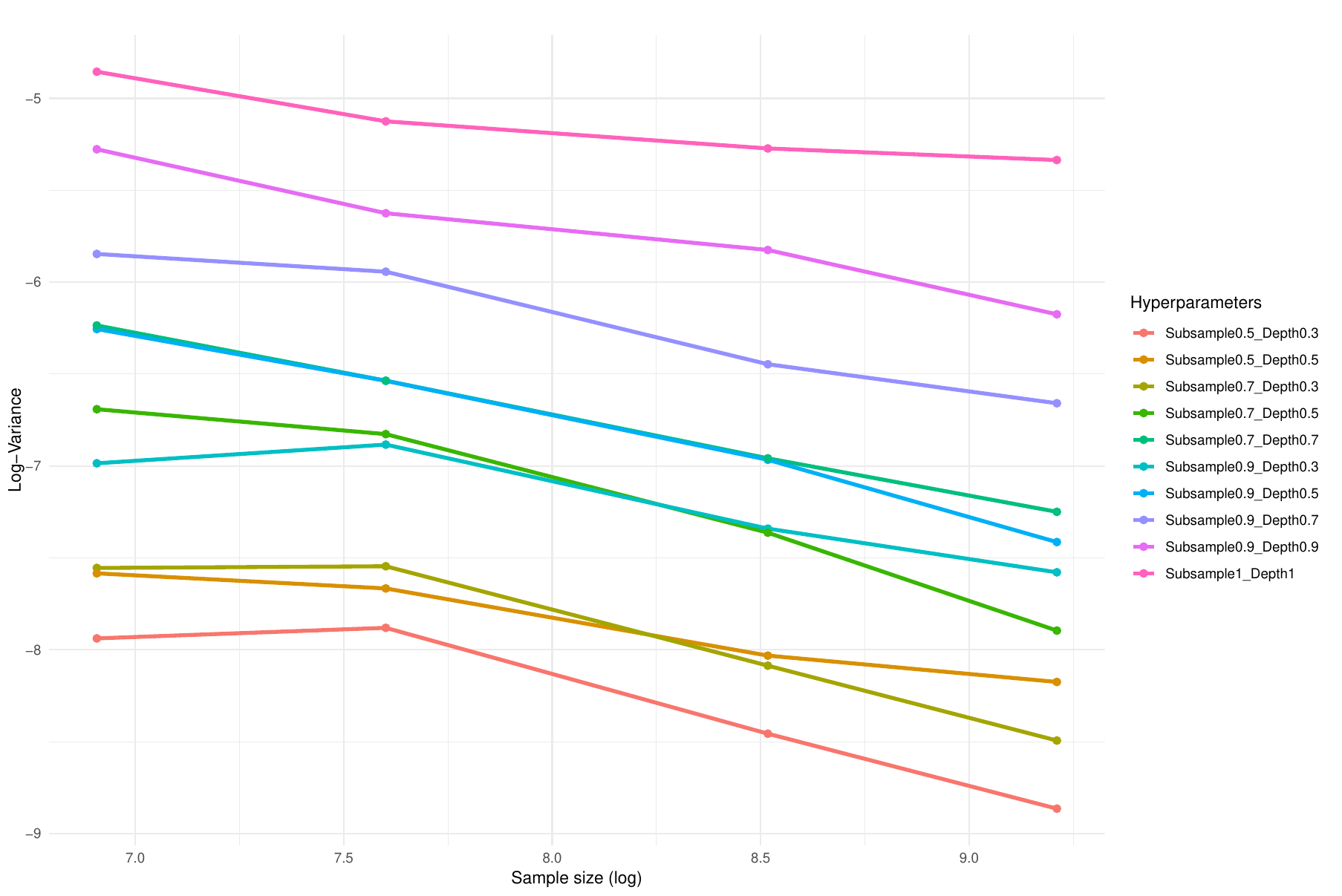}
     \caption{Variance of the CLT (second term in \eqref{quantities_left_hand_side_tcl_experiments}) for the ICRF, as a function of the logarithm of the sample size, for different choices of parameters $\alpha, \beta$ (displayed in the legend as Subsample$\alpha$\_Depth$\beta$).}
     \label{fig:var}
 \end{figure}
 
In \Cref{fig:var}, all curves appear to be straight lines as predicted by the theory for ICRF. Indeed, assume that the convergence in \Cref{th2} holds in $L^2$, that is 
$$
\dfrac {n \depth^{(d-1)/2}}{2^{\depth}} \E [( \widehat{\mu}_{s}^{\textup{ICRF}}(\bm x)-\E[\widehat{\mu}_{s}^{\textup{ICRF}}(\bm x)])^2] \to C(d)\mu(\bm x)(1-\mu(\bm x))\,.
$$
Taking the logarithm, and recalling that we choose the number of nodes $2^k = n^\beta$, we have
\begin{align}
& \log (\E [( \widehat{\mu}_{s}^{\textup{ICRF}}(\bm x)-\E[\widehat{\mu}_{s}^{\textup{ICRF}}(\bm x)])^2] ) \sim - (1 - \beta) \log n   - \frac{d-1}{2} \log \log n + C_{1,d,\beta}({\bf x})    \nonumber 
\end{align}
with 
\begin{align}
C_{1,d,\beta}({\bf x}) =    \log (C(d)\mu(\bm x)(1-\mu(\bm x)))  - \frac{d-1}{2} \log \beta  - \frac{d-1}{2} \log \log 2. \nonumber
\end{align}
Thus the theory for ICRF predicts that both the slope and the intercept of lines in \Cref{fig:var} depend on $\beta$, and in particular do not depend on the subsample size. 
We observe that several lines seem to have the same slope: for example, the three lines corresponding to the same value of $\beta=0.3$ (\textit{Depth0.3} in the caption), the two lines corresponding to $\beta = 0.7$, and two out of the three lines corresponding to $\beta=0.5$. We also observe that large values of $\beta \in (0,1)$ correspond to lower slopes as anticipated by theory. While the intercept cannot be visualised here, it appears that they are not in line with the theory we developed. Nevertheless, it is remarkable that some theoretical results obtained for ICRF can be retrieved empirically for Breiman's random forests.  

\Cref{fig:expect} displays biases, the first term in \eqref{quantities_left_hand_side_tcl_experiments}. We first note that biased are negative, which is not surprising: the dataset contains a majority of $0$, thus RF predictions are likely to be close to zero, which creates a negative bias. We also observe that all biases tends to zero, which is expected since the tree depth increases ($\depth = \beta \log_2 n$).

 \begin{figure}
     \centering
     \includegraphics[scale=0.4]{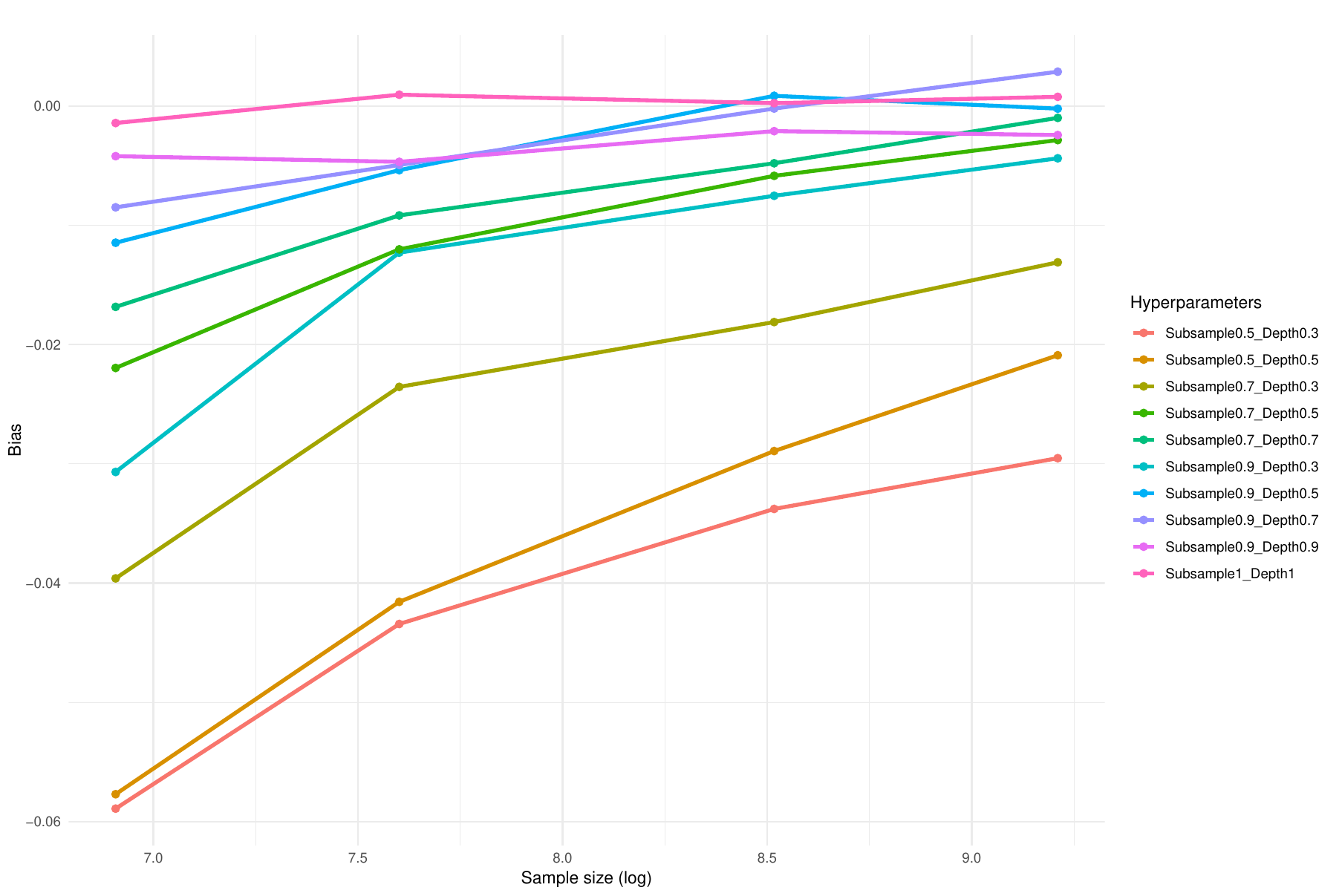}
     \caption{Bias of the CLT (first term in \eqref{quantities_left_hand_side_tcl_experiments}) as a function of the logarithm of the sample size, for different choices of parameters $\alpha, \beta$ (displayed in the legend as Subsample$\alpha$\_Depth$\beta$).}
     \label{fig:expect}
 \end{figure}
 
We then compare the three estimators on  a small sample size setting $n=100$. Histograms of the values $\widehat{\mu}_{s}^{\textup{ICRF}}(\bm x)$, $\widehat{\mu}_{\textup{RB},s}^{\textup{ICRF}}(\bm x)$ and $\widehat{\mu}_{\textup{IS},s}^{\textup{ICRF}}(\bm x)$ are displayed in \Cref{fig:hist_small_sample}. They correspond respectively to Breiman random forests trained on the original sample, Breiman forests trained on a subsample (see parameter \texttt{sample.fraction} above) and the debiasing approach applied to this last forest. 
 
In \Cref{fig:hist_small_sample}, we first observe that the average of the predictions (blue dashed line) is close to the theoretical value (red line), which is either $\mu({\bf x})$ for the first and third plot, or $\mu'({\bf x})$ for the second one. This is particularly true for Breiman's forests trained on the original dataset (RF). We also note that the prediction of Breiman's forest may be far from the correct value (predictions up to $0.6$ compared to the true value $0.17$). However, we see that the predictions of the IS-RF (the debiased forest) are closer to the true value. This is quantified by a smaller variance $0.061$ compared to the variance of the prediction of Breiman's forests $0.121$. The same conclusions hold for the large sample size setting $n=10.000$ and small number of Monte Carlo repetitions $B=100$ whose results are presented in \Cref{app_sec_additional_experiments} (see \Cref{fig:hist_large_sample}). This corroborates experimentally the variance reduction proved in \Cref{thm_comparison_variances} for small $p$.

\begin{figure}[h!]
    { \centering                                       
     \begin{subfigure}[b]{0.45\textwidth}
         \centering
         \includegraphics[width=1.3\textwidth]{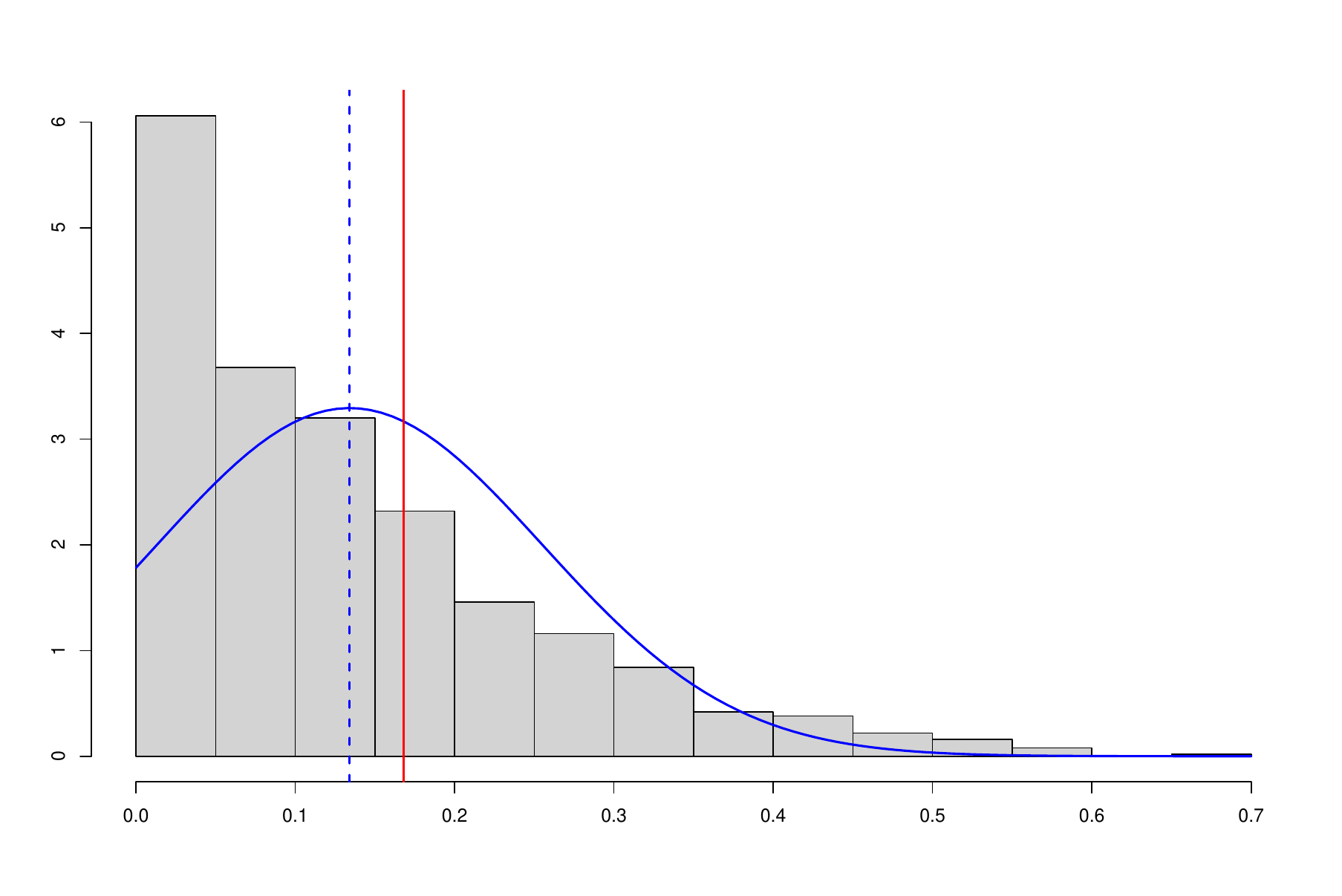}
         \caption{RF}
         \label{hist_small_sample_RF}
     \end{subfigure}
     \hfill
     \begin{subfigure}[b]{0.45\textwidth}
         \centering
         \includegraphics[width=1.3\textwidth]{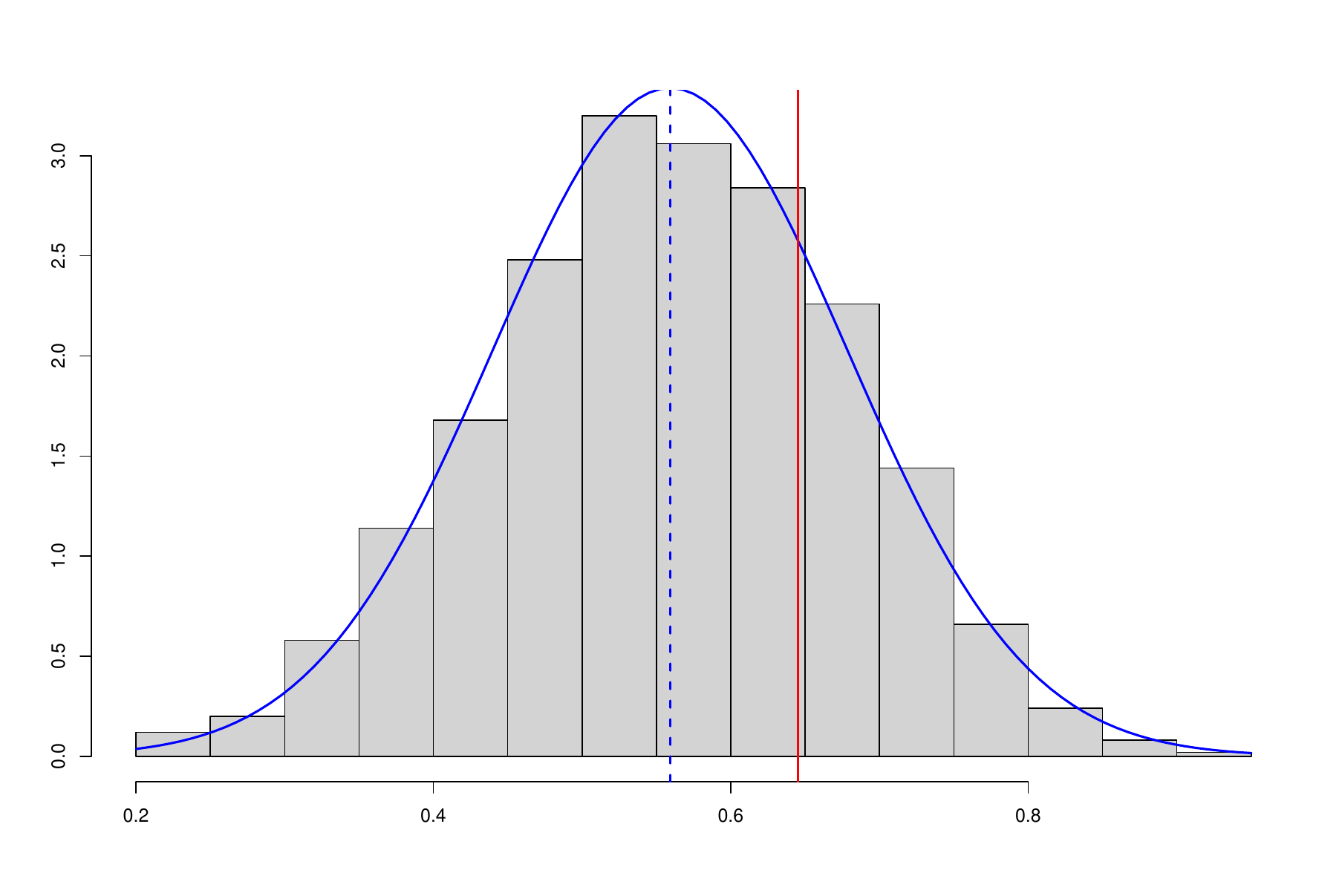}
         \caption{RB-RF}
         \label{hist_small_sample_RBRF}
     \end{subfigure}
      \hfill
     \begin{subfigure}[b]{0.45\textwidth}
         \centering
        \includegraphics[width=1.3\textwidth]{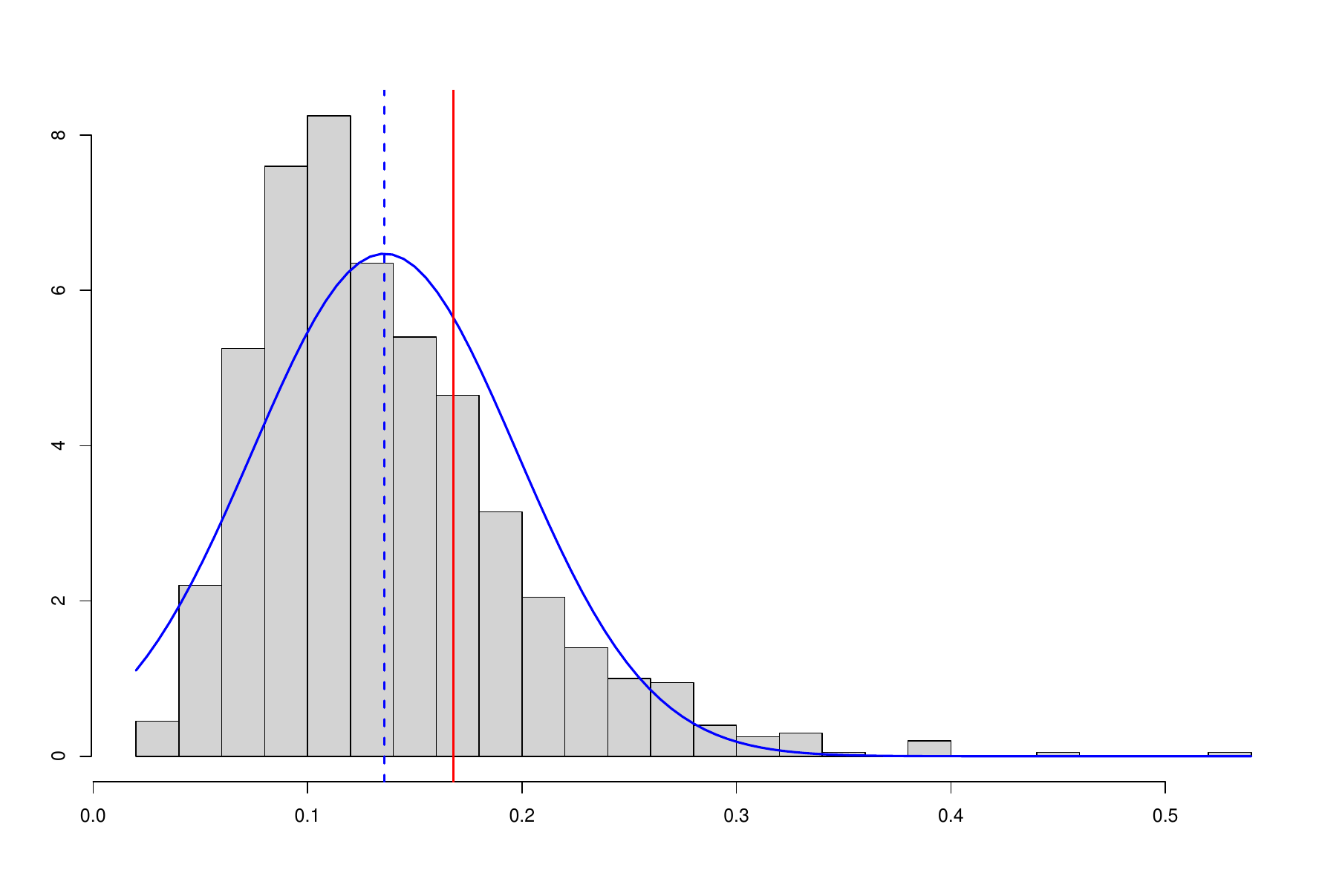}
        \caption{IS-RF}
         \label{hist_small_sample_ISRF}
    \end{subfigure}
        \caption{Histograms of predictions from different estimators in scenario \textbf{(ImB-Sc)} with $\alpha_1 = 0.9$, $\alpha_2 = 0.7$, $p'=0.5$, $n=100$ and $B=1000$ replicates. The blue curve is that of a Gaussian density whose mean (blue dashed line) and variance are estimated from the data. The red line corresponds to the centering term provided by our theory ($\mu(x)$ for the first and third graph, $\mu'(x)$ for the second one). From left to right, empirical variances are: $0.121$, $0.119$, $0.061$.}
        \label{fig:hist_small_sample}}
\end{figure}

\clearpage

\appendix

\section{Proof of \Cref{th2} and \Cref{th: TCL}}\label{sec:Proof_th_TCL}

For the sake of clarity, we begin by recalling the notations used throughout \Cref{sec:Proof_th_TCL}. Let \( Z_i := (X_i, Y_i) \) be  i.i.d\ copies of a random pair \( (X, Y) \), where \( X \in [0,1]^d \) and $Y=\{ 0,1\}$ with \( d \geq 2 \). Fixing $S=\{1,\ldots,s \}$ the prediction of a tree trained on the subsampled dataset $\bm Z_s= (Z_1,\ldots,Z_s)$ at point $\bm x \in [0,1]^d$  is defined as  
\begin{equation}\label{eq:T_onesample}
	T^s(\bm{x}; \U; \bm Z_{s})= \frac{\sum_{i=1}^s Y_{i}\1{\{X_{i} \in L_{\U}(\bm{x}) \}}}{N_{L_{\U}(\bm{x})}( \bm X_{s})}
	\end{equation}
	where 
	$N_{L_{\U}(\bm{x})}(\bm X_{s}):=
	\sum_{i=1}^s \1{\{X_{i} \in L_{\U}(\bm{x}) \}} \,$ 
is the number of observations $\bm X_{s}=(X_1,\ldots,X_s)$  falling into the leaf $L_{\U}(\bm{x})$.
	We also define the centered version of the tree prediction conditional on the first observation as
\begin{align*}
T^s_1 := \mathbb{E}[T^s(\bm{x}; \U; \bm{Z}_s) \mid Z_1] - \mathbb{E}[T^s(\bm{x}; \U; \bm{Z}_s)]
\end{align*}
and denote its associated variance by 
\begin{align*}
V^s_1 := \v(T^s_1).
\end{align*}
For notational convenience, we denote \( N^l := N_{L_{\U}(\bm{x})}(\bm{X}_l) \) for \( l = 1, \ldots, n \). We further define the following conditional quantities:
\begin{align}
\mu_{\depth,\U}(\bm{x})& :=\P(Y=1|X \in L_{\U}(\bm{x}),\U ) \\
\sigma^2_{\depth,\U}(\bm x) & := \mu_{\depth,\U}(\bm{x})(1- \mu_{\depth,\U}(\bm{x}))\\
p_{\depth,\U}(\bm{x}) & :=\mathbb{P}(X \in L_\U(\bm x) | \U) \label{def:psU}
\end{align}
where $p_{\depth,\U}(\bm{x})$ is the probability that a random element $X$ belongs to  $L_{\U}(\bm{x})$.
We also introduce the constants
\begin{align*}
 \alpha_1 = 1- 1/(2d) \ \ \text{and} \ \ \alpha_2 = 1 - 3/(4d).  
\end{align*} 
The diameter associated with the  leaf $L_{\U}(\bm{x})$ is 
$$
\textup{Diam}(L_{\U}(\bm{x})):= \sup_{\bm{x},\bm{x'} \in L_\U(\bm x)} \|\bm{x}-\bm{x'}\|_\infty,
$$ 
where $\|\cdot\|_\infty$ refers to the max-norm.  To lighten the notation, we omit the dependence on the tree depth \( \depth \) in \( L_{\U}(\bm{x}) \). Finally, recall that throughout this section, all asymptotics are considered in the limit as $n,s,k \to \infty$. In the remainder of this Appendix, we outline the main steps in the proofs of \Cref{th2} and \Cref{th: TCL} in \Cref{subsec:Proof_th_TCL}. Auxiliary results needed for the analysis are collected in Appendix \ref{subsec:pre_results}, while the computations related to the variance terms  \( \v(T^s) \) and \( V_1^s \) are deferred to Appendices \ref{subsec:Control_vETs} and \ref{subsec:equiv_V1s}, respectively.

\subsection{Outline of the proof of \Cref{th2} and \Cref{th: TCL}}\label{subsec:Proof_th_TCL}

The proofs of \Cref{th2} and \Cref{th: TCL} rely heavily on Theorem 4.1 and Corollary 4.2 from \citet{mayala2024infinite}. For the sake of completeness, we have restated these results in \Cref{th: thmaya}. Although originally established for a broader class of forests than the ICRF, we present them here specifically in the context of the ICRF.

\begin{thm}\label{th: thmaya}[Theorem 4.1 and Corollary 4.2 in \cite{mayala2024infinite}]
	Let $\widehat{\mu}_{s}^{\textup{ICRF}}$ be the ICRF estimator at point  $\bm x\in [0,1]^d$ as defined in \Cref{def:ICRF}.  If $nV_1^s\to \infty$ as $n,s,k \to \infty$, we have  
	$$
\sqrt{\dfrac {n}{s^2V_1^s}}( \widehat{\mu}_{s}^{\textup{ICRF}}(\bm x)-\E[\widehat{\mu}_{s}^{\textup{ICRF}}(\bm x)])\overset{d}{\underset{n \to \infty}{\longrightarrow}} \mathcal N(0,1)\,.
$$
Furthermore, assume that 
\begin{align}\label{eq:secThA1}
    \sqrt{\dfrac {n}{s^2V_1^s}}\left|\E[\widehat{\mu}_{s}^{\textup{ICRF}}(\bm x)]-\mu(\bm{x})\right| \to 0.
\end{align}
Then, 
\begin{align}
    \sqrt{\dfrac {n}{s^2V_1^s}}( \widehat{\mu}_{s}^{\textup{ICRF}}(\bm x)- \mu({\bf x}))\overset{d}{\underset{n \to \infty}{\longrightarrow}} \mathcal N(0,1)\,.
\end{align}
\end{thm} 

According to \Cref{th: thmaya}, the proofs of \Cref{th2} and \Cref{th: TCL} reduce to verifying that the condition \( nV_1^s \to \infty \) and the condition given in \Cref{eq:secThA1} hold as $n,s,k \to \infty$, respectively.

\begin{proof}[Proof of \Cref{th2}]
According to \Cref{th: thmaya}, to prove \Cref{th2}, we only have to show that  $nV_1^s\to \infty$ holds as $n,s,k \to \infty$. Using \Cref{subsec:equiv_V1s}, under conditions \textbf{(H0)}  and \textbf{(H1)}, we have
\begin{align*}
V_1^s = \frac{C(d) 2^{\depth}}{s^2 \depth^{(d-1)/2}}\mu(\bm x)(1-\mu(\bm x))+o\Big(\frac{(2 \sqrt{\alpha_2})^\depth}{s^2}\Big)+O\Big(\frac{4^\depth e^{-s2^{-\depth}}}{s^2}\Big)
\,.
\end{align*}
If moreover condition {\bf (G1)} holds, as $n,s,k \to \infty$ we have 
\begin{align}\label{eq1:proof_th2}
    V_1^s \sim \frac{C(d) 2^{\depth}}{s^2 \depth^{(d-1)/2}}\mu(\bm x)(1-\mu(\bm x)),
\end{align}
which under the additional assumption  
\begin{align}
    \frac{n 2^{\depth}}{s^2 \depth^{(d-1)/2}} \to \infty 
\end{align}
ensures $n V_1^s \to \infty$. Finally \Cref{th: thmaya} applies and leads to 
\begin{align}
    \sqrt{\dfrac {n \depth^{(d-1)/2}}{2^{\depth}}}( \widehat{\mu}_{s}^{\textup{ICRF}}(\bm x)-\E[\widehat{\mu}_{s}^{\textup{ICRF}}(\bm x)])\overset{d}{\underset{n \to \infty}{\longrightarrow}} \mathcal N(0, C(d) \mu(\bm x)(1-\mu(\bm x)) )
\end{align}
concluding the proof.
\end{proof}

\begin{proof}[Proof of \Cref{th: TCL}]
Assume the conditions of \Cref{th2} hold. According to \Cref{th: thmaya}, to prove \Cref{th: TCL}, we only have to show that  the condition
$$\sqrt{n/(s^2V_1^s)}\left|\E[\widehat{\mu}_{s}^{\textup{ICRF}}(\bm x)]-\mu(\bm{x})\right| \to 0$$ holds as $n,s,k \to \infty$, which from \Cref{eq1:proof_th2} turns to 
\begin{align*}
    \sqrt{\dfrac {n \depth^{(d-1)/2}}{2^{\depth}}}\left|\E[\widehat{\mu}_{s}^{\textup{ICRF}}(\bm x)]-\mu(\bm{x})\right| \to 0.
\end{align*}
From the definition of the ICRF estimator given in  \Cref{def:ICRF}, we have
\begin{align*}
  \E[\widehat{\mu}_{s}^{\textup{ICRF}}(\bm x)]&=\E\left[  \binom{n}{s}^{-1} \sum_{S \subset \{1,\ldots,n \}, |S|=s} \mathbb{E}[T^s(\bm{x}; \U;\bm Z_{s}) \mid \bm Z_{s}] \right]\\
  &= \E\big[\E[T^s(\bm{x}; \U;\bm Z_{s}) \mid \bm Z_{s}]\big]\\
  &= \E[T^s(\bm{x}; \U; \bm Z_{s})].
\end{align*}

Together with \Cref{lem:bias_Ts} i), as $n,s,k \to \infty$, we obtain 
\begin{align*}
    \E[\widehat{\mu}_{s}^{\textup{ICRF}}(\bm x)] - \mu({\bf x}) & = \E[T^s(\bm{x}; \U; \bm Z_{s})] - \mu({\bf x})\\
    & = O(\alpha_1^\depth).
\end{align*}
Consequently, 
\begin{align}
    & \sqrt{\dfrac {n \depth^{(d-1)/2}}{2^{\depth}}}\left|\E[\widehat{\mu}_{s}^{\textup{ICRF}}(\bm x)]-\mu(\bm{x})\right| \\
    & = O\big( \sqrt{\dfrac {n \depth^{(d-1)/2}}{2^{\depth}}} \alpha_1^{\depth} \big)\\
    & = O\big( n^{1/2} \left(\frac{\alpha_1}{\sqrt{2}}\right)^\depth \depth^{(d-1)/4} \big). \label{eq_proof_th2_5_1}
\end{align}
Noticing that 
\begin{align}
\left( \frac{\alpha_1}{\sqrt{2}} \right)^{\depth} & = \exp \left( \depth \log \left( \frac{\alpha_1}{\sqrt{2}}\right) \right) \\
& = \exp \left( \frac{\depth \log(\alpha_1/\sqrt{2})}{\log 2} \log 2 \right) \\
& = \left(2^{\depth}\right)^{\frac{\log(\alpha_1/\sqrt{2})}{\log 2}},
\end{align}
the quantity \eqref{eq_proof_th2_5_1} tends to zero if 
\begin{align}
n^{-1/2} \left(2^{\depth}\right)^{- \frac{\log(\alpha_1/\sqrt{2})}{\log 2}} \depth^{-(d-1)/4} \to \infty,
\end{align}
which is equivalent to \begin{align}
n^{- \frac{\log 2}{2 \log(\sqrt{2}/\alpha_1)}} 2^{\depth} \depth^{- \frac{(d-1)\log 2}{4 \log(\sqrt{2}/\alpha_1)}} \to \infty, \label{eq_proof_th2_5_final_condition}
\end{align}
since $\frac{\log(\sqrt{2}/\alpha_1)}{\log 2} \geq 0.$ The condition \eqref{eq_proof_th2_5_final_condition} is implied by 
\begin{align}
 2^{\depth} \depth^{- \frac{d-1}{2 }} n^{- \frac{1}{1 - 2 \log_2(\alpha_1)}}\to \infty. \label{eq_proof_th2_5_final_condition_bis}   
\end{align}
Using the fact that $\log (1-x) \leq -x$ for all $x>1$, we have
\begin{align}
\frac{1}{1 - 2 \log_2(\alpha_1)} \leq \frac{d \log 2}{1 + d \log 2}.
\end{align}
Thus, condition \eqref{eq_proof_th2_5_final_condition_bis} is implied by 
\begin{align}
 2^{\depth} \depth^{- \frac{d-1}{2 }} n^{- \frac{d \log 2}{1 + d \log 2}}\to \infty.  
\end{align}
This proves that the second condition given in \Cref{eq:secThA1} in \Cref{th: thmaya} is satisfies and completes the proof.

\end{proof}

\subsection{Preliminary results}\label{subsec:pre_results}
The next Lemma characterizes inverse moments of a binomial random variable.  

\begin{lem}[\textbf{Inverse moment of a binomial distribution}]\label{lem:cribari}
Let $Z\sim \mathcal{B}(n,p)$, with $n\geq 1$ and $p>0$. Then the following two assertions holds 

\begin{itemize}
    \item [i)]The first inverse moment of $Z$ satisfies the identity $$\mathbb{E}\Big[\frac{1}{1+Z}\Big] = \frac{1-(1-p)^{n+1}}{(n+1)p};$$
    
    \item [ii)] For all $\alpha >0$, the $\alpha$-th moment  asymptotically satisfies 
$$\mathbb{E}\Big[\frac{1}{(1+Z)^\alpha}\Big] = (np)^{-\alpha} + O\Big(n^{-(\alpha+1)}\Big).$$
\end{itemize}
\end{lem}
\begin{proof}[Proof of \Cref{lem:cribari}.]
See Lemma 11 in \cite{biau2012analysis} and \cite{cribari2000note}. 
\end{proof}

The next result provides upper bounds  of the bias and the variance for the diameter of the random leaf $L_{\U}(\bm{x})$ in the context of RF that are built from centered random trees as defined in \Cref{subsec:Notation}.

\begin{lem}[\textbf{Bias and variance of $Diam(L_{\U}(\bm{x})$}]\label{lem:diam}

For all $\bm x \in [0,1]^d$, the diameter $\textup{Diam}(L_{\U}(\bm{x}))$ of a leaf $L_{\U}(\bm{x})$ of a CRT as defined in \Cref{subsec:Notation} satisfies
\begin{itemize}
    \item [i)] $\E \Big[ \textup{Diam}(L_{\U}(\bm{x}))\Big]\leq d\Big(1-\frac{1}{2d}\Big)^{\depth}=O(\alpha_{1}^k);$
    
    \item [ii)]$\E \Big[ \textup{Diam}(L_{\U}(\bm{x}))^2\Big] \leq d\Big(1-\frac{3}{4d}\Big)^{\depth}=O(\alpha_{2}^k).$
\end{itemize}

\end{lem}

\begin{proof}[Proof of \Cref{lem:diam}.] Fix  $\bm x \in [0,1]^d$. 
\begin{itemize}
\item[i)] We proceed similarly to  \cite{biau2012analysis}. Let $V_{nj}({\bf x},\U) $  be the size of the $j$-th dimension of the cell $L_{\U}(\bm{x})$ containing $\bm{x}$. Since at each node, the $j$-th coordinate is choosen uniformly at random  with probability $1/d$, all $V_{nj}({\bf x},\U) $ have the same distribution. Therefore, $$\E \Big[ \textup{Diam}(L_{\U}(\bm{x}))\Big]=    \E \Big[ \max_{1\le j\le d} V_{nj}(\bm{x},\U)\Big]\\  \leq d \E \Big[V_{n1}(\bm{x},\U)\Big].
    $$
Besides, for centered random forest, since a cell is split $\depth$ times in its middle, we have
$$V_{n1}(\bm{x},\U)\overset{d}{=}2^{-K_{n1}(\bm{x},\U)}$$
where $K_{n1}(\bm{x},\U)$ is the
number of times the cell containing $\bm{x}$ is split along the first  coordinate. Thus, $K_{n1}(\bm{x},\U) \sim \mathcal{B}(\depth, 1/d)$. Denoting by $B_{\ell}$ the random variable that equals one if the split at level $\ell$ in the cell containing $\bm x$ occurs along the first coordinate, we have that $B_1, \ldots, B_{\depth}$ are i.i.d.\ and distributed as Bernoulli random variables with parameter $1/d$.  Thus,  

\begin{align*}
    \E \Big[ \textup{Diam}(L_{\U}(\bm{x}))\Big]&\leq      \E d \Big[2^{-K_{n1}(\bm{x},\U)}\Big]\\
     & = d   \E \Big[2^{-\sum_{\ell=1}^{\depth}B_{\ell}}   \Big]  \\
     & = d  \Big( \E \Big[2^{-B_1}  \Big] \Big)^{\depth} \\
     & = d\Big(1-\dfrac1d+\frac{1}{2d}\Big)^{\depth}\\
       & = d\Big(1-\frac{1}{2d}\Big)^{\depth}
\end{align*}
concluding the proof of the first statement. 

\item[ii)] Similar calculations show that 
\begin{align*}
\E \Big[ \textup{Diam}(L_{\U}(\bm{x}))^2\Big] & \leq d\Big(1-\frac{3}{4d}\Big)^{\depth},
\end{align*}
which concludes the proof of the second statement.
\end{itemize}
\end{proof}

	The next lemma provides the expressions of the expectation and variance of $T^s(\bm{x}; \U; \bm Z_{s}) \mid \U$, \textit{i.e.}, the individual trees conditioned on the partition $\U$. 
\begin{lem}
[\textbf{Bias and variance of $T^s(\bm{x}; \U; \bm Z_{s})\mid \U$}]
\label{lem:E_ynn}
Let $T^s(\bm{x}; \U; \bm Z_{s})$ 
be the individual tree at point $\bm x \in [0,1]^d$ as defined in \Cref{eq:T_onesample}. Then  the following two assertions hold:
\begin{enumerate}
    \item[$i)$]  
   $\mathbb{E}  [T^s(\bm{x}; \U; \bm Z_{s}) \mid \U]=\mu_{\depth, \U}(\bm{x})(1-(1-p_{\depth, \U}(\bm{x}))^s)\,;
$
\item[$ii)$] $
   \v(T^s(\bm{x}; \U; \bm Z_{s}) \mid \U)= \dfrac{\sigma^2_{\depth, \U}(\bm x)}{sp_{\depth,\U}(\bm{x})}(1-(1-p_{\depth,\U}(\bm{x}))^s)^2\,.$
  \end{enumerate}
\end{lem}
\begin{proof}[Proof of \Cref{lem:E_ynn}.]
Observe first that given  $N^s=j$  and $\U$, we have $$T^s  \overset{d}{=} \frac{1}{j} \sum_{i=1}^{j} \tilde{Y}_i,$$ 
where $\tilde{Y}_i \overset{i.i.d.}{\sim} \mathcal{L}(Y| X \in L_{\U}(\bm{x}),\U)$ is distributed as a Bernoulli random variable with parameter  $\mu_{\depth, \U}(\bm{x})$. Then
\begin{equation}\label{eq: onepoint}
 \E[T^s| N^s=j,\U]=
\begin{cases}
0 & \text{if } j=0\,, \\
\mu_{\depth, \U}(\bm{x}) & \text{if } j\neq  0\,,
\end{cases}
\end{equation}
and 
\begin{equation}\label{eq: secondpoint}
 \v(T^s | N^s=j,\U)=
\begin{cases}
0 & \text{if } j=0\,, \\
\sigma^2_{\depth, \U}(\bm x)/j & \text{if } j\neq  0\,.
\end{cases}
\end{equation}
\begin{enumerate}
     \item[$i)$] By definition, conditionally on $\U$, since $N^s$  is a sum of $s$ i.i.d.\ random variables $\1{\{X_{i} \in L_{\U}(\bm{x}) \}}$ that follows a Bernoulli distribution with parameter $p_{\depth,\U}(\bm{x})$ then   $N^s$ is distributed as a Binomial random variable  $\mathcal{B}(s,p_{\depth,\U}(\bm{x}))$. Thus, from \Cref{eq: onepoint}, we have 
 \begin{align*}
		& \mathbb{E}  [T^s\mid \U] \\
		&=\sum_{j=0}^{s} \E\Big[\frac{\sum_{i =1}^s Y_{i}\1{\{X_{i} \in L_{\U}(\bm{x}) \}}}{N^s}  \1{\{N^s>0 \}}| N^s=j, \U\Big]\\
		& \qquad \times \mathbb{P}\Big( N^s=j |\U \Big)\\
  &= \sum_{j=1}^{s} \frac{1}{j} \E\bigg[ \sum_{i=1}^{j}Y_{i}\1{\{X_{i} \in L_{\U}(\bm{x}) \}} \mid N^s=j, \U\bigg]\mathbb{P}( N^s=j | \U )\\
  &= \mu_{\depth, \U}(\bm{x})\sum_{j=1}^{s}\mathbb{P}( N^s=j | \U )\\
  &=\mu_{\depth, \U}(\bm{x})(1-\mathbb{P}(N^s=0 \mid \U))\\
  &=\mu_{\depth, \U}(\bm{x})(1-(1-p_{\depth,\U}(\bm{x}))^s),
\end{align*}
which is the desired result.

\item[$ii)$]  From Equation \eqref{eq: secondpoint},  we have
\begin{align*}
    \E \Big[ \v (T^s| N^s, \U)\Big]&= \E\Big[\E\Big[ \sigma^2_{\depth, \U}\bm x)\dfrac{\mathds{1}{\{ N^s>0\}}}{N^s} \mid \U\Big] \Big].
\end{align*}
Applying the same reasoning as in item i), we have that the random variable $N^{s-1}$ conditionally on $\U$ is distributed as  $\mathcal{B}(s-1,p_{\depth,\U}(\bm{x}))$. Using \cref{lem:cribari} and exchangeability, it holds
\begin{align*}
&\E\Big[\sigma^2_{\depth, \U}(\bm x)\dfrac{\mathds{1}{\{ N^s>0\}}}{N^s} \mid \U \Big]\\
&=  \E\Big[\sigma^2_{\depth, \U}(\bm x)\dfrac{1}{N^s}|  N^s>0, \U\Big]\mathbb{P}(N^{s}>0 \mid \U)\\
		&= \sigma^2_{\depth, \U}(\bm x)  \E\Big[\dfrac{1}{1+ N^{s-1}}|   \U\Big]\mathbb{P}(N^s>0 \mid \U)\\
		 & =\sigma^2_{\depth, \U}(\bm x)  \E\Big[\dfrac{1}{1+ N^{s-1}}\Big](1-\mathbb{P}(N^s=0 \mid \U))\\
		&= \sigma^2_{\depth, \U}(\bm x) \frac{1}{sp_{\depth,\U}(\bm{x})} \Big[1-(1-p_{\depth,\U}(\bm{x}))^{s}\Big]\Big(1-(1-p_{\omega_s,\U}(\bm{x}))^{s}\Big)\\
		&=   \frac{\sigma^2_{\depth, \U}(\bm x)}{sp_{\depth,\U}(\bm{x})} \Big(1-(1-p_{\depth,\U}(\bm{x}))^{s}\Big)^2,
\end{align*}
which concludes the proof.
\end{enumerate}
\end{proof}

\subsection{Calculating  $\mathbb{E}[T^s]$ and  $\v(T^s)$}\label{subsec:Control_vETs}
The calculation of $\mathbb{E}[T^s]$ and $\v(T^s)$ requires the following intermediary result that provides the expectation and the variance of $\mu_{\depth,\U}(\bm{x})$.

\begin{lem}[\textbf{Bias and Variance of $\mu_{\depth,\U}(\bm{x})$}]\label{lem:bias_mutilde}
Let $\bm x \in [0,1]^d$ and assume that Condition \textbf{(H1)}  holds. Then,   as $n,s,k \to \infty$, we have  
\begin{itemize}
      \item[$i)$]$\E[\mu_{\depth,\U}(\bm{x})]=\mu(\bm x)+O\left(\alpha_1^{\depth}\right)$;
      \item[$ii)$] $\v(\mu_{\depth,\U}(\bm{x}) ) =O\left(\alpha_2^{\depth}\right)$;
    \item[$iii)$] $\E[\mu_{\depth,\U}^2 (\bm{x})]= \mu^2(\bm x)+O\left(\alpha_1^{\depth}\right).$
  \end{itemize}
\end{lem}
\begin{proof}[Proof of \Cref{lem:bias_mutilde}.]
 Fix  $\bm x \in [0,1]^d$. 
\begin{itemize}
\item[$i)$] Under Condition \textbf{(H1)}, we have
\begin{eqnarray*}
		|\E[\mu_{\depth,\U}(\bm x)] -  \mu(\bm{x}) | &=&  \bigg|\E\bigg[\frac{\mathbb{E}\big[ \int_{ L_{\U}(\bm{x})} \mu(\bm{y})d\mathbb{P}_X(\bm{y})\mid \U\big]}{\mathbb{P}(X \in L_{\U}(\bm{x}) \mid \U) } - \mu(\bm{x})\bigg]\bigg|\\ 
		&\leq & \E\bigg[\frac{\mathbb{E}\big[ \int_{ L_{\U}(\bm{x})} | \mu(\bm{y}) - \mu(\bm{x}) | d\mathbb{P}_X(\bm{y})\mid \U\big]}{\mathbb{P}(X \in L_{\U}(\bm{x})\mid \U) }\bigg] \\
		&\leq & \E\bigg[\frac{\mathbb{E}\big[ \int_{ L_{\U}(\bm{x})}  L||   \bm{y} - \bm{x} \| _{\infty} d\mathbb{P}_X(\bm{y})\mid \U\big]}{\mathbb{P}(X \in L_{\U}(\bm{x})\mid \U) } \bigg]\\
  		&\leq & L\,\E\bigg[\frac{\textup{ Diam}(L_{\U}(\bm{x}))\mathbb{E}\big[\1{\{X\in L_{\U}(\bm{x})\}}\mid \U \big]}{\mathbb{P}(X \in L_{\U}(\bm{x})) }\bigg] \\
		&\leq & L\,  \mathbb{E}\big[ \textup{Diam}(L_{\U}(\bm{x}))\big].
	\end{eqnarray*}

The desired result follows from a direct application of \cref{lem:diam} and letting $n,s,k \to \infty$. 

\item[$ii)$] 
From \Cref{lem:E_ynn}  and \Cref{lem:diam}, under Condition \textbf{(H1)}, we have

\begin{align*}
   \v(\mu_{\depth,\U}(\bm{x}) )  &= \E \Big[\big(\mu_{\depth,\U}(\bm{x})- \E\big( \mu_{\depth,\U}(\bm{x})\big)\big)^2\Big]\\
  &= \E \Big[\big(\mu_{\depth,\U}(\bm{x})- \mu (\bm{x}) +O(\alpha_1^{\depth}) \big)^2\Big]\\
    & =\E \Big[\big(\mu_{\depth,\U}(\bm{x})-\mu (\bm{x})\big)^2\Big] +O(\alpha_1^{2\depth})+2 O(\alpha_1^{\depth}) \E \Big[\mu_{\depth,\U}(\bm{x})-\mu (\bm{x})\Big]\\
    & \leq L^2 \E \Big[ \textup{Diam}(L_{\U}(\bm{x}))^2\Big]+ O(\alpha_1^{2 \depth}) \\
    & =O(\alpha_2^{ \depth}) + O(\alpha_1^{2 \depth})\\
    &= O(\alpha_2^{ \depth}),
\end{align*}
since $\alpha_2 \geq \alpha_1^2.$

\item[$iii)$] Under Condition \textbf{(H1)} and the first two statements of this lemma, we have
\begin{align*}
    \E[\mu_{\depth,\U}^2 (\bm{x})]&= \v( \mu_{\depth,\U}(\bm{x}))+ \E[\mu_{\depth,\U}(\bm{x})]^2\\
    &= O(\alpha_2^{\depth}) +(\mu(\bm x)+O(\alpha_1^{\depth}))^2\\
    &= \mu(\bm x)^2+ O(\alpha_2^{\depth}) + O(\alpha_1^{\depth}) + O(\alpha_1^{2 \depth}) \\
    & = \mu(\bm x)^2+O(\alpha_1^{\depth}),
\end{align*}
since  $ \alpha_1 \geq  \alpha_2$. This  concludes the proof.
\end{itemize}
\end{proof}

\begin{lem}\label{rem:grandTau}
Let $\bm x \in [0,1]^d$ and grant {\bf (H0)}. Thus, $p_{\depth,\U}(\bm{x}) = 2^{- \depth}$ and as $n,s,k \to \infty$ we have
\begin{align*}
(1-p_{\depth,\U}(\bm{x}))^s=O(e^{-s 2^{- \depth}}).    
\end{align*}
\end{lem}

\begin{proof}[Proof of \Cref{rem:grandTau}.]
By definition, since $X$ is uniform on $[0,1]^d$ according to {\bf (H0)}, $p_{\depth,\U}(\bm{x}) = 2^{- \depth}.$ Thus, 
\begin{align}
 (1-p_{\depth,\U}(\bm{x}))^s & =   \exp \left( s \log (1-p_{\depth,\U}(\bm{x})) \right) \\
 & \leq   \exp \left( - s  p_{\depth,\U}(\bm{x}) \right) \\
 & \leq   \exp \left( - s 2^{- \depth} \right).
\end{align}
\end{proof}

We are now ready to provide the bias and the variance of the individual trees $T^s(\bm{x}; \U; \bm Z_{s})$.

\begin{prp}[\textbf{Bias and Variance of $T^s(\bm{x}; \U; \bm Z_{s})$}]\label{lem:bias_Ts} Let $\bm x \in [0,1]^d$ and assume Conditions \textbf{(H0)}, \textbf{(H1)} and \textbf{(G1)}  hold. Then, as $n,s,k \to \infty$, we have
\begin{itemize}
      \item[$i)$] $\E[T^s(\bm{x}; \U; \bm Z_{s})]=\mu(\bm x)  +O(\alpha_1^k);$
    \item[$ii)$] $
    \v(T^s(\bm{x}; \U; \bm Z_{s}))=\dfrac{2^{\depth} }{s}\Big(\mu(\bm{x})(1-\mu(\bm{x}))+ O(\alpha_1^k)\Big)+ O(\alpha_2^k)\,.
$
  \end{itemize}
\end{prp}

\begin{proof}[Proof of \Cref{lem:bias_Ts}.]
Let $\bm x \in [0,1]^d$.

\begin{itemize}
\item[$i)$] From \Cref{lem:E_ynn} $(i)$,  \Cref{lem:bias_mutilde} $(i)$ and \Cref{rem:grandTau}, under Condition \textbf{(H1)} we have \begin{align}
\E[T^s]& =\E[\E[T^s\mid \U]]\\
& = \E[\mu_{\depth,\U}(\bm x)]+O(e^{-s 2^{- \depth}}) \\
& =\mu(\bm x)+O(\alpha_1^k)+O(e^{-s 2^{- \depth}}).
\end{align}
Besides, according to \textbf{(G1)}, $$O(e^{-s 2^{- \depth}})=O(\alpha_1^k),$$ which concludes the proof.

\item[$ii)$] 
By the variance decomposition formula
  \begin{align}\label{eq:Var_T_s}
    \v(T^s)&= \E [ \v (T^s| \U)] + \v(\E[T^s| \U] )\,.
\end{align}
According to the definition of $\sigma^2_{\depth, \U}(\bm x)$ and \Cref{lem:bias_mutilde} $(i)$ and $(iii)$, we have 
\begin{align}
 \E [\sigma^2_{\depth, \U}(\bm x)]=\mu(\bm x)-\mu(\bm x)^2+O(\alpha_1^k).  
\end{align}
According to \Cref{rem:grandTau}, we have  $$\Big(1-(1-p_{\depth,\U}(\bm{x}))^{s}\Big)^2=1+O(e^{-s2^{-\depth}}).$$
Thus, using  \cref{lem:E_ynn} ii) and the fact that $p_{\depth,\U}(\bm{x}) = 2^{- \depth}$ under \textbf{(H0)}, we have 
\begin{align*}
   \E [ \v (T^s| \U)] & = \E \Big[\sigma^2_{\depth, \U}(\bm x) \frac{2^{\depth}}{s} \Big(1-(1-p_{\depth,\U}(\bm{x}))^{s}\Big)^2\Big]\\
   & = \frac{2^{\depth}}{s}\E \Big[\sigma^2_{\U}(\bm x) \big( 1+O(e^{-s2^{-\depth}})\big)\Big]\\
   &= \frac{2^{\depth}}{s} \Big(\mu(\bm{x})(1-\mu(\bm{x}))+ O(\alpha_1^k) +O(e^{-s2^{-\depth}})\Big)\\
   & =\frac{2^{\depth}}{s} \Big(\mu(\bm{x})(1-\mu(\bm{x}))+ O(\alpha_1^k) \Big),
   \end{align*}
   since, by assumption {\bf (G1)}, $O(e^{-s 2^{- \depth}})=O(\alpha_1^k).$
   Besides 
\begin{align*}
   \v(\E[T^s|\U] )= \v\Big(\mu_{\depth, \U}(\bm{x})(1-(1-p_{\depth,\U}(\bm{x}))^s)\Big)= O(\alpha_2^k),
\end{align*}
since $\v(\mu_{\depth, \U}(\bm{x}))=O(\alpha_2^k)$ according to \Cref{lem:bias_mutilde} $(ii)$, concluding the proof. 
\end{itemize}
\end{proof}

\subsection{Equivalent of $V_1^s$}\label{subsec:equiv_V1s}

\begin{prp}[\textbf{Equivalent of $V_1^s$}]\label{prop:v1} 
Let $\bm x \in [0,1]^d$ with  $d\ge 2$ and  assume Conditions \textbf{(H0)} and  \textbf{(H1)}  hold. Then  as $n,s,k \to \infty$ we have
\begin{align*}
V_1^s = \frac{C(d) 2^{\depth}}{s^2 \depth^{(d-1)/2}}\mu(\bm x)(1-\mu(\bm x))+o\Big(\frac{(2 \sqrt{\alpha_2})^\depth}{s^2}\Big)+O\Big(\frac{4^\depth e^{-s2^{-\depth}}}{s^2}\Big)
\,,
\end{align*}
with
\[
\dfrac{2\Gamma(d-1)}{(\log 2)^{d-1}\Gamma((d-1)/2)}d^{-(d-1)/2}\le C(d)\le\dfrac{2\Gamma(d-1)}{(\log 2)^{d-1}\Gamma((d-1)/2)}2^{-(d-1)/2}\,.
\]
\end{prp}
\begin{proof}[Proof of \Cref{prop:v1}.]
Denote $\bm{Z}_{\{2,\ldots,s\}}:=(Z_2,\ldots,Z_s)$ and $\bm{X}_{\{2,\ldots,s\}}:=(X_2,\ldots,X_s)$. From the definition of the individual trees given in  \Cref{eq:T_onesample},  we have
\begin{align*}
T^s(\bm{x}; \U; \bm Z_{s}) & =T^{s-1}(\bm{x}; \U; \bm Z_{\{2,\ldots,s\}})\1(X_1 \notin L_{\U}(\bm{x}))\\
& \quad  +\tilde{T}^{s-1}_1(\bm{x}; \U; \bm Z_{s})\1(X_1 \in L_{\U}(\bm{x}))\,,
\end{align*}
where
\begin{equation*}
		T^{s-1}(\bm{x}; \U; \bm Z_{\{2,\ldots,s\}})= \frac{\sum_{i=2}^s Y_{i}\1{\{X_{i} \in L_{\U}(\bm{x}) \}}}{N_{L_{\U}(\bm{x})}( \bm X_{\{2,\ldots,s\}})}  
	\end{equation*}
and
\[
\tilde{T}^{s-1}_1(\bm{x}; \U; \bm Z_{s}) = \frac{1}{N_{L_{\U}(\bm{x})}(\bm{X}_{\{2,\ldots,s\}})+1}\Big(Y_1+\sum_{i= 2}^s Y_{i}\1{\{X_{i} \in L_{\U}(\bm{x}) \}}\Big)\,.
\]
Taking expectation conditionally
to $Z_1, \U$, we have  
\begin{align}
  	\mathbb{E}[T^s(\bm{x}; \U; \bm Z_{s}) \ | \ Z_1, \U]&= \E[T^{s-1}(\bm{x}; \U; \bm Z_{\{2,\ldots, s\}})|Z_1, \U]\1(X_1 \notin L_{\U}(\bm{x})) \nonumber \\
  	& \quad +   \E[\tilde{T}^{s-1}_1(\bm{x}; \U; \bm Z_{s})|Z_1, \U]\1(X_1 \in L_{\U}(\bm{x})) \nonumber\\
  &= \E[T^{s-1}(\bm{x}; \U; \bm Z_{\{2,\ldots,s\}})| \U]\1(X_1 \notin L_{\U}(\bm{x})) \nonumber\\
  	& \quad +   \E[\tilde{T}^{s-1}_1(\bm{x}; \U; \bm Z_{s})|Y_1, \U]\1(X_1 \in L_{\U}(\bm{x})). \label{proof_eq_decomp_global}
\end{align}
Similarly to \Cref{rem:grandTau}, it is easy to see that \begin{align*}
(1-p_{\depth,\U}(\bm{x}))^s=O(e^{-sp_{\depth,\U}(\bm{x})})).    
\end{align*}
Together with \cref{lem:E_ynn}, we obtain 
\begin{align}
    \E [T^{s-1}(\bm{x}; \U; \bm Z_{\{2,\ldots, s\}})| \U ]& =\mu_{\depth,\U}(\bm{x}) \Big[1-(1-p_{\depth,\U}(\bm{x}))^{s-1}\Big] \\
    & =\mu_{\depth,\U}(\bm{x})(1+O(e^{-sp_{\depth,\U}(\bm{x})}))\,. \label{proof_eq_decomp1}
\end{align}
Besides, we have
	\begin{align*}
	\E[\tilde{T}^{s-1}_1(\bm{x}; \U; \bm Z_{s})|Z_1, \U] &= \underbrace{\mathbb{E}\Big[\frac{1}{N_{L_{\U}(\bm{x})}\Big(\bm{X}_{\{2,\ldots,s\}}\Big)+1} \mid \U\Big]}_{=:a(\U)} Y_1 \\
		& \quad + \underbrace{\E\Big[\frac{1}{1+ N_{L_{\U}(\bm{x})}\Big(\bm{X}_{\{2,\ldots,s\}}\Big)} \sum_{i=2}^s Y_i\mathds{1}_{\{X_i \in L_{\U}(\bm{x}) \}} \Big | \U\Big]}_{=:b(\U)}.
	\end{align*}
	Using \Cref{lem:cribari} i) yields
	\begin{align*}
	a(\U)& = \E\Big[\frac{1}{N_{L_{\U}(\bm{x})}(\bm{X}_{\{2,\ldots,s\}})+1}\mid \U\Big] \\
	& 	= \frac{1}{sp_{\depth,\U}(\bm{x})} \Big[1-(1-p_{\depth,\U}(\bm{x}))^{s}\Big] \\
	& =\frac{1}{sp_{\depth,\U}(\bm{x})}(1+O(e^{-sp_{\depth,\U}(\bm{x})}))
	\end{align*}
and
\begin{align*}
   b(\U)&=  \E\Big[\frac{1}{1+ N_{L_{\U}(\bm{x})}(\bm{X}_{\{2,\ldots,s\}})} \sum_{i =2}^s Y_i\mathds{1}_{\{X_i \in L_{\U}(\bm{x}) \}} \Big | \U\Big]\\
    &= \mu_\U(\bm{x}) \sum_{k=0}^{s-1} \frac{k}{1+k} \mathbb{P}\Big( N_{L_{\U}(\bm{x})}(\bm{X}_{\{2,\ldots,s\}})=k\mid \U\Big)\\
    & = \mu_\U(\bm{x}) \E\Big[\frac{N_{L_{\U}(\bm{x})}(\bm{X}_{\{2,\ldots,s\}})}{1+N_{L_{\U}(\bm{x})}(\bm{X}_{\{2,\ldots,s\}})}\mid \U \Big]\\
    & =\mu_\U(\bm{x}) \Big( 1- \E\Big[\frac{1}{1+N_{L_{\U}(\bm{x})}(\bm{X}_{\{2,\ldots,s\}})}\mid \U\Big]\Big) \\
    &=\mu_\U(\bm x) \Big(1-\frac{1}{sp_{\depth,\U}(\bm{x})}(1+O(e^{-sp_{\depth,\U}(\bm{x})}))\Big).
    \end{align*}
Thus, 
\begin{align}
\E[\tilde{T}^{s-1}_1(\bm{x}; \U; \bm Z_{s})|Z_1, \U] &=  \frac{Y_1}{sp_{\depth,\U}(\bm{x})} + \mu_\U(\bm x) \Big(1-\frac{1}{sp_{\depth,\U}(\bm{x})}\Big) \\
& \quad + (Y_1 + \mu_\U(\bm x) )O\Big(\frac{e^{-sp_{\depth,\U}(\bm{x})}}{sp_{\depth,\U}(\bm{x})}\Big).  \label{proof_eq_decomp3}
\end{align}
Gathering \eqref{proof_eq_decomp1} and \eqref{proof_eq_decomp3} in \eqref{proof_eq_decomp_global}, we have
\begin{align}\label{eq:EsT}
  & \mathbb{E}[T^s(\bm{x}; \U; \bm Z_{s}) \ | \ Z_1, \U] \nonumber \\
  &=
  	\mu_{\depth,\U}(\bm{x})(1+O(e^{-sp_{\depth,\U}(\bm{x})})) \1(X_1 \notin L_{\U}(\bm{x})) \nonumber\\
  	& \quad +   \Big(\frac{1}{sp_{\depth,\U}(\bm{x})}Y_1 + \mu_{\depth,\U}(\bm x)\Big(1-\dfrac{1}{sp_{\depth,\U}(\bm{x})}\Big)\Big)\1(X_1 \in L_{\U}(\bm{x})) \nonumber\\
  	& \quad +  (Y_1 + \mu_{\depth,\U}(\bm x) )O\Big(\frac{e^{-sp_{\depth,\U}(\bm{x})}}{sp_{\depth,\U}(\bm{x})}\Big) \1(X_1 \in L_{\U}(\bm{x})) \nonumber \\
  	&=  \mu_{\depth,\U}(\bm{x})+\dfrac{1}{sp_{\depth,\U}(\bm{x})} (Y_1-\mu_{\depth,\U}(\bm x))\1(X_1 \in L_{\U}(\bm{x})) \nonumber \\
  	& \quad +(Y_1 + \mu_{\depth,\U}(\bm x) )O\Big(\frac{e^{-sp_{\depth,\U}(\bm{x})}}{sp_{\depth,\U}(\bm{x})}\Big).
\end{align}
Recalling that 
$\mu_{\depth,\U}(\bm{x})=\P(Y=1|X \in L_{\U}(\bm{x}),\U ),$ let
$$
\varepsilon_{\U}(\bm{x}) = \frac{ \mu_{\depth,\U}(\bm{x}) - \mu(\bm x)}{\textup{Diam}(L_{\U}(\bm{x}))}.
$$
Thus, under Condition {\bf (H1)}, we have $|\varepsilon_\U|\leq L$ almost surely. Using the fact that under {\bf (H0)}  $p_{\depth,\U}(\bm{x})=2^{-\depth}$, and since $\mu_{\depth,\U}(\bm{x})$ does not depend on $Z_1$, we have, 
\begin{align*}
V_1^s&=\v(\E[T^s(\bm{x}; \U; \bm Z_{s}) \ | \ Z_1])\\
&=\v(\E[\mathbb{E}[T^s(\bm{x}; \U; \bm Z_{s}) \ | \ Z_1, \U]|Z_1])\\
&=\v\Big(\E[ \mu_{\depth,\U}(\bm{x})]+\frac{2^{\depth}}{s} \E[(Y_1-\mu_{\depth,\U}(\bm x))\1(X_1 \in L_{\U}(\bm{x}))|Z_1] \\
& \qquad +(Y_1 + \E[\mu_{\depth,\U}(\bm x)] )O\Big(\frac{e^{-s2^{- \depth}}}{s 2^{- \depth}}\Big)\Big)\\
&= \Big(\dfrac{2^{\depth}}{s}\Big)^2\v(\E[(Y_1- \mu_{\depth,\U}(\bm{x}))\1\{X_1\in L_\U(\bm{x})\}|Z_1]+Y_1 O(e^{-s2^{- \depth}}))\\
&=\Big(\dfrac{2^{\depth}}{s}\Big)^2\v((Y_1-\mu(X_1))\1(X_1\in L_\U(\bm{x})) \\
&\quad +\E[\varepsilon_{\U}(X_1) \textup{Diam}(L_{\U}(\bm{x})))\1(X_1\in L_\U(\bm{x}))|Z_1]+Y_1O(e^{-s2^{-\depth}}))\\
&= \Big(\dfrac{2^{\depth}}{s}\Big)^2\big(\v((Y_1-\mu(X_1))\P(X_1\in L_\U(\bm{x})|X_1))\\
& \quad+ 2\cov\big((Y_1-\mu(X_1))\1(X_1\in L_\U(\bm{x})), \E[\varepsilon_{\U}(X_1) \textup{Diam}(L_{\U}(\bm{x}))\1(X_1\in L_\U(\bm{x}))|X_1]\big)\\
&\quad + \v\big(\E[\varepsilon_{\U}(X_1)\textup{Diam}(L_{\U}(\bm{x}))\1(X_1\in L_\U(\bm{x}))|X_1]\big)+O(e^{-s2^{-\depth}})\big)\,,
\end{align*}
since each term in the sum of the previous line are a.s. bounded by a constant. Moreover the covariance term is null since $\E[Y_1|X_1]=\mu(X_1)$ and the variance term is controlled by an application of \Cref{prp:asequiv} ii)  
\begin{align*}
    &\v\big(\E[\varepsilon_{\U}(X_1) \textup{Diam}(L_{\U}(\bm{x}))\1(X_1\in L_\U(\bm{x}))|X_1]\big)\\
    & \qquad\leq \E \big[(\E[\textup{Diam}(L_{\U}(\bm{x}))\1(X_1\in L_\U(\bm{x})|X_1])^2 \big]\\
    &\qquad = O((\alpha_2/2)^k).
\end{align*}
Thus, we obtain
\begin{align*}
& V_1^s\\
&= \Big(\dfrac{2^{\depth}}{s}\Big)^2 (\v((Y_1-\mu( X_1))\P(X_1\in L_\U(\bm{x})|X_1))+O((\alpha_2/2)^k) +O(e^{-s2^{-\depth}}))\\
&= \Big(\dfrac{2^{\depth}}{s}\Big)^2 (\E[\v(Y_1|X_1)\P(X_1\in L_\U(\bm{x})|X_1)^2]+O((\alpha_2/2)^k)+O(e^{-s2^{-\depth}}))\\
&= \Big(\dfrac{2^{\depth}}{s}\Big)^2 (\E[\mu(X_1)(1-\mu(X_1))\P(X_1\in L_\U(\bm{x})|X_1)^2] +O((\alpha_2/2)^k)+O(e^{-s2^{- \depth}}))
\,.
\end{align*}
Recall that $\mu$ is $L$-Lipschitz and apply Cauchy-Schwarz inequality to rewrite the main term in the right hand side such as 
\begin{align*}
\E& [\mu(X_1)(1-\mu(X_1))\P(X_1\in L_\U(\bm{x})|X_1)^2]\\
&=\E [\mu(\bm{x})(1-\mu(\bm{x}))\P(X_1\in L_\U(\bm{x})|X_1)^2]\\
&\qquad + O(L \,\E [\E[\|X_1-\bm x\|_\infty\1(X_1\in L_\U(\bm{x}))|X_1]\P(X_1\in L_\U(\bm{x})|X_1))]\\
&=\mu(\bm{x})(1-\mu(\bm{x}))\E [\P(X_1\in L_\U(\bm{x})|X_1)^2]\\
&\qquad + O(\E [\E[\textup{Diam}(L_{\U}(\bm{x}))\1(X_1\in L_\U(\bm{x}))|X_1]\P(X_1\in L_\U(\bm{x})|X_1))]\\
&=\mu(\bm x)(1-\mu(\bm x))\E[ \P(X_1\in L_\U(\bm{x})|X_1)^2]\\
&\qquad+ O\big(\sqrt{\E [\E[\textup{Diam}(L_{\U}(\bm{x}))\1(X_1\in L_\U(\bm{x}))|X_1]^2]\E[ \P(X_1\in L_\U(\bm{x})|X_1)^2]}\big).
\end{align*}
According to Proposition \ref{prp:asequiv}, the second term is $o((\sqrt{\alpha_2}/2)^k)$. Thus, according to Proposition \ref{prp:asequiv} again, since $\alpha_2 \leq 1$,
\begin{align*}
V_1^s & =  \Big(\dfrac{2^{\depth}}{s}\Big)^2 (\mu({\bf x}) (1 - \mu({\bf x})) \E[ \P(X_1\in L_\U(\bm{x})|X_1)^2]+o((\sqrt{\alpha_2}/2)^k)\\& \quad +O(e^{-s2^{-\depth}}))
\,.
\end{align*}
A last application of Proposition \ref{prp:asequiv} leads to the following equivalent 
\begin{align*}
V_1^s = \frac{C(d) 2^{\depth}}{s^2 \depth^{(d-1)/2}}\mu(\bm x)(1-\mu(\bm x))+o\Big(\frac{(2 \sqrt{\alpha_2})^\depth}{s^2}\Big)+O\Big(\frac{4^\depth e^{-s2^{-\depth}}}{s^2}\Big)
\,.
\end{align*}
\end{proof}
The key step in the proof above is the following refinement of Lemma B.4 by \citet{arnould2023interpolation} which establishes an equivalent of $\E[\P(X_1\in L_\U(\bm{x})|X_1)^2]$.

\begin{prp}[\textbf{Kernel equivalent}]\label{prp:asequiv}
Let $\bm x \in [0,1]^d$ with $d\ge 2$. If Condition \textbf{(H1)}  holds we have
\begin{itemize}
\item [i)]$
\E \Big[\P(X_1\in L_\U(\bm{x})|X_1)^2 \Big] \sim  \dfrac{C(d)}{2^{\depth} \depth^{(d-1)/2}} \qquad k\to \infty\,,
$  with
\[
\dfrac{2\Gamma(d-1)}{(\log 2)^{d-1}\Gamma((d-1)/2)}d^{-(d-1)/2}\le C(d)\le\dfrac{2\Gamma(d-1)}{(\log 2)^{d-1}\Gamma((d-1)/2)}2^{-(d-1)/2};
\]
\item [ii)]$
\E \big[(\E[\textup{Diam}(L_{\U}(\bm{x}))\1(X_1\in L_\U(\bm{x})|X_1])^2 \big]=
O((\alpha_2/2)^k)\,,\qquad k\to \infty\,.$
\end{itemize}
\end{prp}
\begin{proof}[Proof of \Cref{prp:asequiv}.]
i) Using the distribution of the splits in centered trees as in \citet{arnould2023interpolation}, we obtain
\begin{align*}
\E \Big[\P(X_1\in L_\U(\bm{x})|X_1)^2 \Big] &= \E \Big[ \Big( \sum_{\substack{k_1, \ldots, k_d \\ \sum_{j=1}^d k_j = k}} \frac{k!}{k_1! \ldots k_d!} \Big( \frac{1}{d} \Big)^k \prod_{j=1}^d \mathds{1}_{\lceil 2^{k_j} x^{(j)} \rceil  = \lceil 2^{k_j} X_1^{(j)} \rceil } \Big)^2 \Big]\\
&= \sum_{\substack{(k_1, \ldots, k_d)\\ (l_1, \ldots, l_d)\\ \sum_{j=1}^d k_j = \sum_{j=1}^d l_j =k}} \frac{k!}{k_1! \ldots k_d!} \frac{l!}{l_1! \ldots l_d!} \Big( \frac{1}{d} \Big)^{2k} \\
&\times \prod_{j=1}^d \P(\lceil 2^{k_j} x^{(j)} \rceil  = \lceil 2^{k_j} X_1^{(j)} \rceil ,\lceil 2^{l_j} x^{(j)} \rceil  = \lceil 2^{l_j} X_1^{(j)} \rceil ) .
\end{align*}
Moreover for all $j$, \, 
\[
\P(\lceil 2^{k_j} x^{(j)} \rceil  = \lceil 2^{k_j} X_1^{(j)} \rceil ,\lceil 2^{l_j} x^{(j)} \rceil  = \lceil 2^{l_j} X_1^{(j)} \rceil )= 2^{-\max(k_j,l_j)}.
\]
Since 
\begin{align}
  \sum_{j=1}^d  \max(k_j,l_j) & = \sum_{j=1}^d \Big( \frac{k_j + l_j}{2} + \frac{|k_j - l_j|}{2} \Big)\\
  & = k + \frac{1}{2 }\sum_{j=1}^d  |k_j - l_j|, 
\end{align}

we obtain
\begin{align*}
& \E \Big[\P(X_1\in L_\U(\bm{x})|X_1)^2 \Big] \\
&= \sum_{\substack{(k_1, \ldots, k_d)\\ (l_1, \ldots, l_d)\\ \sum_{j=1}^d k_j = \sum_{j=1}^d l_j =k}} \frac{k!}{k_1! \ldots k_d!} \frac{l!}{l_1! \ldots l_d!} \Big( \frac{1}{d} \Big)^{2k}2^{-\sum_{j=1}^d\max(k_j,l_j)}\\
&= \sum_{\substack{(k_1, \ldots, k_d)\\ (l_1, \ldots, l_d)\\ \sum_{j=1}^d k_j = \sum_{j=1}^d l_j =k}} \frac{k!}{k_1! \ldots k_d!} \frac{l!}{l_1! \ldots l_d!} \Big( \frac{1}{d} \Big)^{2k}2^{-k-  1/2\sum_{j=1}^d|k_j-l_j|}\\
&= 2^{-k}\E\Big[2^{-  1/2\sum_{j=1}^d|K_j-L_j|}\Big]
\end{align*}
for two  independent multinomial variables $(K_1,\ldots,K_d)$ and $(L_1,\ldots,L_d)$, of parameters $k$ and, for all $j$, $p_j=1/d$.

Let $N_1, \hdots, N_d$ be i.i.d. random variables distributed as $\mathcal{N}(0,1/d)$. Let also $\bar N=d^{-1}\sum_{j=1}^d N_j$. Let \begin{align}
\textrm{Lat}_{d-1}(k) = \Big\{k^{-1/2}({\bf k}- k/d): {\bf k}=(k_j)_{1\le j\le d-1}, \; k_j\ge 0\,, \quad \sum_{j=1}^{d-1}k_j\le k\Big\}.
\end{align}
From Lemma 2.1 in \citet{siotani1984asymptotic} (see also \citet{ouimet2021precise}), we have, for all ${\bf n}\in \textrm{Lat}_{d-1}(k)$,
$$
\P(k^{-1/2}(K_j- k/d)_{1\le j\le d-1}={\bf n})=k^{-(d-1)/2}f_{(N_j-\bar{N})_{1\le j\le d-1}}({\bf n})(1+O(\|{\bf n}\|_\infty^3k^{-1/2})).
$$

    The probability mass function above can be rewritten in terms of the
multivariate normal density 
\begin{align*}
    \P(k^{-1/2}(K_j- k/d)_{1\le j\le d-1}={\bf n})&=k^{-(d-1)/2}(2\pi)^{-(d-1)/2}|\Omega|^{-1/2}\exp\big(- \dfrac{1}{2} {\bf n ^\top }\Omega^{-1} {\bf n}\big)\\
    & \quad \times(1+O(\|{\bf n}\|_\infty^3k^{-1/2})),
\end{align*}

    where $\Omega = \text{diag}(p_1, \ldots, p_d) - qq^\top $, $q = (p_1, \ldots, p_d)$ and \(q^\top \) is its transpose.
    
Using the multinomial constraint $\sum_{j=1}^d K_j =k= \sum_{j=1}^d L_j $, 
we might rewrite  $$\sum_{j=1}^d|K_j-L_j|= \sum_{j=1}^{d-1}|K_j-L_j|+|K_d-L_d|=\sum_{j=1}^{d-1}|K_j-L_j|+ |\sum_{j=1}^{d-1}(K_j-L_j)|.$$
Thus, applying arguments of Lemme 2 in \citet{siotani1984asymptotic},  we have 
\begin{align*}
    & \E\Big[2^{-  1/2\sum_{j=1}^d|K_j-L_j|}\Big]\\
        & =\E\Big[2^{-  1/2\sum_{j=1}^{d-1}|K_j-L_j|-1/2|\sum_{j=1}^{d-1}(K_j-L_j)|}\Big]\\
    &=k^{-(d-1)}\sum_{{\bf n} \in \textrm{Lat}_{d-1}(k)}\sum_{{\bf n'} \in \textrm{Lat}_{d-1}(k)}2^{- \sqrt k/2(\sum_{j=1}^{d-1}|{\bf n}_j-{\bf n}_j'|+|\sum_{j=1}^{d-1}{\bf n}_j-{\bf n}_j'|)}f_{(N_j-\bar{N})_{1\le j\le d-1}}({\bf n})\\
    &\qquad\times f_{(N_j-\bar{N})_{1\le j\le d-1}}({\bf n}')(1+O(\|{\bf n}\|_\infty^3k^{-1/2}))^2\\
    &=k^{-(d-1)/2}\sum_{{\bf y} \in \textrm{Diff-Lat}_{d-1}(k) } \sum_{{\bf n} \in \textrm{Lat}_{d-1}(k)} 2^{-  \sqrt k/2(\sum_{j=1}^{d-1}|{\bf y}_j|+|\sum_{j=1}^{d-1}{\bf y}_j|)} k^{-(d-1)/2} f_{(N_j-\bar{N})_{1\le j\le d-1}}({\bf n})\\
    &\qquad\times f_{(N_j-\bar{N})_{1\le j\le d-1}}({\bf y-\bf n})
     (1+O(\|{\bf n}\|_\infty^3k^{-1/2}))^2\,,
\end{align*}
letting $\bf y = \bf n - \bf n'$, where $\textrm{Diff-Lat}_{d-1}(k) = \{ {\bf n - \bf n'}: {\bf n}, {\bf n'} \in \textrm{Lat}_{d-1}(k)\}$. An hypercube of the lattice  $\textrm{Lat}_{d-1}(k)$ has its Lebesgues measure equals to $(k^{-1/2})^{d-1}$. For any $c>0$, let $\mathcal{B}_c = \{ {\bf n} \in \textrm{Lat}_{d-1}(k): \|{\bf n}\|_\infty\le c \sqrt{\log k} \}$. Thus, one can see the above formula as a Riemann sum plus an additional term
\begin{align}\label{eq: approx_multinom}
    & 2^{-k} \E\Big[2^{-  1/2\sum_{j=1}^d|K_j-L_j|}\Big]\nonumber\\
    &=k^{-(d-1)}\sum_{{\bf y} \in \textrm{Diff-Lat}_{d-1}(k) \cap \mathcal{B}_c} \sum_{{\bf n} \in \textrm{Lat}_{d-1}(k) \cap \mathcal{B}_c}
    2^{-  \sqrt k/2(\sum_{j=1}^{d-1}|{\bf y}_j|+|\sum_{j=1}^{d-1}{\bf y}_j|)}f_{(N_j-\bar{N})_{1\le j\le d-1}}({\bf n})\nonumber\\
    &\qquad\times f_{(N_j-\bar{N})_{1\le j\le d-1}}({\bf y-\bf n})(1+O((\log k)^{3/2} k^{-1/2}))^2+ R_k\,,
\end{align}
where 
\begin{align}
|R_k| & \leq C k^{-(d-1)} k^{2d} (1+ O(k^3 k^{-1/2}))^2 \sup_{{\bf n} \in \textrm{Lat}_{d-1}(k) \cap \mathcal{B}_c^c}  f_{(N_j-\bar{N})_{1\le j\le d-1}}({\bf n})\\
& \leq C k^{d+6}  \sup_{{\bf n} \in  \mathcal{B}_c^c}  f_{(N_j-\bar{N})_{1\le j\le d-1}}({\bf n}).
\end{align}

The Gaussian density function $f_{(N_j-\bar{N})_{1\le j\le d-1}}({\bf n})$ is maximized when the quadratic form ${\bf n}^\top \Omega^{-1} {\bf n}$ is minimized. Using the Sherman Morrison's formula,   the inverse of the covariance matrix $\Omega^{-1}$ is 

$$
\Omega^{-1} = \text{diag}(d, \ldots, d) + \frac{d}{d-1} \1\1^\top\,.
$$
This quadratic form ${\bf n}^\top \Omega^{-1} {\bf n}$ is minimized by ${\bf n}=(c \sqrt{\log k}, 0,\ldots,0) \in \mathcal{B}_c^c $ and we compute
\begin{align*}
{\bf n}^\top \Omega^{-1} {\bf n}&=  d c^2\log k + \frac{d}{d-1} c^2\log k\\
&=\left(d + \frac{d}{d-1}\right) c^2\log k\\
&=\frac{d^2}{d-1} c^2\log k.
\end{align*}
Therefore,  the final expression for the supremum is
$$
\sup_{{\bf n} \in  \mathcal{B}_c^c}  f_{(N_j-\bar{N})_{1\le j\le d-1}}({\bf n}) = C_1 k^{-c^2d^2/2(d-1)}
$$
for some constant $C_1>0$.
Consequently, we obtain this control
\begin{align}
|R_k| 
& =o(k^{d+6 - c^2d/2}).
\end{align}
Considering the first term in \Cref{eq: approx_multinom} and using the approximation error of Riemann sum, we have

\begin{align*}
    & 2^{-k}\E\Big[2^{-  1/2\sum_{j=1}^d|K_j-L_j|}\Big]\\
&=\int \int 2^{-  \sqrt k/2(\sum_{j=1}^{d-1}|{\bf y}_j|+|\sum_{j=1}^{d-1}{\bf y}_j|)} f_{(N_j-\bar{N})_{1\le j\le d-1}}({\bf n})f_{(N_j-\bar{N})_{1\le j\le d-1}}({\bf y-\bf n})d{\bf n}d{\bf y}\\
&\qquad\times(1+O((\log k)^3k^{-1/2})) +O( (\log k)^{d}/k^{(d-1)}) + o(k^{d+6 - c^2d/2})\\
&=\int 2^{-  \sqrt k/2(\sum_{j=1}^{d-1}|{\bf y}_j|+|\sum_{j=1}^{d-1}{\bf y}_j|)}\underbrace{\int f_{(N_j-\bar{N})_{1\le j\le d-1}}({\bf n})f_{(N_j-\bar{N})_{1\le j\le d-1}}({\bf y-\bf n})d{\bf n}}_{=:f_{\sqrt 2(N_j-\bar{N})_{1\le j\le d-1}}({\bf y})}d{\bf y}\\
&\qquad\times(1+O((\log k)^3k^{-1/2})) +O( (\log k)^{d}/k^{(d-1)}) +o(k^{d+6 - c^2d/2})\\
	&=\int 2^{-  \sqrt k/2(\sum_{j=1}^{d-1}|{\bf y}_j|+|\sum_{j=1}^{d-1}{\bf y}_j|)}f_{\sqrt 2(N_j-\bar{N})_{1\le j\le d-1}}({\bf y})d{\bf y}\\
		&\qquad\times(1+O((\log k)^3k^{-1/2}))+O(\sqrt{\log k} /k^{(d-1)})+ o(k^{d+6 - c^2d/2})\,\\
		&=\E\Big[2^{-  \sqrt k/2(\|{\bf Y}\|_1+|\sum_{j=1}^{d-1}{\bf Y}_j|)}\Big] \times(1+O((\log k)^3k^{-1/2}))+O(\sqrt{\log k} /k^{(d-1)})+ o(k^{d+6 - c^2d/2})\,.
	\end{align*}
	This  result is held. Indeed, we recognize that the inner integral is the convolution of two identical Gaussian densities, which yields another Gaussian with twice the variance, i.e  ${\bf Y}=\sqrt{2}(\bm{N}_j-\bar{\bm{N}})\sim \mathcal{N}(0,2\Sigma)$, where $\Sigma$ is the reduced covariance matrix, centered on the hyperplane $\sum_{j=1}^d \bm{N}_j=0. $ Putting it together, and undoing the change of variable  ${\bf Y}$, we get

	\begin{align*}
	&2^{-k} \E\Big[2^{-  1/2\sum_{j=1}^d|K_j-L_j|}\Big]\\ &=\E\Big[2^{-  \sqrt{ k/2}\|\bm{N}_j-\bar{\bm{N}}\|_1}\Big] \times(1+O((\log k)^3k^{-1/2}))+O(\sqrt{\log k} /k^{(d-1)})+ o(k^{d+6 - c^2d/2}).\\
	\end{align*}

Let   $\bm N_j-  \overline{\bm N}\1\sim N(0,P)$, 
where  the covariance matrix  $P$ is the projection onto the hyperplane orthogonal to the constant vector $\1=(1,\ldots,1)$ and more explicitly  $$P= I_d- \dfrac{1}{d} \1 \1^{\top}= \begin{pmatrix}
1-\dfrac{1}{d} & -\dfrac{1}{d} & \ldots &-\dfrac{1}{d}\\
-\dfrac{1}{d} & 1-\dfrac{1}{d} & \ddots & \vdots \\
\vdots &\ddots &\ddots & -\dfrac{1}{d}\\
-\dfrac{1}{d} &  \ldots &-\dfrac{1}{d} & 1-\dfrac{1}{d}
\end{pmatrix}\,.$$

Then $P \bm{N}=(\bm{N}- \overline{\bm{N}}\1)$ is an isotropic Gaussian vector in the sub-vector space $\1^\perp$, the  orthogonal to $\1$. Let us denote $\|\cdot\|_{\1^\perp}$ the Euclidian norm associated in this sub-vector space. We have
\begin{align*}
    \|(\bm{N}-\overline{\bm{N}} \1)\|_{\1^\perp}^ 2&=\|P\bm{N}\|_{\1^\perp}^2=(P\bm{N})^\top P(P\bm{N})=\bm{N}^\top P^3 \bm{N} = \bm{N}^\top P^2 \bm{N} = \|(\bm{N}-\overline{\bm{N}} \1)\|^2\,,
\end{align*}
that is  chi-squared distributed with degree of freedom $d-1$ denoted by $\chi_{(d-1)}^2$ (it is Cochran's theorem). Moreover by isotropy the radius $\|(\bm{N}-\overline {\bm{N}} \1)\|_{\1^\perp}^2$ is  independent of the angle 
\begin{equation}\label{eq:angle}
\Theta := (\bm{\bm{N}}-\overline{\bm{N}} \1)/\|(\bm{\bm{N}}-\overline {\bm{N}} \1)\|_{\1^\perp}= (\bm{\bm{N}}-\overline{\bm{N}} \1)/\|(\bm{N}-\overline{\bm{N}} \1)\|=\sum_{j=1}^{d-1}\omega_je^{\1^\perp}_j\in \mathbb S^{d-1}\,,
\end{equation} 
where $(\omega_j)$ is uniformly distributed over the hypersphere $\mathbb S^{d-2}$ and $(e^{\1^\perp}_j)$ is an orthonormal basis of $\1^\perp$. We therefore turn to the original expression, we have  
\begin{align}\label{eq : expect_proba}
2^k&\E \Big[\P(X_1\in L_\U(\bm{x})|X_1)^2 \Big]\nonumber\\
&=  \E \Big[2^{-   \sqrt{k/2}\| (N-\overline{\bm{N}} \1)/\|(\bm{N}-\overline{\bm{N}} \1)\|\|_1\|(N-\overline{\bm{N}} \1)\| } \Big](1+O((\log k)^3k^{-1/2}))\nonumber\\
& \quad +O((\log k)^d /k^{(d-1)}) + o(k^{d+6 - c^2d/2}).
\end{align}
Let us control the first in \Cref{eq : expect_proba}, which is 
\begin{align*}
    &\E\Big[\E\Big[2^{-  \sqrt{k/2}\|\Theta\|_1\chi_{(d-1)} } | \|\Theta\|_1\Big]\Big]\\
    &=\E\Big[ \int_0^{+\infty}  2^{- \sqrt{k/2} \|\Theta\|_1 \sqrt{x} } \dfrac{x^{(d-3)/2} \exp(-x/2)}{ 2^{(d-1)/2}\Gamma\left((d-1)/2\right)} dx \Big]\\
    &= \dfrac{1}{2^{(d-3)/2}\Gamma\left((d-1)/2\right)}\E\Big[\int_0^{+\infty}   x^{d-2} \exp(-x/2)\exp(- \sqrt{k/2}\log 2 \|\Theta\|_1x) dx\Big]\\
    &= \dfrac{\E[(\sqrt{k/2}\log 2 \|\Theta\|_1)^{-(d-1)}]}{2^{(d-3)/2}\Gamma\left((d-1)/2\right)}\int_0^{+\infty}   x^{d-2} \exp(-x)\exp(-2 \big(x/(\sqrt{k/2}\log 2 \|\Theta\|_1)\big)^2) dx\\
    & \sim \dfrac{\E[(\sqrt{k/2}\log 2 \|\Theta\|_1)^{-(d-1)}]\Gamma(d-1)}{2^{(d-3)/2}\Gamma\left((d-1)/2\right)}, \qquad   k\to \infty, \\
    &= \dfrac{2k^{-(d-1)/2}\Gamma(d-1)}{(\log 2)^{d-1}\Gamma((d-1)/2)}\E\big[\|\Theta\|_1^{-(d-1)}\big]\,.
\end{align*}
Therefore,  we obtain
\begin{align*}
\E \Big[\P(X_1\in L_\U(\bm{x})|X_1)^2 \Big]
&=  C(d)2^{-k}k^{-(d-1)/2} (1+O((\log k)^3k^{-1/2}))\\
& \quad + O(2^{-k} (\log k)^d /k^{(d-1)}) + o(2^{-k} k^{d+6 - c^2d/2}),
\end{align*}
where 
\begin{align}\label{eq_bis: const}
C(d)&=\dfrac{2\Gamma(d-1)}{(\log 2)^{d-1}\Gamma((d-1)/2)}\E[\|\Theta\|_1^{-(d-1)}]\,.
\end{align}
Choosing $c$ sufficiently large, this leads to 
\[
\E \Big[\P(X_1\in L_\U(\bm{x})|X_1)^2 \Big]\sim C(d)2^{-k}k^{-(d-1)/2}\,,\qquad k\to \infty\,.
\]
We conclude the first assertion by choosing 
\[
e^{\1^\perp}_j=\dfrac{1}{\sqrt{j+1}}(\underbrace{1,\ldots,1}_{j\text{ times}},-1,0,\ldots,0)
\]
in \eqref{eq:angle} and noticing that
\[
\sqrt{2}=\|e^{\1^\perp}_{1}\|_1\le \|\Theta\|_1\le \|e^{\1^\perp}_{d-1}\|_1=\sqrt{d}\,,\qquad a.s.
\]
Consequently, for every $d\ge 2$,
\begin{align}
\dfrac{2\Gamma(d-1)}{(\log 2)^{d-1}\Gamma((d-1)/2)}d^{-(d-1)/2}\le C(d)\le\dfrac{2\Gamma(d-1)}{(\log 2)^{d-1}\Gamma((d-1)/2)}2^{-(d-1)/2}\,. \label{eq_encadrement_Cd}
\end{align}
\bigskip

ii) For the second assertion, we use again the distribution 
\begin{align*}
&\E \Big[(\E[\textup{Diam}(L_{\U}(\bm{x}))\1(X_1\in L_\U(\bm{x}))|X_1])^2 \Big]\\
& = \E \Big[(\E[2^{- \min(K_1({\bm x}), \hdots, K_d({\bm x}))}\1(X_1\in L_\U(\bm{x}))|X_1])^2 \Big]\\
& = \E \Big[(\E[ \max\left( 2^{- K_1({\bm x})}, \hdots, 2^{- K_d({\bm x})} \right) \1(X_1\in L_\U(\bm{x}))|X_1])^2 \Big]\\
& \leq  \E \Big[ \left( \sum_{j=1}^d \E[2^{- K_j({\bm x})} \1(X_1\in L_\U(\bm{x}))|X_1] \right)^2 \Big]\\
& \leq d^2 \E \Big[(\E[2^{- K_1({\bm x})} \1(X_1\in L_\U(\bm{x}))|X_1])^2 \Big]\\
&= d^2 \E \Big[ \Big( \sum_{\substack{k_1, \ldots, k_d \\ \sum_{j=1}^d k_j = k}} \frac{k!}{k_1! \ldots k_d!} \Big( \frac{1}{d} \Big)^k 2^{-k_1}\prod_{j=1}^d \mathds{1}_{\lceil 2^{k_j} x^{(j)} \rceil  = \lceil 2^{k_j} X_1^{(j)} \rceil } \Big)^2 \Big]\\
&= d^2 \sum_{\substack{(k_1, \ldots, k_d)\\ (l_1, \ldots, l_d)\\ \sum_{j=1}^d k_j = \sum_{j=1}^d l_j =k}} \frac{k!}{k_1! \ldots k_d!} \frac{l!}{l_1! \ldots l_d!} \Big( \frac{1}{d} \Big)^{2k}2^{-k_1- l_1 - \sum_{j=1}^d \max(k_j,l_j)}\\
&= d^2 2^{-k}\sum_{\substack{(k_1, \ldots, k_d)\\ (l_1, \ldots, l_d)\\ \sum_{j=1}^d k_j = \sum_{j=1}^d l_j =k}} \frac{k!}{k_1! \ldots k_d!} \frac{l!}{l_1! \ldots l_d!} \Big( \frac{1}{d} \Big)^{2k}2^{-k_1- l_1 -(1/2)\sum_{j=1}^d|k_j-l_j|}\\
&=  d^2 2^{-k}\E\Big[2^{-K_1-L_1-(1/2)\sum_{j=1}^d|K_j-L_j|}\Big]\\
&\le  d^2 2^{-k}\E\Big[2^{-2(K_1+L_1)}\Big]^{1/2}\E\Big[2^{-\sum_{j=1}^d|K_j-L_j|}\Big]^{1/2}\,,
\end{align*}
by an application of Cauchy-Schwarz inequality. Simple calculations involving binomial distributions show that $\E[2^{-2K_1}]=(1-3/(4d))^k = \alpha_2^k$. Therefore 
\begin{align*}
\E \Big[(\E[\textup{Diam}(L_{\U}(\bm{x}))\1(X_1\in L_\U(\bm{x}))|X_1])^2 \Big]
&  \leq d^2 2^{-k} \alpha_2^k \E\Big[2^{-\sum_{j=1}^d|K_j-L_j|}\Big]^{1/2}\\
& \leq d^2 2^{-k} \alpha_2^k.
\end{align*}

\end{proof}

\section{Proof of \Cref{th: TCLB} and \Cref{cor: TCLB}}\label{sec: proof_TCLB}
\Cref{sec: proof_TCLB} is organized as follows: preliminary results are gathered in \Cref{appB:prelim}; the proof of \Cref{th: TCLB} is presented in \Cref{AppB:proofTH32}, and that of \Cref{cor: TCLB} is given in \Cref{proof: corTCLB}.

\subsection{Preliminaries on the rebalanced distribution $X'$}\label{appB:prelim}
\begin{lem}[\textbf{Density of $X'$}]\label{lem_density_Xprime}
Let $\bm x \in [0,1]^d$. Assume that the random variable $X$ admits a density denoted by $f_X$. Then, for all $p'\in (0,1)$, the random variable $X'$ defined in \Cref{definition_xprime} admits a density $f_{X'}$ satisfying
\begin{align}
    f_{X'}(\bm{x}) & =  \frac{1 - p'}{1-p} f_{X}(\bm{x}) + \left( \frac{p' - p}{p(1-p)} \right)  \mu({\bf x}) f_X({\bf x}).
\end{align}
\end{lem}

\begin{proof}[Proof of \Cref{lem_density_Xprime}]
Let $f_{X|Y=j}$ (resp.\ $f_{X'|Y'=j}$) be the conditional density of $X$ (resp. of $X'$) given $Y=j$ (resp. given $Y'=j$) for all $j \in \{0,1\}$. According to the Bayes formula, we have
\begin{align}\label{eq:psUprime1}
f_X(\bm{x})&= pf_{X|Y=1}({\bf x})+(1-p)f_{X|Y=0}({\bf x})
\end{align}
and 
\begin{align}\label{eq:psUprime2}
   f_{X'}(\bm{x})&= p'f_{X'|Y'=1}({\bf x})+(1-p')f_{X'|Y'=0}({\bf x}).
\end{align}
By assumption, 
\begin{align}
 f_{X|Y=1}({\bf x}) & =f_{X'|Y=1}({\bf x}), \\
  f_{X|Y=0}({\bf x}) & = f_{X'|Y'=0}({\bf x}).
\end{align}
Therefore, 
\begin{align}
    f_{X'}(\bm{x}) - \frac{1 - p'}{1-p} f_{X}(\bm{x}) & = \left( p' - p \frac{1 - p'}{1-p} \right) f_{X'|Y'=1}({\bf x}), 
\end{align}
where
\begin{align}
f_{X'|Y'=1}({\bf x}) & = f_{X|Y=1}({\bf x})  = \frac{\mu({\bf x}) f_X({\bf x})}{p}.
\end{align}
Finally, we obtain, 
\begin{align}
    f_{X'}(\bm{x}) & =  \frac{1 - p'}{1-p} f_{X}(\bm{x}) + \left( \frac{p' - p}{p(1-p)} \right)  \mu({\bf x}) f_X({\bf x})
\end{align}
concluding the proof.
\end{proof}

Based on the odds ratio formula, the following lemma makes explicit the relationship between $\mu'$ and $\mu$ that is used to correct the bias.

\begin{lem}[\textbf{Odds ratio}]\label{lem_mu_prime_expression}
Let $\bm x \in [0,1]^d$ and  $p,p'\in (0,1)$. Under Condition {\bf (H0)}, we have
\begin{equation}
\mu'(\bm x)=\dfrac{p'(1-p)\mu(\bm x)}{p(1-p')(1-\mu(\bm x))+(1-p)p'\mu(\bm x)},
\end{equation}
which is equivalent to 
\begin{align}
\mu(\bm x)=\dfrac{p(1-p')\mu'(\bm x)}{p'(1-p)(1-\mu'(\bm x))+(1-p')p\mu'(\bm x)}.
\end{align}
\end{lem}

\begin{proof}[Proof of \Cref{lem_mu_prime_expression}]
By assumption, 
\begin{align}
 f_{X|Y=1}({\bf x}) & =f_{X'|Y'=1}({\bf x}), \\
  f_{X|Y=0}({\bf x}) & = f_{X'|Y'=0}({\bf x}), \\
  \P(Y' = 1) & = p' \in (0,1).
\end{align}
Thus, 
\begin{align}
    \frac{f_{X|Y=1}({\bf x})}{f_{X|Y=0}({\bf x})} & = \frac{f_{X'|Y'=1}({\bf x})}{f_{X'|Y'=0}({\bf x})}.
\end{align}
By Bayes formula, 
\begin{align}
f_{X|Y=1}({\bf x}) & = \frac{\mu({\bf x}) f_{X}({\bf x})}{p} \\
f_{X|Y=0}({\bf x}) & = \frac{(1 - \mu({\bf x})) f_{X}({\bf x})}{1-p} \\
f_{X'|Y'=1}({\bf x}) & = \frac{\mu'({\bf x}) f_{X'}({\bf x})}{p'} \\
f_{X'|Y'=0}({\bf x}) & = \frac{(1-\mu'({\bf x})) f_{X'}({\bf x})}{1 - p'}.
\end{align}
Therefore, 
\begin{align}
    \frac{\mu({\bf x})}{1 - \mu({\bf x})} \frac{1-p}{p}& = \frac{\mu'({\bf x})}{1 - \mu'({\bf x})} \frac{1-p'}{p'}.
\end{align}
Letting $\gamma = p(1-p')/(p'(1-p))$, simple calculations show that
\begin{align}
    & \mu({\bf x}) (1 - \mu'({\bf x}))  - \gamma \mu'({\bf x}) (1 - \mu({\bf x})) = 0\\
   \Longleftrightarrow &     \mu'({\bf x}) ( - \mu({\bf x})  - \gamma + \gamma \mu({\bf x})) + \mu({\bf x})= 0\\
\Longleftrightarrow &     \mu'({\bf x})  = \frac{\mu({\bf x}) }{  \mu({\bf x}) +  \gamma (1- \mu({\bf x}))}\\
\Longleftrightarrow &     \mu'({\bf x})  = \frac{p'(1-p) \mu({\bf x}) }{  p'(1-p) \mu({\bf x}) +  p(1-p') (1- \mu({\bf x}))}.
\end{align}
Similar calculations lead to the second equality. 
\end{proof}

\subsection{Proof of \Cref{th: TCLB}}\label{AppB:proofTH32}

The proof of \cref{th: TCLB} follows the lines of the proof of \Cref{th: TCL}. Accordingly to \Cref{sec:Proof_th_TCL}, for simplicity, we write 
\begin{align}
    {T^s}'&=T^s(\bm x, \U, \bm Z_{s}'),\\
    T_{1,s}'&=\E[{T^s}'\mid Z_1]-\E[{T^s}'],\\
    V_{1,s}'& :=\v(T_{1,s}'), \\
    \mu'(\bm x)& =\P(Y'=1\mid X'= \bm x).
\end{align}
We also denote 
\begin{align*}
 p_{\depth,\U}'(\bm{x})&: = \P(X' \in L_{\U}(\bm{x})\mid \U)\,,\\
\mu_{\depth, \U}' (\bm{x})&:=\P(Y'=1|X' \in L_{\U}(\bm{x}),\U )\,,\\
 {\sigma'}^2_{\depth, \U}(\bm x)&:= \mu_{\depth, \U}' (\bm{x})(1- \mu_{\depth, \U}' (\bm{x})).
\end{align*}

The preliminary results for ICRFs  provided in  \Cref{subsec:pre_results} remain valid for  the rebalanced ICRFs. Indeed, the bias and the variance of $\textup{Diam}(L_{\U}(\bm{x}))$ given in \Cref{lem:diam} remain unchanged since the tree splitting process is the same for the ICRFs and the rebalanced ICRFs (see \Cref{subsec:Notation} for the contruction of CRTs). \Cref{lem:E_ynn_tilde} and  \Cref{lem:bias_mu_prime} below are respectively simple extensions of  \Cref{lem:E_ynn} and \Cref{lem:bias_mutilde} to the rebalanced ICRFs setup.

\begin{lem}[\textbf{Bias and variance of $T^s(\bm{x}; \U; \bm Z_{s}')\mid \U$}]\label{lem:E_ynn_tilde}
Let $\bm x \in [0,1]^d$. Then, the following two equalities hold:

\begin{itemize}
    \item [i)] $
\mathbb{E}  [T^s(\bm{x}; \U; \bm Z_{s}') \mid \U] =\mu'_{\depth, \U}(\bm{x})(1-(1-p'_{\depth,\U}(\bm{x}))^s);\label{eq:biasTprime}$

 \item [ii-]
  $\v(T^s(\bm{x}; \U; \bm Z_{s}') \mid \U)  = \dfrac{\sigma'^2_{\depth, \U}(\bm x)}{sp'_{\depth,\U}(\bm{x})}(1-(1-p'_{\depth,\U}(\bm{x}))^s)^2.\label{eq:varTprime}$
  
   \end{itemize}
\end{lem}

\begin{proof}[Proof of \cref{lem:E_ynn_tilde}]
Simple extension of \Cref{lem:E_ynn}. 
\end{proof}

\begin{lem}[\textbf{Bias and Variance of $\mu'_{\depth, \U}(\bm{x})$}]\label{lem:bias_mu_prime}
Let $\bm x \in [0,1]^d$ and assume Condition \textbf{(H1)}  holds. Then, as $n',s,k \to \infty$,  we have   \begin{itemize}
      \item[$i)$]$\E[\mu'_{\depth, \U}(\bm{x})]=\mu'(\bm x)+O(\alpha_1^\depth);$
      \item[$ii)$] $\v(\mu'_{\depth, \U}(\bm{x}) ) =O(\alpha_2^\depth);$
    \item[$iii)$] $\E[{\mu'_{\depth, \U}}^2 (\bm{x})]= \mu'(\bm x)^2+O(\alpha_1^\depth).$ 
  \end{itemize}
\end{lem}

\begin{proof}[Proof of \Cref{lem:bias_mu_prime}]
Straightforward extension of \Cref{lem:bias_mutilde}. 
\end{proof}

In the rebalanced ICRF setting, the covariate $X'$ is no longer uniformly distributed (see \Cref{lem_density_Xprime}), so the identity $p_{\depth,\U}(\bm{x}) = 2^{-k}$ no longer holds. The following lemma aims to provide an accurate characterization of $p_{\depth,\U}'(\bm{x})$. 

\begin{lem}[\textbf{Characterisation of $ p_{\depth,\U}'(\bm{x})$}]\label{lem: pU}
Let $\bm x \in [0,1]^d$ and assume Conditions \textbf{(H0)}, \textbf{(H1)} and \textbf{(G1)}  hold. Then,  we have
\begin{align}
\label{eq_pprimeU_lemmaB3_1}
    p_{\depth,\U}'(\bm{x}) = \dfrac{c'(\bm x)}{2^\depth} \left(1+ \alpha'({\bf x}) \varepsilon'_{\U}({\bf x}) \textup{Diam}(L_{\U}(\bm{x}))\right) \qquad \text{a.s.}
\end{align}
where $\alpha'({\bf x}) = (p'-p)/(p(1-p)c'({\bf x}))$ and 
$$
c'(\bm x):=\dfrac{p'(1-p')}{p'(1-p)(1-\mu'(\bm x))+(1-p')p\mu'(\bm x)}\,,
$$
and $\varepsilon'_{\U}({\bf x})$ is a random variable satisfying $|\varepsilon'_{\U}({\bf x})| \leq 2L$ almost surely. Morever, a.s.,
\begin{align}
\label{eq_pprimeU_lemmaB3}
 p_{\depth,\U}'(\bm{x})\ge \left(\frac{1-p'}{1-p}\right) 2^{-\depth}.  
\end{align}
\end{lem}

\begin{proof}[Proof of \Cref{lem: pU}.] By assumption, we have
\begin{align}\label{eq:psUprime1}
p_{\depth,\U}'(\bm{x})&= p'\P(X \in L_{\U}(\bm{x})\mid \U, Y=1)+(1-p')\P(X  \in L_{\U}(\bm{x})\mid \U, Y=0)
\end{align}
and 
\begin{align}\label{eq:psUprime2}
    p_{\depth, \U}(\bm x)= p \P(X\in L_{\U}(\bm{x})\mid \U, Y=1)+(1-p)\P(X \in L_{\U}(\bm{x})\mid \U, Y=0)\,.
\end{align}
As noticed in \Cref{rem:trees},  Condition \textbf{(H0)} implies $p_{k, \U}(\bm x)= 2^{-\depth}$. 
Combining \Cref{eq:psUprime1} and \Cref{eq:psUprime2} yields
\begin{align*}
   p_{\depth,\U}'(\bm{x}) &=\frac{1-p'}{1-p} 2^{-\depth} + \Big(p'-\frac{(1-p')p}{1-p}\Big) \P(X \in L_{\U}(\bm{x})\mid \U, Y=1)\\
   & = \frac{1-p'}{1-p}2^{-\depth}  + \frac{p'-p}{1-p}\P(X \in L_{\U}(\bm{x})\mid \U, Y=1)\\
   & \geq \left(\frac{1-p'}{1-p}\right) 2^{-\depth},
\end{align*}
since, by assumption, $p'>p$. This proves the second statement. Let $\varepsilon_\U$ be a random variable such that $|\varepsilon_\U|\leq L'$  a.s, under Condition \textbf{(G1)} we also have  
\begin{align*}
    &  \P(X \in L_{\U}(\bm{x})\mid \U, Y=1) \\
    &=  \frac{\P(X \in L_{\U}(\bm{x})\mid \U)}{\P(Y=1 | X \in L_{\U}(\bm{x}),  \U)} \P(Y = 1 \mid X \in L_{\U}(\bm{x}),  \U)\\
    & = \dfrac 1p\int_{L_{\U}(\bm{x})} \mu(y)dy\\
    & = \dfrac 1p\int_{L_{\U}(\bm{x})} \mu({\bf x}) + (\mu(y) - \mu({\bf x})dy\\
    & = \frac{\mu({\bf x})}{p} \int_{L_{\U}(\bm{x})} dy + \frac{1}{p} \int_{L_{\U}(\bm{x})}  (\mu(y) - \mu({\bf x}))dy.
\end{align*}
Let 
$$
\varepsilon_{\U}(\bm{x}) = \frac{ \mu_{\depth, \U}(\bm{x}) - \mu(\bm x)}{\textup{Diam}(L_{\U}(\bm{x}))},
$$
which verifies, under Condition {\bf (H1)}, $|\varepsilon_\U({\bf x})|\leq L$ almost surely. Then,
\begin{align}
& \int_{L_{\U}(\bm{x})}  (\mu(y) - \mu({\bf x}))dy \nonumber \\
&  = \int_{L_{\U}(\bm{x})} \textup{Diam}(L_{\U}(\bm{x})) (\varepsilon_{\U}(y) - \varepsilon_{\U}(\bm{x})) dy \nonumber \\
& = \textup{Diam}(L_{\U}(\bm{x})) \P(X \in L_{\U}(\bm{x}) \mid \U) \left(\frac{1}{\P( X \in L_{\U}(\bm{x}) \mid \U)} \int_{L_{\U}(\bm{x})} (\varepsilon_{\U}(y) - \varepsilon_{\U}(\bm{x})) dy\right) \nonumber \\
& = \textup{Diam}(L_{\U}(\bm{x})) \P(X \in L_{\U}(\bm{x}) \mid \U) \varepsilon'_{\U}({\bf x}),
\end{align}
where 
\begin{align}
\varepsilon'_{\U}({\bf x}) =    \frac{1}{\P( X \in L_{\U}(\bm{x}) \mid \U)} \int_{L_{\U}(\bm{x})} (\varepsilon_{\U}(y) - \varepsilon_{\U}(\bm{x})) dy 
\end{align}
satisfies, almost surely, $|\varepsilon'_{\U}({\bf x})|\leq 2L$. Thus, recalling that $\P(X \in L_{\U}(\bm{x}) \mid \U ) = 2^{-\depth}$,
\begin{align*}
 & \P(X \in  L_{\U}(\bm{x})\mid \U, Y=1)\\
   & =  \dfrac {\mu(\bm x)}p \P(X \in L_{\U}(\bm{x}) \mid \U ) + \frac{1}{p} \textup{Diam}(L_{\U}(\bm{x})) \P(X \in L_{\U}(\bm{x}) \mid \U) \varepsilon'_{\U}({\bf x})\\
   &= \frac{ 2^{-\depth}}{p} (\mu({\bf x}) + \varepsilon'_{\U}({\bf x}) \textup{Diam}(L_{\U}(\bm{x}))).
\end{align*}
Combining those results we get
\begin{align*}
   p_{\depth,\U}'(\bm{x})
   & = \frac{1-p'}{1-p}2^{-\depth}  + \frac{p'-p}{1-p}\P(X^{1} \in L_{\U}(\bm{x})\mid \U)\\
   & = \left(\frac{1-p'}{1-p}+ \frac{p'-p}{1-p} \dfrac {\mu({\bf x})}p  + \frac{p'-p}{1-p} \dfrac {1}p \varepsilon'_{\U}({\bf x}) \textup{Diam}(L_{\U}(\bm{x}) \right)2^{-\depth}.\label{eq:pprime}
\end{align*}

According to \Cref{lem_mu_prime_expression}, we have 
$$
\dfrac{\mu(\bm x)}p=\dfrac{(1-p')\mu'(\bm x)}{p'(1-p)(1-\mu'(\bm x))+(1-p')p\mu'(\bm x)}
$$
 achieving the identity 
$$
\frac{1-p'}{1-p}+ \frac{p'-p}{1-p} \dfrac {\mu(\bm x)}p=\dfrac{p'(1-p')}{p'(1-p)(1-\mu'(\bm x))+(1-p')p\mu'(\bm x)}\,.
$$
Plugging the latter equation into \Cref{eq:pprime} concludes the proof of the first assertion.
\end{proof}
The next result characterises the bias and variance of $T^s(\bm{x}; \U; \bm Z_{s}')$ and extends \Cref{lem:bias_Ts} to the rebalanced ICRF setup. 
\begin{prp}[\textbf{Bias and Variance of $T^s(\bm{x}; \U; \bm Z_{s}')$}]\label{prop:biasB} Let $\bm x \in [0,1]^d$ and assume Conditions \textbf{(H0)}, \textbf{(H1)} and \textbf{(G1)}  hold. Then, as $n',s,k \to \infty$, we have

  \begin{itemize}
      \item[$i)$] $\E[T^s(\bm{x}; \U; \bm Z_{s}')]=\mu'(\bm x)  +O(\alpha_1^\depth);$
    \item[$ii)$] $
    \v(T^s(\bm{x}; \U; \bm Z_{s}'))=\dfrac{2^\depth}{c'(\bm x)s}\Big(\mu'(\bm{x})(1-\mu'(\bm{x}))+ O(\alpha_1^\depth)\Big) + O(\alpha_2^\depth)\,.
$
  \end{itemize}
\end{prp}
\begin{proof}
The proof is similar to the proof of \Cref{lem:bias_Ts} so we only briefly mention the modifications. On the one hand, using \Cref{eq:biasTprime}, we have
\begin{align}
 \E[T^s(\bm{x}; \U; \bm Z_{s}')]& =\E[\E[T^s(\bm{x}; \U; \bm Z_{s}')\mid \U]] \\
 & = \E[\mu_{\depth, \U}'(\bm x)]+O(e^{-sp_{\depth,\U}'(\bm{x})})\\
 & =\mu'(\bm x)+O(\alpha_1^\depth)+O(e^{-sp_{\depth,\U}'(\bm{x})}).   
\end{align}
Denoting $W_{\U}({\bf x}) = \alpha'({\bf x}) \varepsilon'_{\U}({\bf x}) \textup{Diam}(L_{\U}({\bf x}))$, according to \Cref{lem: pU}, almost surely, we have 
\begin{align}
  p_{\depth,\U}'(\bm{x}) = \dfrac{c'(\bm x)}{2^\depth} \left(1+ W_{\U}({\bf x})\right).   
\end{align}
From \eqref{eq_pprimeU_lemmaB3_1} and  \eqref{eq_pprimeU_lemmaB3}, we have
\begin{align}
    1 + W_{\U}({\bf x}) \geq \frac{1}{c'({\bf x})} \left( \frac{1-p'}{1-p}\right).
\end{align}
Thus, 
\begin{align}
O\left(e^{-sp_{\depth,\U}'(\bm{x})} \right) & =O\left(\exp \left[-\left(\frac{1-p'}{2(1-p)}\right) s 2^{- \depth} \right] \right)\\
& =o\left(\alpha_1^\depth\right), \label{eq_negligeability_exp_ps}
\end{align}
according to Assumption \textbf{(G1)}, which proves the first statement.

Using \Cref{eq:varTprime}, 
\begin{align}
    \E[\v(T^s(\bm{x}; \U; \bm Z_{s}') \mid \U)] & = \E \left[ \dfrac{\sigma'^2_{\depth, \U}(\bm x)}{sp'_{\depth,\U}(\bm{x})}(1-(1-p'_{\depth,\U}(\bm{x}))^s)^2 \right].
\end{align}

Thus, 
\begin{align}
    & \E[\v(T^s(\bm{x}; \U; \bm Z_{s}') \mid \U)] \\
    & = \frac{2^\depth}{c'({\bf x}) s }\E \left[ \sigma'^2_{\depth, \U}(\bm x) \frac{1 + O(\alpha_1^\depth)}{1 + W_{\U}({\bf x})} \right]\\
    & = \frac{2^\depth}{c'({\bf x}) s } \left( \E \left[ \sigma'^2_{\depth, \U}(\bm x) \right] +\E\left[ \sigma'^2_{\depth, \U}(\bm x)  \frac{- W_{\U}({\bf x})+ O(\alpha_1^\depth)}{1 + W_{\U}({\bf x})} \right] \right). \label{eq_proof_var_prime1}
\end{align}
The last term in \eqref{eq_proof_var_prime1} satisfies, for all $\depth$ large enough
\begin{align}
& \left| \E\left[ \sigma'^2_{\depth, \U}(\bm x)  \frac{- W_{\U}({\bf x})+ O(\alpha_1^\depth)}{1 + W_{\U}({\bf x})} \right] \right| \\
& \leq  \frac{1}{4}  \E\left[   \frac{| W_{\U}({\bf x})|+ |O(\alpha_1^\depth)|}{|1 + W_{\U}({\bf x})|} \right]\\
& \leq  \frac{c'({\bf x})}{4} \left( \frac{1-p}{1-p'}\right)  \E\left[ | W_{\U}({\bf x})|+ |O(\alpha_1^\depth)| \right]\\
& \leq O(\alpha_1^\depth) + \frac{c'({\bf x})}{2} L \left( \frac{1-p}{1-p'}\right) |\alpha({\bf x})| \E\left[ | \textup{Diam}(L_{\U}({\bf x}))| \right] \\
& \leq O(\alpha_1^\depth).
\end{align}
Besides, using the first and third statements of \Cref{lem:bias_mu_prime}, 
\begin{align}
   \E[  {\sigma'}^2_{\depth, \U}(\bm x)] &=  \E[ \mu_{\depth, \U}' (\bm{x})(1- \mu_{\depth, \U}' (\bm{x}))]\\
   & = \mu'(\bm x)-\mu'(\bm x)^2+O(\alpha_1^\depth).
\end{align}
Combining these results into \eqref{eq_proof_var_prime1} leads to 
\begin{align}
    & \E[\v(T^s(\bm{x}; \U; \bm Z_{s}') \mid \U)] \\
    & = \frac{2^\depth}{c'({\bf x}) s } \left( \mu'(\bm x)-\mu'(\bm x)^2+O(\alpha_1^\depth) \right). 
\end{align}
Furthermore, using  \Cref{eq:biasTprime}, we have
\begin{align*}
& \v (\mathbb{E}  [T^s(\bm{x}; \U; \bm Z_{s}') \mid \U]) \\
& = \v \Big( \mu'_{\depth, \U}(\bm{x})(1-(1-p'_{\depth,\U}(\bm{x}))^s) \Big)\\
& = \v \Big( \mu'_{\depth, \U}(\bm{x}) \Big) + \v \Big(\mu'_{\depth, \U}(\bm{x}) (1-p'_{\depth,\U}(\bm{x}))^s \Big) - 2 \textup{Cov}(\mu'_{\depth, \U}(\bm{x}), \mu'_{\depth, \U}(\bm{x})(1-p'_{\depth,\U}(\bm{x}))^s)\\
& \leq \v \Big( \mu'_{\depth, \U}(\bm{x}) \Big) + \E \Big((1-p'_{\depth,\U}(\bm{x}))^{2s} \Big) + 2 \sqrt{ \v \Big( \mu'_{\depth, \U}(\bm{x}) \Big)} \sqrt{\E \Big((1-p'_{\depth,\U}(\bm{x}))^{2s} \Big)}\\
& \leq O(\alpha_2^\depth),
\end{align*}
since, according to \eqref{eq_negligeability_exp_ps},
$$\E \Big[(1-p'_{\depth,\U}(\bm{x}))^{2s} \Big] = o(\alpha_2^\depth),$$
and  $\v ( \mu'_{\depth, \U}(\bm{x})) = O(\alpha_2^\depth) $ thanks to \Cref{lem:bias_mu_prime}. 
Finally, using the variance decomposition formula, we obtain
\begin{align}
  \v(T^s(\bm{x}; \U; \bm Z_{s}')  & =    \E[\v(T^s(\bm{x}; \U; \bm Z_{s}') \mid \U)] + \v (\mathbb{E}  [T^s(\bm{x}; \U; \bm Z_{s}') \mid \U])\\
    & = \frac{2^\depth}{c'({\bf x}) s } \left( \mu'(\bm x)-\mu'(\bm x)^2+O(\alpha_1^\depth) \right) + O(\alpha_2^\depth).
\end{align}
This concludes the proof.
\end{proof}
The next proposition provides upper and lower bounds of $V_{1,s}'$ and 
extends \Cref{prop:v1} to the rebalanced ICRF setup.

\begin{prp}[\textbf{Equivalent of $V_{1,s}'$}]\label{prop:v1prime} 
Let $\bm x \in [0,1]^d$ with $d\ge 2$ and assume Condition \textbf{(H1)}  holds. Then, for all $\depth$ large enough, 
\begin{align}
   &  \frac{1}{2} C(d) \left( \frac{1-p'}{1-p}\right) \frac{ \mu'({\bf x})(1-\mu'({\bf x}))}{(1 + 2L \alpha({\bf x}))c'({\bf x})^2}  \nonumber \\
    & \leq  \dfrac{s^2 \depth^{(d-1)/2}}{2^{\depth}} V_{1,s}' \nonumber \\
    & \leq 2 C(d) \frac{p'}{p} \frac{ \mu'({\bf x})(1-\mu'({\bf x}))}{ c'({\bf x})^2},
\end{align}
where $\alpha({\bf x})$ and $c({\bf x})$ are defined in \Cref{lem: pU}. 

\end{prp}
\begin{proof}[Proof of \Cref{prop:v1prime}.]
We adopt the same notation as in \Cref{prop:v1}. From the definition of the individual trees given in  \Cref{def:base_two_sampl},  we have
\begin{align*}
T^s(\bm{x}; \U; \bm Z_{s}') & =T^{s-1}(\bm{x}; \U; \bm Z'_{\{2,\ldots,s\}})\1(X_1' \notin L_{\U}(\bm{x}))\\
& \quad  +\tilde{T}^{s-1}_1(\bm{x}; \U; \bm Z'_{s})\1(X_1' \in L_{\U}(\bm{x}))\,,
\end{align*}
where
\begin{equation*}
		T^{s-1}(\bm{x}; \U; \bm Z'_{\{2,\ldots,s\}})= \frac{\sum_{i=2}^s Y'_{i}\1{\{X'_{i} \in L_{\U}(\bm{x}) \}}}{\bm{N}_{L_{\U}(\bm{x})}( \bm X'_{\{2,\ldots,s\}})}  
	\end{equation*}
and
\[
\tilde{T}^{s-1}_1(\bm{x}; \U; \bm Z'_{s}) = \frac{1}{N_{L_{\U}(\bm{x})}(\bm{X}'_{\{2,\ldots,s\}})+1}\Big(Y_1'+\sum_{i= 2}^s Y'_{i}\1{\{X'_{i} \in L_{\U}(\bm{x}) \}}\Big)\,.
\]
Taking expectation conditionally
to $Z_1', \U$, we have  
\begin{align}
  	\mathbb{E}[T^s(\bm{x}; \U; \bm Z'_{s}) \ | \ Z'_1, \U]&= \E[T^{s-1}(\bm{x}; \U; \bm Z'_{\{2,\ldots, s\}})|Z'_1, \U]\1(X'_1 \notin L_{\U}(\bm{x})) \nonumber \\
  	& \quad +   \E[\tilde{T}^{s-1}_1(\bm{x}; \U; \bm Z'_{s})|Z'_1, \U]\1(X'_1 \in L_{\U}(\bm{x})) \nonumber\\
  &= \E[T^{s-1}(\bm{x}; \U; \bm Z'_{\{2,\ldots,s\}})| \U]\1(X'_1 \notin L_{\U}(\bm{x})) \nonumber\\
  	& \quad +   \E[\tilde{T}^{s-1}_1(\bm{x}; \U; \bm Z'_{s})|Y'_1, \U]\1(X'_1 \in L_{\U}(\bm{x})). \label{proof_eq_decomp_global_bis}
\end{align}
Under Condition \textbf{(H1)}, using \cref{lem:E_ynn_tilde}, we obtain 
\begin{align}
    \E [T^{s-1}(\bm{x}; \U; \bm Z'_{\{2,\ldots, s\}})| \U ]& =\mu'_{\depth,\U}(\bm{x}) \Big[1-(1-p'_{\depth,\U}(\bm{x}))^{s-1}\Big] \\
    & =\mu'_{\depth,\U}(\bm{x})(1+O(e^{-sp'_{\depth,\U}(\bm{x})}))\,. \label{proof_eq_decomp1_bis}
\end{align}
Besides, we have
	\begin{align*}
	\E[\tilde{T}^{s-1}_1(\bm{x}; \U; \bm Z'_{s})|Z'_1, \U] &= \underbrace{\mathbb{E}\Big[\frac{1}{N_{L_{\U}(\bm{x})}\Big(\bm{X}'_{\{2,\ldots,s\}}\Big)+1} \mid \U\Big]}_{=:a(\U)} Y_1' \\
		& \quad + \underbrace{\E\Big[\frac{1}{1+ N_{L_{\U}(\bm{x})}\Big(\bm{X}'_{\{2,\ldots,s\}}\Big)} \sum_{i=2}^s Y_i'\mathds{1}_{\{X_i' \in L_{\U}(\bm{x}) \}} \Big | \U\Big]}_{=:b(\U)}.
	\end{align*}
	Using \Cref{lem:cribari} i) yields
	\begin{align*}
	a(\U)& = \E\Big[\frac{1}{N_{L_{\U}(\bm{x})}(\bm{X}'_{\{2,\ldots,s\}})+1}\mid \U\Big] \\
	& 	= \frac{1}{sp'_{\depth,\U}(\bm{x})} \Big[1-(1-p'_{\depth,\U}(\bm{x}))^{s}\Big] \\
	& =\frac{1}{sp'_{\depth,\U}(\bm{x})}(1+O(e^{-sp'_{\depth,\U}(\bm{x})}))
	\end{align*}
and
\begin{align*}
   b(\U)&=  \E\Big[\frac{1}{1+ N_{L_{\U}(\bm{x})}(\bm{X}'_{\{2,\ldots,s\}})} \sum_{i =2}^s Y_i'\mathds{1}_{\{X_i' \in L_{\U}(\bm{x}) \}} \Big | \U\Big]\\
    &= \mu'_{\depth,\U}(\bm{x}) \sum_{k=0}^{s-1} \frac{k}{1+k} \mathbb{P}\Big( N_{L_{\U}(\bm{x})}(\bm{X}'_{\{2,\ldots,s\}})=k\mid \U\Big)\\
    & = \mu'_{\depth, \U}(\bm{x}) \E\Big[\frac{N_{L_{\U}(\bm{x})}(\bm{X}_{\{2,\ldots,s\}})}{1+N_{L_{\U}(\bm{x})}(\bm{X}'_{\{2,\ldots,s\}})}\mid \U \Big]\\
    & =\mu'_{\depth, \U}(\bm{x}) \Big( 1- \E\Big[\frac{1}{1+N_{L_{\U}(\bm{x})}(\bm{X}'_{\{2,\ldots,s\}})}\mid \U\Big]\Big) \\
    &=\mu'_{\depth, \U}(\bm x) \Big(1-\frac{1}{sp'_{\depth,\U}(\bm{x})}(1+O(e^{-sp'_{\depth,\U}(\bm{x})}))\Big).
    \end{align*}
Thus, 
\begin{align}
\E[\tilde{T}^{s-1}_1(\bm{x}; \U; \bm Z'_{s})|Z_1', \U] &=  \frac{Y_1'}{sp'_{\depth,\U}(\bm{x})} + \mu'_{\depth, \U}(\bm x) \Big(1-\frac{1}{sp'_{\depth,\U}(\bm{x})}\Big) \\
& \quad + (Y_1' + \mu'_{\depth, \U}(\bm x) )O\Big(\frac{e^{-sp'_{\depth,\U}(\bm{x})}}{sp'_{\depth,\U}(\bm{x})}\Big).  \label{proof_eq_decomp3_bis}
\end{align}
Gathering \eqref{proof_eq_decomp1_bis} and \eqref{proof_eq_decomp3_bis} in \eqref{proof_eq_decomp_global_bis}, we have
\begin{align}\label{eq:EsT}
  & \mathbb{E}[T^s(\bm{x}; \U; \bm Z'_{s}) \ | \ Z_1', \U] \nonumber \\
  &=
  	\mu'_{\depth,\U}(\bm{x})(1+O(e^{-sp'_{\depth,\U}(\bm{x})})) \1(X_1' \notin L_{\U}(\bm{x})) \nonumber\\
  	& \quad +   \Big(\frac{1}{sp'_{\depth,\U}(\bm{x})}Y_1' + \mu'_{\depth,\U}(\bm x)\Big(1-\dfrac{1}{sp'_{\depth,\U}(\bm{x})}\Big)\Big)\1(X_1' \in L_{\U}(\bm{x})) \nonumber\\
  	& \quad +  (Y_1' + \mu'_{\depth,\U}(\bm x) )O\Big(\frac{e^{-sp'_{\depth,\U}(\bm{x})}}{sp'_{\depth,\U}(\bm{x})}\Big) \1(X_1' \in L_{\U}(\bm{x})) \nonumber \\
  	&=  \mu'_{\depth,\U}(\bm{x})+\dfrac{1}{sp'_{\depth,\U}(\bm{x})} (Y_1'-\mu'_{\depth,\U}(\bm x))\1(X_1' \in L_{\U}(\bm{x})) \nonumber \\
  	& \quad +(Y_1' + \mu'_{\depth,\U}(\bm x) )O\Big(\frac{e^{-sp'_{\depth,\U}(\bm{x})}}{sp'_{\depth,\U}(\bm{x})}\Big).
\end{align}
Recalling that 
$\mu'_{\depth,\U}(\bm{x})=\P(Y'=1|X' \in L_{\U}(\bm{x}),\U ),$ let
$$
\varepsilon'_{\U}(\bm{x}) = \frac{ \mu'_{\depth,\U}(\bm{x}) - \mu'(\bm x)}{\textup{Diam}(L_{\U}(\bm{x}))}.
$$
Thus, under Condition {\bf (H1)}, we have $|\varepsilon'_\U|\leq L$ almost surely. 
Using \Cref{lem: pU}, almost surely, we have 
\begin{align}
  p_{\depth,\U}'(\bm{x}) = \dfrac{c'(\bm x)}{2^\depth} \left(1+ W_{\U}({\bf x})\right),    
\end{align}
with, together with \eqref{eq_pprimeU_lemmaB3_1} and  \eqref{eq_pprimeU_lemmaB3}, yield
\begin{align}
    1 + W_{\U}({\bf x}) \geq \frac{1}{c'({\bf x})} \left( \frac{1-p'}{1-p}\right). \label{eq_WU_lower_bounded}
\end{align}
According to \eqref{eq:EsT}, since $\mu'_{\depth, \U}(\bm{x})$ does not depend on $Z_1'$, we have, 
\begin{align*}
V_{1,s}'&=\v(\E[T^s(\bm{x}; \U; \bm Z'_{s}) \ | \ Z'_1])\\
&=\v(\E[\mathbb{E}[T^s(\bm{x}; \U; \bm Z'_{s}) \ | \ Z'_1, \U]|Z'_1])\\
&=\v\Big(\E[ \mu'_{\depth,\U}(\bm{x})]+ \E[\frac{(Y_1'-\mu_{\depth,\U}(\bm x))}{sp'_{\depth,\U}(\bm{x})}\1(X_1' \in L_{\U}(\bm{x}))|Z_1] \\
& \qquad +\E \Big[(Y_1' + \mu'_{\depth,\U}(\bm x))O\Big(\frac{e^{-sp'_{\depth,\U}(\bm{x})}}{sp'_{\depth,\U}(\bm{x})}\Big) \Big]\Big)\\
&= \Big(\dfrac{2^{\depth}}{s c'({\bf x})}\Big)^2\v\Big( \E\Big[\frac{(Y_1'-\mu'_{\depth,\U}(\bm x))}{1+ W_{\U}({\bf x})}\1(X_1' \in L_{\U}(\bm{x}))|Z_1 \Big] \\
& \qquad +\E \Big[(Y_1' + \mu'_{\depth,\U}(\bm x))O\Big(\frac{e^{-sp'_{\depth,\U}(\bm{x})}}{1+ W_{\U}({\bf x})}\Big) \mid Z_1' \Big]\Big).
\end{align*}
Decomposing the variance, noting that the first term is bounded using \eqref{eq_WU_lower_bounded}, via Cauchy-Schwarz inequality, we get
\begin{align*}
V_{1,s}'
&= \Big(\dfrac{2^{\depth}}{s c'({\bf x})}\Big)^2\v\Big( \E\Big[\frac{(Y_1'-\mu'_{\depth,\U}(\bm x))}{1+ W_{\U}({\bf x})}\1(X_1' \in L_{\U}(\bm{x}))|Z_1 \Big] \Big)   \\
& \qquad + \Big(\dfrac{2^{\depth}}{s c'({\bf x})}\Big)^2 O \left( \sqrt{\v\Big[\E \Big[(Y_1' + \mu'_{\depth,\U}(\bm x))O\Big(\frac{e^{-sp'_{\depth,\U}(\bm{x})}}{1+ W_{\U}({\bf x})}\Big) \mid Z_1' \Big] \Big]} \right).
\end{align*}
According to \eqref{eq_negligeability_exp_ps}, 
\begin{align}
O\left(e^{-sp_{\depth,\U}'(\bm{x})} \right) & =O\left(\exp \left[-\left(\frac{1-p'}{2(1-p)}\right) s 2^{- \depth} \right] \right), 
\end{align}
which implies
\begin{align}
V_{1,s}'
&= \Big(\dfrac{2^{\depth}}{s c'({\bf x})}\Big)^2\v\Big( \E\Big[\frac{(Y_1'-\mu'_{\depth,\U}(\bm x))}{1+ W_{\U}({\bf x})}\1(X_1' \in L_{\U}(\bm{x}))|Z_1 \Big] \Big)   \\
& \qquad + \Big(\dfrac{2^{\depth}}{s c'({\bf x})}\Big)^2 O\left(\exp \left[-\left(\frac{1-p'}{2(1-p)}\right) s 2^{- \depth} \right] \right).    
\end{align}
Now, 
\begin{align}
    & \v\Big( \E\Big[\frac{(Y_1'-\mu'_{\depth,\U}(\bm x))}{1+ W_{\U}({\bf x})}\1(X_1' \in L_{\U}(\bm{x}))|Z_1 \Big] \Big) \\
    & = \v\Big( \E\Big[\frac{Y_1'-\mu'(X_1')}{1+ W_{\U}({\bf x})}\1(X_1' \in L_{\U}(\bm{x}))|Z_1 \Big] \\
    & \quad + \E\Big[\frac{\varepsilon'_{\U}(X_1')\textup{Diam}(L_{\U}(\bm{x}))}{1+ W_{\U}({\bf x})}\1(X_1' \in L_{\U}(\bm{x}))|Z_1 \Big]\Big),
\end{align}
with $\varepsilon'_{\U}(X_1')\textup{Diam}(L_{\U}(\bm{x})) = \mu'(X_1') - \mu'_{\depth,\U}(\bm x)$. Decomposing the variance leads to 
\begin{align}
& \v\Big( \E\Big[\frac{(Y_1'-\mu'_{\depth,\U}(\bm x))}{1+ W_{\U}({\bf x})}\1(X_1' \in L_{\U}(\bm{x}))|Z_1' \Big] \Big) \\
    & = \v\Big( \E\Big[\frac{Y_1'-\mu'(X_1')}{1+ W_{\U}({\bf x})}\1(X_1' \in L_{\U}(\bm{x}))|Z_1' \Big] \Big)\\
    & \quad + \v\Big( \E\Big[\frac{\varepsilon'_{\U}(X_1')\textup{Diam}(L_{\U}(\bm{x}))}{1+ W_{\U}({\bf x})}\1(X_1' \in L_{\U}(\bm{x}))|Z_1' \Big]\Big)\\
    & \quad + 2 \cov \Big( \E\Big[\frac{Y_1'-\mu'(X_1')}{1+ W_{\U}({\bf x})}\1(X_1' \in L_{\U}(\bm{x}))|Z_1' \Big], \\
    & \qquad \E\Big[\frac{\varepsilon'_{\U}(X_1')\textup{Diam}(L_{\U}(\bm{x}))}{1+ W_{\U}({\bf x})}\1(X_1' \in L_{\U}(\bm{x}))|Z_1' \Big] \Big),
\end{align}
where the covariance term is null since $\E[Y_1'|X_1']=\mu'(X_1')$. Besides, the variance term is controlled by an application of \Cref{prp:asequiv_bis} ii). Indeed, 
\begin{align*}
    &\v\Big( \E\Big[\frac{\varepsilon'_{\U}(X_1')\textup{Diam}(L_{\U}(\bm{x}))}{1+ W_{\U}({\bf x})}\1(X_1' \in L_{\U}(\bm{x}))|Z_1' \Big]\Big)\\
    & \qquad\leq  \frac{2L^2c'({\bf x}) (1-p)}{ 1-p' } \E \big[(\E[\textup{Diam}(L_{\U}(\bm{x}))\1(X_1'\in L_\U(\bm{x})|X_1'])^2 \big]\\
    &\qquad = O((\alpha_2/2)^k),
\end{align*}
since, under Condition {\bf (H1)}, $|\varepsilon'_\U|\leq L$, and according to \eqref{eq_WU_lower_bounded}. Consequently, 
\begin{align}
    V_{1,s}'
&= \Big(\dfrac{2^{\depth}}{s c'({\bf x})}\Big)^2\v\Big( \E\Big[\frac{Y_1'-\mu'(X_1')}{1+ W_{\U}({\bf x})}\1(X_1' \in L_{\U}(\bm{x}))|Z_1 \Big] \Big)  \\
& \qquad +\Big(\dfrac{2^{\depth}}{s c'({\bf x})}\Big)^2  O((\alpha_2/2)^k) \\
& \qquad + \Big(\dfrac{2^{\depth}}{s c'({\bf x})}\Big)^2 O\left(\exp \left[-\left(\frac{1-p'}{2(1-p)}\right) s 2^{- \depth} \right] \right). 
\end{align}
Regarding the first term, 
\begin{align}
& \v\Big( \E\Big[\frac{Y_1'-\mu'(X_1')}{1+ W_{\U}({\bf x})}\1(X_1' \in L_{\U}(\bm{x}))|Z_1' \Big] \Big) \\
 & = \v\Big( (Y_1'-\mu'(X_1')) \E\Big[\frac{\1(X_1' \in L_{\U}(\bm{x}))}{1+ W_{\U}({\bf x})}|Z_1' \Big] \Big) \\
 & = \v\Big( (Y_1'-\mu'(X_1')) \E\Big[\frac{\1(X_1' \in L_{\U}(\bm{x}))}{1+ W_{\U}({\bf x})}|X_1' \Big] \Big) \\
 & = \E \Big[ \v\Big( (Y_1'-\mu'(X_1')) \E\Big[\frac{\1(X_1' \in L_{\U}(\bm{x}))}{1+ W_{\U}({\bf x})}|X_1' \Big] \mid X_1' \Big) \Big] \\
 & + \v \Big[ \E\Big( (Y_1'-\mu'(X_1')) \E\Big[\frac{\1(X_1' \in L_{\U}(\bm{x}))}{1+ W_{\U}({\bf x})}|X_1' \Big] \mid X_1' \Big) \Big] \\
 & = \E \Big[ \left(  \E\Big[\frac{\1(X_1' \in L_{\U}(\bm{x}))}{1+ W_{\U}({\bf x})}|X_1' \Big]  \right)^2  \mu'(X_1') (1 - \mu'(X_1')  \Big) \Big],
\end{align}
using the law of total variance, and $\E[Y_1'-\mu'(X_1') \mid X_1']= 0$.
Recall that $\mu'$ is $L'$-Lipschitz, according to \Cref{lem_mu_prime_expression} and because $\mu$ is Lipschitz. Then, 
\begin{align}
& \E \Big[ \mu'(X_1') (1 - \mu'(X_1') \left(  \E\Big[\frac{\1(X_1' \in L_{\U}(\bm{x}))}{1+ W_{\U}({\bf x})}|X_1' \Big]  \right)^2     \Big]\\
& = \E \Big[ \E \Big[ \mu'(X_1')(1-\mu'(X_1')) \frac{\1(X_1' \in L_{\U}(\bm{x}))}{1+ W_{\U}({\bf x})}|X_1'\Big] \E \Big[ \frac{\1(X_1' \in L_{\U}(\bm{x}))}{1+ W_{\U}({\bf x})}|X_1'\Big]\Big]\\
& = \E \Big[ \E \Big[ \mu'({\bf x})(1-\mu'({\bf x})) \frac{\1(X_1' \in L_{\U}(\bm{x}))}{1+ W_{\U}({\bf x})}|X_1'\Big] \E \Big[ \frac{\1(X_1' \in L_{\U}(\bm{x}))}{1+ W_{\U}({\bf x})}|X_1'\Big]\Big] \\
& \quad + O\Big( \E \Big[ \E \Big[ \|X_1' - {\bf x}\|_{\infty} \frac{\1(X_1' \in L_{\U}(\bm{x}))}{1+ W_{\U}({\bf x})}|X_1'\Big] \E \Big[ \frac{\1(X_1' \in L_{\U}(\bm{x}))}{1+ W_{\U}({\bf x})}|X_1'\Big]\Big] \Big)\\
& = \mu'({\bf x})(1-\mu'({\bf x})) \E \Big[ \E \Big[  \frac{\1(X_1' \in L_{\U}(\bm{x}))}{1+ W_{\U}({\bf x})} |X_1'\Big]^2 \Big] \\
& \quad + O\Big( \E \Big[ \E \Big[ \textup{Diam}(L_{\U}(\bm{x})) \frac{\1(X_1' \in L_{\U}(\bm{x}))}{1+ W_{\U}({\bf x})}|X_1'\Big] \E \Big[ \frac{\1(X_1' \in L_{\U}(\bm{x}))}{1+ W_{\U}({\bf x})}|X_1'\Big]\Big] \Big).
\end{align}
Finally, applying Cauchy-Schwarz inequality yields
\begin{align}
   V_{1,s}'
&= \Big(\dfrac{2^{\depth}}{s c'({\bf x})}\Big)^2 \mu'({\bf x})(1-\mu'({\bf x})) \E \Big[ \E \Big[  \frac{\1(X_1' \in L_{\U}(\bm{x}))}{1+ W_{\U}({\bf x})} |X_1'\Big]^2 \Big] \nonumber \\
&  +\Big(\dfrac{2^{\depth}}{s c'({\bf x})}\Big)^2  \left( O((\alpha_2/2)^k) + O\left(\exp \left[-\left(\frac{1-p'}{2(1-p)}\right) s 2^{- \depth} \right] \right) \right) \nonumber \\
& + \Big(\dfrac{2^{\depth}}{s c'({\bf x})}\Big)^2 O\Big( \sqrt{\E \Big[ \E \Big[ \textup{Diam}(L_{\U}(\bm{x})) \1(X_1' \in L_{\U}(\bm{x}))|X_1'\Big]^2 \Big] } \Big) \nonumber \\
& \times O \Big(\sqrt{\E\Big[ \E \Big[ \1(X_1' \in L_{\U}(\bm{x}))|X_1'\Big]^2\Big]} \Big), \label{final_equation_th_undersample}
\end{align}
since $1+ W_{\U}({\bf x}) $ is lower bounded. 
According to Proposition \ref{prp:asequiv_bis}, the third term is $o((\sqrt{\alpha_2}/2)^k)$. Thus, according to Proposition \ref{prp:asequiv_bis} again, since $\alpha_2 \leq 1$,
\begin{align}
   V_{1,s}'
&= \Big(\dfrac{2^{\depth}}{s c'({\bf x})}\Big)^2 \mu'({\bf x})(1-\mu'({\bf x})) \E \Big[ \E \Big[  \frac{\1(X_1' \in L_{\U}(\bm{x}))}{1+ W_{\U}({\bf x})} |X_1'\Big]^2 \Big] \nonumber \\
&  +\Big(\dfrac{2^{\depth}}{s c'({\bf x})}\Big)^2  \left( o((\sqrt{\alpha_2}/2)^k) + O\left(\exp \left[-\left(\frac{1-p'}{2(1-p)}\right) s 2^{- \depth} \right] \right) \right).
\end{align}
Note that 
\begin{align}
 &   \E \Big[  \frac{\1(X_1' \in L_{\U}(\bm{x}))}{1+ W_{\U}({\bf x})} |X_1'\Big]^2  \\
 & = \left( \E[\1(X_1' \in L_{\U}(\bm{x})) |X_1'] - \E\left[ \frac{ W_{\U}({\bf x}) \1(X_1' \in L_{\U}(\bm{x}))}{1+ W_{\U}({\bf x})} |X_1' \right] \right)^2   \\
 & \leq  \left( \E[\1(X_1' \in L_{\U}(\bm{x})) |X_1'] + \E\left[ \frac{ |W_{\U}({\bf x})| \1(X_1' \in L_{\U}(\bm{x}))}{1+ W_{\U}({\bf x})} |X_1' \right] \right)^2   \\
 & \leq  \left( \P(X_1' \in L_{\U}(\bm{x}) |X_1') + 2 L \alpha'({\bf x}) c'({\bf x}) \left( \frac{1-p}{1-p'}\right) \E\left[   \textup{Diam}(L_{\U}({\bf x})) \1(X_1' \in L_{\U}(\bm{x}))  |X_1' \right] \right)^2,
\end{align}
since $W_{\U}= \alpha({\bf x}) \varepsilon'_{\U}({\bf x}) \textup{Diam}(L_{\U}({\bf x}))$ and  
\begin{align}
    \frac{1}{c'({\bf x})} \left( \frac{1-p'}{1-p}\right) \leq 1 + W_{\U}({\bf x}) \leq 1 + 2L \alpha'({\bf x}). \label{eq_Wu_inequality3}
\end{align}
Thus, using \Cref{prp:asequiv_bis}, we have
\begin{align}
     & \E \Big[ \E \Big[  \frac{\1(X_1' \in L_{\U}(\bm{x}))}{1+ W_{\U}({\bf x})} |X_1'\Big]^2 \Big] \\
     & \leq 2 \E\left[  \P(X_1' \in L_{\U}(\bm{x}) |X_1') \right]^2 \\
     & \qquad + 4 L \alpha'({\bf x}) c'({\bf x}) \left( \frac{1-p}{1-p'}\right)  \E\left[ \E\left[   \textup{Diam}(L_{\U}({\bf x})) \1(X_1' \in L_{\U}(\bm{x}))  |X_1' \right] \right]^2 \\
     & \leq \dfrac{2 C(d)}{2^{\depth} \depth^{(d-1)/2}}  \dfrac{p'}{p} + 4 L \alpha'({\bf x})c'({\bf x}) \left( \frac{1-p}{1-p'}\right) o\Big(2^{-k} \alpha_2^k\Big). \label{eq_proof_V1primes_replace_bounded_density1}
\end{align}
Besides, using \eqref{eq_Wu_inequality3} and \Cref{prp:asequiv_bis}, we have, for all $\depth$ large enough 
\begin{align}
      \E \Big[ \E \Big[  \frac{\1(X_1' \in L_{\U}(\bm{x}))}{1+ W_{\U}({\bf x})} |X_1'\Big]^2 \Big]      & \geq \frac{1}{1 + 2L\alpha'({\bf x})} \E\left[  \P(X_1' \in L_{\U}(\bm{x}) |X_1') \right]^2 \\
     & \geq \frac{1}{2} \frac{1}{1 + 2L\alpha({\bf x})} \dfrac{1 - p'}{1-p} \dfrac{C(d)}{2^{\depth} \depth^{(d-1)/2}}.
\end{align}
Gathering the two previous inequalities yields
\begin{align}
   & \frac{1}{2} \Big(\dfrac{2^{\depth}}{s c'({\bf x})}\Big)^2 \dfrac{C(d)}{2^{\depth} \depth^{(d-1)/2}} \frac{\mu'({\bf x})(1-\mu'({\bf x}))}{1 + 2L \alpha'({\bf x})} \left( \frac{1-p'}{1-p}\right) \\
    & \leq \Big(\dfrac{2^{\depth}}{s c'({\bf x})}\Big)^2 \mu'({\bf x})(1-\mu'({\bf x})) \E \Big[ \E \Big[  \frac{\1(X_1' \in L_{\U}(\bm{x}))}{1+ W_{\U}({\bf x})} |X_1'\Big]^2 \Big] \\
    & \leq \Big(\dfrac{2^{\depth}}{s c'({\bf x})}\Big)^2 \dfrac{2 C(d)}{2^{\depth} \depth^{(d-1)/2}} \mu'({\bf x})(1-\mu'({\bf x}))  \frac{p'}{p} + \Big(\dfrac{2^{\depth}}{s}\Big)^2o\Big(2^{-k} \alpha_2^k\Big).
\end{align}
Thus, for all $\depth$ large enough, 
\begin{align}
\label{eq_constante_c1_explicit_last_theorem}
    \frac{1}{2} C(d) \left( \frac{1-p'}{1-p}\right) \frac{ \mu'({\bf x})(1-\mu'({\bf x}))}{(1 + 2L \alpha'({\bf x}))c'({\bf x})^2}  
    & \leq  \dfrac{s^2 \depth^{(d-1)/2}}{2^{\depth}} V_{1,s}' \nonumber \\
    & \leq 2 C(d) \frac{p'}{p} \frac{ \mu'({\bf x})(1-\mu'({\bf x}))}{ c'({\bf x})^2},
\end{align}
where $\alpha'({\bf x})$ and $c'({\bf x})$ are defined in \Cref{lem: pU}, concluding the proof.

\end{proof}

The next proposition provides upper and lower bounds of the kernel $\E [\P(X_1\in L_\U(\bm{x})|X_1)^2 ]$.

\begin{prp}[\textbf{Bounds of the kernel}]\label{prp:asequiv_bis}
Let $\bm x \in [0,1]^d$ with $d\ge 2$ and assume Condition \textbf{(H1)}  holds. Then, for all $\depth$ large enough,
\begin{itemize}
   \item [i)]$ 
    \dfrac{1 - p'}{1-p} \dfrac{C(d)}{2^{\depth} \depth^{(d-1)/2}} \leq \E \Big[\P(X_1\in L_\U(\bm{x})|X_1)^2 \Big] \leq \dfrac{C(d)}{2^{\depth} \depth^{(d-1)/2}}  \dfrac{p'}{p};$

\item [ii)]$
    \E \Big[(\E[\textup{Diam}(L_{\U}(\bm{x}))\1(X_1\in L_\U(\bm{x})|X_1])^2 \Big]=o\Big(2^{-k} \alpha_2^k\Big)$.

\end{itemize}

\end{prp}

\begin{proof}[Proof of \Cref{prp:asequiv_bis}]
i) Using the distribution of the splits in centered trees as in \citet{arnould2023interpolation}, we obtain
\begin{align*}
\E \Big[\P(X_1\in L_\U(\bm{x})|X_1)^2 \Big] &= \E \Big[ \Big( \sum_{\substack{k_1, \ldots, k_d \\ \sum_{j=1}^d k_j = k}} \frac{k!}{k_1! \ldots k_d!} \Big( \frac{1}{d} \Big)^k \prod_{j=1}^d \mathds{1}_{\lceil 2^{k_j} x^{(j)} \rceil  = \lceil 2^{k_j} X_1^{(j)} \rceil } \Big)^2 \Big]\\
&= \sum_{\substack{(k_1, \ldots, k_d)\\ (l_1, \ldots, l_d)\\ \sum_{j=1}^d k_j = \sum_{j=1}^d l_j =k}} \frac{k!}{k_1! \ldots k_d!} \frac{l!}{l_1! \ldots l_d!} \Big( \frac{1}{d} \Big)^{2k} \\
&\times  \P( \cap_{j=1}^d (\lceil 2^{k_j} x^{(j)} \rceil  = \lceil 2^{k_j} X_1^{(j)} \rceil  \cap \lceil 2^{l_j} x^{(j)} \rceil  = \lceil 2^{l_j} X_1^{(j)} \rceil) ) .
\end{align*}
Since 
\begin{align}
    & \lceil 2^{k_j} x^{(j)} \rceil  = \lceil 2^{k_j} X_1^{(j)} \rceil  \\
\Longleftrightarrow    &  \lceil 2^{k_j} x^{(j)} \rceil < 2^{k_j} X_1^{(j)} \leq \lceil 2^{k_j} x^{(j)} \rceil+1 \\
\Longleftrightarrow    &  \frac{\lceil 2^{k_j} x^{(j)} \rceil}{2^{k_j}} <  X_1^{(j)} \leq \frac{\lceil 2^{k_j} x^{(j)} \rceil}{2^{k_j}}+2^{-k_j},
\end{align}
it is easy to prove that 
\begin{align}
    & (\lceil 2^{k_j} x^{(j)} \rceil  = \lceil 2^{k_j} X_1^{(j)} \rceil)  \cap (\lceil 2^{l_j} x^{(j)} \rceil  = \lceil 2^{l_j} X_1^{(j)} \rceil)\\
    & = \lceil 2^{\max(k_j,l_j)} x^{(j)} \rceil  = \lceil 2^{\max(k_j,l_j)} X_1^{(j)} \rceil.
\end{align}
Thus, 
\begin{align*}
\E \Big[\P(X_1\in L_\U(\bm{x})|X_1)^2 \Big] 
&= \sum_{\substack{(k_1, \ldots, k_d)\\ (l_1, \ldots, l_d)\\ \sum_{j=1}^d k_j = \sum_{j=1}^d l_j =k}} \frac{k!}{k_1! \ldots k_d!} \frac{l!}{l_1! \ldots l_d!} \Big( \frac{1}{d} \Big)^{2k} \\
&\times  \P( \cap_{j=1}^d (\lceil 2^{\max(k_j,l_j)} x^{(j)} \rceil  = \lceil 2^{\max(k_j,l_j)} X_1^{(j)} \rceil) ) .
\end{align*}
Let us denote by $A_{k,l}({\bf x})$ the cell 
\begin{align}
    A_{k,l}({\bf x}) = \Big\{u \in [0,1]^d: & \textrm{for all } 1\leq j \leq  d, \nonumber \\
    & \lceil 2^{\max(k_j,l_j)} x^{(j)} \rceil  <  2^{\max(k_j,l_j)} u^{(j)} \leq  \lceil 2^{\max(k_j,l_j)} x^{(j)} \rceil + 1 \Big\}.
\end{align}
We have 
\begin{align*}
& \E \Big[\P(X_1\in L_\U(\bm{x})|X_1)^2 \Big] \\
&= \sum_{\substack{(k_1, \ldots, k_d)\\ (l_1, \ldots, l_d)\\ \sum_{j=1}^d k_j = \sum_{j=1}^d l_j =k}} \frac{k!}{k_1! \ldots k_d!} \frac{l!}{l_1! \ldots l_d!} \Big( \frac{1}{d} \Big)^{2k}    \P( X' \in A_{k,l}({\bf x})).
\end{align*}
According to \Cref{lem_density_Xprime}, $X'$ admits a density $f_{X'}$ satisfying
\begin{align} \label{eq_proof_C_bounded_density}
   & \frac{1 - p'}{1-p} f_{X}(\bm{x}) \leq   f_{X'}(\bm{x})  \leq   \frac{p'}{p}    f_X({\bf x}).
\end{align}
Thus, 
\begin{align*}
& \frac{1 - p'}{1-p} \sum_{\substack{(k_1, \ldots, k_d)\\ (l_1, \ldots, l_d)\\ \sum_{j=1}^d k_j = \sum_{j=1}^d l_j =k}} \frac{k!}{k_1! \ldots k_d!} \frac{l!}{l_1! \ldots l_d!} \Big( \frac{1}{d} \Big)^{2k}    \P( X \in A_{k,l}({\bf x})) \\
& \leq \E \Big[\P(X_1\in L_\U(\bm{x})|X_1)^2 \Big] \\
& \leq  \frac{p'}{p}  \sum_{\substack{(k_1, \ldots, k_d)\\ (l_1, \ldots, l_d)\\ \sum_{j=1}^d k_j = \sum_{j=1}^d l_j =k}} \frac{k!}{k_1! \ldots k_d!} \frac{l!}{l_1! \ldots l_d!} \Big( \frac{1}{d} \Big)^{2k}    \P( X \in A_{k,l}({\bf x})).
\end{align*}
According to the first statement of \Cref{prp:asequiv}, one has that for all $\depth$ large enough, 
\begin{align}
    \frac{1 - p'}{1-p} \dfrac{C(d)}{2^{\depth} \depth^{(d-1)/2}} \leq \E \Big[\P(X_1\in L_\U(\bm{x})|X_1)^2 \Big] \leq \dfrac{C(d)}{2^{\depth} \depth^{(d-1)/2}}  \frac{p'}{p},
\end{align}
which concludes the proof of the first assertion.

ii) For the second assertion, we use again the distribution 
\begin{align*}
&\E \Big[(\E[\textup{Diam}(L_{\U}(\bm{x}))\1(X_1\in L_\U(\bm{x}))|X_1])^2 \Big]\\
& = \E \Big[(\E[2^{- \min(K_1({\bm x}), \hdots, K_d({\bm x}))}\1(X_1\in L_\U(\bm{x}))|X_1])^2 \Big]\\
& = \E \Big[(\E[ \max\left( 2^{- K_1({\bm x})}, \hdots, 2^{- K_d({\bm x})} \right) \1(X_1\in L_\U(\bm{x}))|X_1])^2 \Big]\\
& \leq  \E \Big[ \left( \sum_{j=1}^d \E[2^{- K_j({\bm x})} \1(X_1\in L_\U(\bm{x}))|X_1] \right)^2 \Big]\\
& \leq d^2 \E \Big[(\E[2^{- K_1({\bm x})} \1(X_1\in L_\U(\bm{x}))|X_1])^2 \Big]\\
&= d^2 \E \Big[ \Big( \sum_{\substack{k_1, \ldots, k_d \\ \sum_{j=1}^d k_j = k}} \frac{k!}{k_1! \ldots k_d!} \Big( \frac{1}{d} \Big)^k 2^{-k_1}\prod_{j=1}^d \mathds{1}_{\lceil 2^{k_j} x^{(j)} \rceil  = \lceil 2^{k_j} X_1^{(j)} \rceil } \Big)^2 \Big]\\
&= d^2 \sum_{\substack{(k_1, \ldots, k_d)\\ (l_1, \ldots, l_d)\\
\sum_{j=1}^d k_j = \sum_{j=1}^d l_j =k}} \frac{k!}{k_1! \ldots k_d!} \frac{l!}{l_1! \ldots l_d!} \Big( \frac{1}{d} \Big)^{2k}2^{-k_1- l_1} \P(X_1' \in A_{k,l}({\bf x})).
\end{align*}
As in the proof of the first statement, one can control $\P(X_1' \in A_{k,l}({\bf x}))$ via 
\begin{align}
   & \frac{1 - p'}{1-p} \P(X \in A_{k,l}({\bf x})) \leq   \P(X' \in A_{k,l}({\bf x}))  \leq   \frac{p'}{p}    \P(X \in A_{k,l}({\bf x})).
\end{align}
Continuing the proof as in the uniform case leads to 
\begin{align*}
& \E \Big[(\E[\textup{Diam}(L_{\U}(\bm{x}))\1(X_1\in L_\U(\bm{x}))|X_1])^2 \Big]\\
& \leq \frac{p'}{p} d^2 2^{-k}\E\Big[2^{-K_1-L_1-(1/2)\sum_{j=1}^d|K_j-L_j|}\Big] \\
&\le  \frac{p'}{p} d^2 2^{-k}\E\Big[2^{-2(K_1+L_1)}\Big]^{1/2}\E\Big[2^{-\sum_{j=1}^d|K_j-L_j|}\Big]^{1/2}\,,
\end{align*}
by an application of Cauchy-Schwarz inequality. Simple calculations involving binomial distributions show that $\E[2^{-2K_1}]=(1-3/(4d))^k = \alpha_2^k$. Therefore 
\begin{align*}
\E \Big[(\E[\textup{Diam}(L_{\U}(\bm{x}))\1(X_1'\in L_\U(\bm{x}))|X_1'])^2 \Big]
& \leq \frac{p'}{p} d^2 2^{-k} \alpha_2^k
\end{align*}
concluding the proof of the second statement.
\end{proof}

\subsection{Proof of \Cref{cor: TCLB}}\label{proof: corTCLB}

Applying the Delta method to \Cref{th: TCLB} with the function 
\begin{align}
    g: u \mapsto \dfrac{p(1-p')u}{p'(1-p)(1- u)+(1-p')pu},
\end{align}
leads to 
\begin{align}
    \frac{1}{g'(\mu'({\bf x}))} \sqrt{\dfrac {n}{s^2 V_{1,s}'}}( \widehat{\mu}_{\textup{RB},s}^{\textup{ICRF}}(\bm x)-\mu'(\bm x))\overset{d}{\underset{n \to \infty}{\longrightarrow}} \mathcal N(0,1)\,, 
\end{align}
with, for all $\depth$ large enough, 
\begin{align}
    \frac{c_1'({\bf x}) 2^{\depth}}{s^2 \depth^{(d-1)/2}} \leq \frac{V_{1,s}'}{\mu'(\bm x)(1-\mu'(\bm x))} \leq \frac{c_2'({\bf x}) 2^{\depth}}{s^2\depth^{(d-1)/2}},
\end{align}
for some constant $c_1'({\bf x})$ and $c_2'({\bf x})$, which are made explicit in  \Cref{sec: proof_TCLB} (\Cref{prop:v1prime}). Simple calculations show that letting
\begin{align}
    \alpha & = p(1-p')\\
    \beta & = p-p'\\
    \gamma & = p'(1-p),
\end{align}
we have
\begin{align}
    g(u) & = \frac{\alpha u}{\beta u + \gamma}\\
    g'(u) & = \frac{\alpha \gamma }{(\beta u + \gamma)^2}.
\end{align}
Since $g(\mu'({\bf x})) = \mu({\bf x})$, we have 
\begin{align}
    \frac{\alpha \mu'({\bf x})}{\beta \mu'({\bf x}) + \gamma}  = \mu({\bf x}) 
    \Longleftrightarrow &  \frac{1}{\beta \mu'({\bf x}) + \gamma}  = \frac{\mu({\bf x})}{\alpha \mu'({\bf x})}.
\end{align}
Thus, 
\begin{align}
    g'(\mu'({\bf x})) & = \frac{\alpha \gamma }{(\beta \mu'({\bf x}) + \gamma)^2} \\
    & = \alpha \gamma \left(\frac{\mu({\bf x})}{\alpha \mu'({\bf x})}\right)^2\\
    & = \frac{\gamma}{\alpha} \left(\frac{\mu({\bf x})}{\mu'({\bf x})}\right)^2 \\
    & = \frac{p'(1-p)}{p(1-p')} \left(\frac{\mu({\bf x})}{\mu'({\bf x})}\right)^2.
\end{align}
This concludes the proof.

\begin{flushright}
    $\Box$
\end{flushright}

\section{Proof of \Cref{thm_comparison_variances}}\label{sec:proof_th_compare}
The proof of \Cref{thm_comparison_variances} relies solely on extending \Cref{prop:v1prime} for the ICRF and IS-ICRF models to the setting defined by Condition {\bf (H3)}. This extension reduces to adapting \Cref{lem: pU} and \Cref{prp:asequiv_bis} to this context.

Before stating the results, 
note that, under assumption {\bf (H3)}, all densities satisfy, for all $x \in [0,1]^d$,
\begin{align}
    b_1 \leq f_X(x), f_{X'}(x), f_{X''}(x) \leq b_2,
\end{align}
since, by assumption, for all $x \in [0,1]^d$,
\begin{align}
    b_1 \leq  f_{X|Y=1}(x),  f_{X|Y=0}(x) \leq b_2, 
\end{align}
and
\begin{align}
f_{X}(x) & = p f_{X|Y=1}(x) + (1-p) f_{X|Y=0}(x), \\
f_{X'}(x) & = p' f_{X|Y=1}(x) + (1-p') f_{X|Y=0}(x), \\
f_{X''}(x) & = p'' f_{X|Y=1}(x) + (1-p'') f_{X|Y=0}(x).
\end{align}
Besides, since $f_{X|Y=0}$ and $f_{X|Y=1}$ are Lipschitz, the three above densities are Lipschitz. Finally, since $\mu(x) = f_{X|Y=1}(x) p/f_X(x)$ and by assumption $f_X$ is lower bounded and Lispchitz, and $f_{X|Y=1}$ is Lipschitz, the function $\mu$ is Lipschitz. This is also true for $\mu'$ and $\mu''$. We respectively denote by $L, L', L''$ their Lipschitz constants.

The next lemma extends \Cref{lem: pU} providing  a characterisation of $p_{\depth,\U}(\bm{x})$ to this setting.

\begin{lem}[\textbf{Characterisation of $p_{\depth,\U}(\bm{x})$}]\label{lem:pUprimeprime}
Under Assumptions \textbf{(H3)}, $\mu''$ is $L''$ Lipschitz. Thus, we have
\begin{align*}
   p_{\depth,\U}(\bm{x})
   & = \frac{c({\bf x})}{2^{\depth}} \left( 1 + W''_{\U}({\bf x})  \right),
\end{align*}
where $W''_{\U}= \alpha({\bf x}) \varepsilon''_{\U}({\bf x}) \textup{Diam}(L_{\U}({\bf x}))$ and
\begin{align}
  c({\bf x}) & =   \frac{1-p}{1-p''} +  \frac{p -p''}{1-p''} \frac{\mu''({\bm x})}{p''},\\
  \alpha({\bf x}) & = \frac{p-p''}{p''(1-p'') c({\bf x})},
\end{align}
and $\varepsilon''_{\U}({\bf x})$ is a bounded random variable satisfying $|\varepsilon''_{\U}({\bf x})| \leq 2L''$ almost surely. Morever, 
\begin{align}
    \frac{b_1}{c({\bf x})}  \leq 1 + W_{\U}''({\bf x}) \leq 1 + 2L'' |\alpha({\bf x})|.
\end{align}
\end{lem}

\begin{proof}[Proof of \Cref{lem:pUprimeprime}.] By assumption, we have
\begin{align}
p_{\depth,\U}''(\bm{x})&= p''\P(X \in L_{\U}(\bm{x})\mid \U, Y=1)+(1-p'')\P(X  \in L_{\U}(\bm{x})\mid \U, Y=0)
\end{align}
and 
\begin{align}
    p_{\depth, \U}(\bm x)= p \P(X\in L_{\U}(\bm{x})\mid \U, Y=1)+(1-p)\P(X \in L_{\U}(\bm{x})\mid \U, Y=0)\,.
\end{align}
Combining the two previous equalities, and noticing that $p_{\depth, \U}''(\bm x)= 2^{-\depth}$ since  $X''$ is uniformly distributed on $[0,1]^d$, 
\begin{align}
    p_{\depth,\U}(\bm{x}) & = \left(\frac{1-p}{1-p''} \right) 2^{-\depth} + \left( p - \frac{(1-p)p''}{(1-p'')} \right) \P(X \in L_{\U}(\bm{x})\mid \U, Y=1) \\
    & = \left(\frac{1-p}{1-p''} \right) 2^{-\depth} + \left( \frac{p(1-p'') - (1-p)p''}{(1-p'')} \right) \P(X \in L_{\U}(\bm{x})\mid \U, Y=1) \\
    & = \left(\frac{1-p}{1-p''} \right) 2^{-\depth} + \left( \frac{p -p''}{1-p''} \right) \P(X \in L_{\U}(\bm{x})\mid \U, Y=1).
\end{align}
As proved in \Cref{lem: pU},  
\begin{align*}
 & \P(X \in  L_{\U}(\bm{x})\mid \U, Y=1)\\
 & =\P(X'' \in  L_{\U}(\bm{x})\mid \U, Y''=1)\\
   & =  \dfrac {\mu''(\bm x)}p'' \P(X'' \in L_{\U}(\bm{x}) \mid \U ) + \frac{1}{p''} \textup{Diam}(L_{\U}(\bm{x})) \P(X'' \in L_{\U}(\bm{x}) \mid \U) \varepsilon'_{\U}({\bf x})\\
   &    =  \frac{ 2^{-\depth}}{p''} (\mu''({\bf x}) + \varepsilon''_{\U}({\bf x}) \textup{Diam}(L_{\U}(\bm{x})))
\end{align*}
where 
$$
\varepsilon''_{\U}(\bm{x}) = \frac{ \mu''_{\depth, \U}(\bm{x}) - \mu''(\bm x)}{\textup{Diam}(L_{\U}(\bm{x}))},
$$
which verifies, under Condition {\bf (H1)}, $|\varepsilon''_\U({\bf x})|\leq L''$ almost surely. 
Combining those results we get
\begin{align*}
   p_{\depth,\U}(\bm{x})
    & = \left(\frac{1-p}{1-p''} \right) 2^{-\depth} + \left( \frac{p -p''}{1-p''} \right) \P(X \in L_{\U}(\bm{x})\mid \U, Y=1) \\
   & = \left( \frac{1-p}{1-p''} +  \frac{p -p''}{1-p''} \frac{\mu''({\bm x})}{p''} + \frac{p -p''}{1-p''} \frac{\varepsilon''_{\U}({\bf x}) \textup{Diam}(L_{\U}(\bm{x}))}{p''}  \right)2^{-\depth}\\
   & = \frac{c({\bf x})}{2^{\depth}} \left( 1 + \alpha({\bf x}) \varepsilon''_{\U}({\bf x}) \textup{Diam}(L_{\U}(\bm{x})) \right),
\end{align*}
with 
\begin{align}
  c({\bf x}) & =   \frac{1-p}{1-p''} +  \frac{p -p''}{1-p''} \frac{\mu''({\bm x})}{p''},
\end{align}
and
\begin{align}
\alpha({\bf x}) = \frac{p-p''}{p''(1-p'') c({\bf x})}.
\end{align}
Letting 
\begin{align*}
   p_{\depth,\U}(\bm{x})
   & = \frac{c({\bf x})}{2^{\depth}} \left( 1 + W''_{\U}({\bf x})  \right),
\end{align*}
where $W_{\U}''{\bf x}= \alpha({\bf x}) \varepsilon''_{\U}({\bf x}) \textup{Diam}(L_{\U}({\bf x}))$, it holds that $p_{\depth,\U}(\bm{x}) \geq b_1 2^{-\depth}$. This concludes the proof since it  implies 
\begin{align}
    \frac{b_1}{c({\bf x})}  \leq 1 + W''_{\U}({\bf x}) \leq 1 + 2L''| \alpha({\bf x})|.
\end{align}

\end{proof}
 The next proposition provides upper and lower bounds of the kernel $\E [\P(X_1\in L_\U(\bm{x})|X_1)^2 ]$  extending  \Cref{prp:asequiv_bis} to this setting.

\begin{prp}[\textbf{Bounds of the kernel}]\label{prp:asequiv_ter}
Let $\bm x \in [0,1]^d$ with $d\ge 2$ and assume Condition \textbf{(H3)}  holds. Then, for all $\depth$ large enough, 
\begin{itemize}
   \item [i)]$ 
     \dfrac{C(d)b_1}{2^{\depth} \depth^{(d-1)/2}} \leq \E \Big[\P(X_1\in L_\U(\bm{x})|X_1)^2 \Big] \leq \dfrac{C(d)b_2}{2^{\depth} \depth^{(d-1)/2}},$

\item [ii)]$
    \E \Big[(\E[\textup{Diam}(L_{\U}(\bm{x}))\1(X_1\in L_\U(\bm{x})|X_1])^2 \Big]=o\Big(2^{-k} \alpha_2^k\Big)$.
\end{itemize}
The two statements hold also replacing $X_1$ by $X_1'$. 
\end{prp}

\begin{proof}[Proof of \Cref{prp:asequiv_ter}]
The proof is a straightforward extension of \Cref{prp:asequiv_bis} noticing that, for all $x \in [0,1]^d$,
\begin{align}
    b_1 \leq f_X(x), f_{X'}(x), f_{X''}(x) \leq b_2.
\end{align}
\end{proof}

We are now ready to prove \Cref{thm_comparison_variances}. 

\begin{proof}[Proof of \Cref{thm_comparison_variances}]

Using \Cref{lem:pUprimeprime} and \Cref{prp:asequiv_ter}, we can adapt the proof of \Cref{prop:v1prime} to show that the asymptotic variance of ICRF, denoted by $V_{1,s}$ satisfies 
\begin{align}
   &   C(d)b_1 \frac{ \mu({\bf x})(1-\mu({\bf x}))}{(1 + 2L |\alpha({\bf x})|)c({\bf x})^2} \leq  \dfrac{s^2 \depth^{(d-1)/2}}{2^{\depth}} V_{1,s}.
\end{align}
Similarly, one can adapt the proof of \Cref{prop:v1prime} to show that the asymptotic variance of IS-ICRF, denoted by $V'_{1,s} g'(\mu'({\bf x}))^2$, satisfies 
\begin{align}
  \dfrac{s^2 \depth^{(d-1)/2}}{2^{\depth}} V_{1,s}'  \leq 2 C(d) b_2 \frac{ \mu'({\bf x})(1-\mu'({\bf x}))}{ c'({\bf x})^2}.
\end{align}
We want to prove that $V_{1,s}' g'(\mu'({\bf x}))^2 < V_{1,s}$. We have
\begin{align}
    & \dfrac{V_{1,s}'}{V_{1,s}} g'(\mu'({\bf x}))^2 \\
    & = \dfrac{V_{1,s}'\dfrac{s^2 \depth^{(d-1)/2}}{2^{\depth}}}{V_{1,s}\dfrac{s^2 \depth^{(d-1)/2}}{2^{\depth}}} g'(\mu'({\bf x}))^2 \\
    & \leq \frac{2 C(d) b_2 \frac{ \mu'({\bf x})(1-\mu'({\bf x}))}{ c'({\bf x})^2}}{ C(d) b_1 \frac{ \mu({\bf x})(1-\mu({\bf x}))}{(1 + 2L |\alpha({\bf x})|)c({\bf x})^2}} \left(\frac{p'(1-p)}{p(1-p')}\right)^2 \left(\frac{\mu({\bf x})}{\mu'({\bf x})}\right)^4\\
    & \leq \frac{2b_2}{b_1}  \frac{ \mu'({\bf x})(1-\mu'({\bf x}))}{\mu({\bf x})(1-\mu({\bf x}))}   \frac{ c({\bf x})^2}{ c'({\bf x})^2} (1 + 2L |\alpha({\bf x})|) \left(\frac{p'(1-p)}{p(1-p')}\right)^2 \left(\frac{\mu({\bf x})}{\mu'({\bf x})}\right)^4 \\
    & \leq \frac{2b_2}{b_1} \frac{ (1-\mu'({\bf x}))}{\mu'({\bf x})^3(1-\mu({\bf x}))}   \frac{ c({\bf x})^2}{ c'({\bf x})^2} (1 + 2L |\alpha({\bf x})|) \left(\frac{p'(1-p)}{1-p'}\right)^2   \frac{\mu({\bm x})^3}{p^2}.
\end{align}
We have
\begin{align}
    \mu(\bm x) & = \frac{f_{X|Y=1}(\bm x) p}{f_X(\bm x)}  = \delta(\bm x) p,
\end{align}
with $\delta(\bm x) = f_{X|Y=1}(\bm x) p/f_X(\bm x)$. Thus, 
\begin{align}
    & \dfrac{V_{1,s}'}{V_{1,s}} g'(\mu'({\bf x}))^2 \\
    & \leq \frac{2b_2}{b_1} \frac{ (1-\mu'({\bf x}))}{\mu'({\bf x})^3(1-\mu({\bf x}))}   \frac{ c({\bf x})^2}{ c'({\bf x})^2} (1 + 2L |\alpha({\bf x})|) \left(\frac{p'(1-p)}{1-p'}\right)^2   \delta(\bm x)^3 p\\
    & = O(p),
\end{align}
which tends to zero, when $p$ tends to zero. 
\end{proof}

\section{Additional experiments}
\label{app_sec_additional_experiments}

In this additional numerical section, we compare $\widehat{\mu}_{s}^{\textup{ICRF}}(\bm x)$, $\widehat{\mu}_{\textup{RB},s}^{\textup{ICRF}}(\bm x)$ and $\widehat{\mu}_{\textup{IS},s}^{\textup{ICRF}}(\bm x)$  with the same setting as in \Cref{fig:hist_small_sample} but for a larger sample size $n=10.000$ and a smaller number of Monte Carlo repetitions $B=100$.

Similarly to  \Cref{fig:hist_small_sample}, we first observe that the average of the predictions (blue dashed line) is close to the theoretical value (red line), which is either $\mu({\bf x})$ for the first and third plot, or $\mu'({\bf x})$ for the second one (and again this holds even more so for the Breiman's forests trained on the original dataset (RF)).

Once more the prediction of Breiman's forest may be far from the correct value (predictions up to $0.35$ compared to the true value $0.17$) and the predictions of the IS-RF (the debiased forest) are closer to the true value. This is quantified by a smaller variance $0.039$ compared to the variance of the prediction of Breiman's forests $0.074$. This further supports, from an experimental perspective, the variance reduction established in \Cref{thm_comparison_variances} for small $p$.  

\begin{figure}[H]
    { \centering                                       
     \begin{subfigure}[b]{0.45\textwidth}
         \centering
         \includegraphics[width=1.3\textwidth]{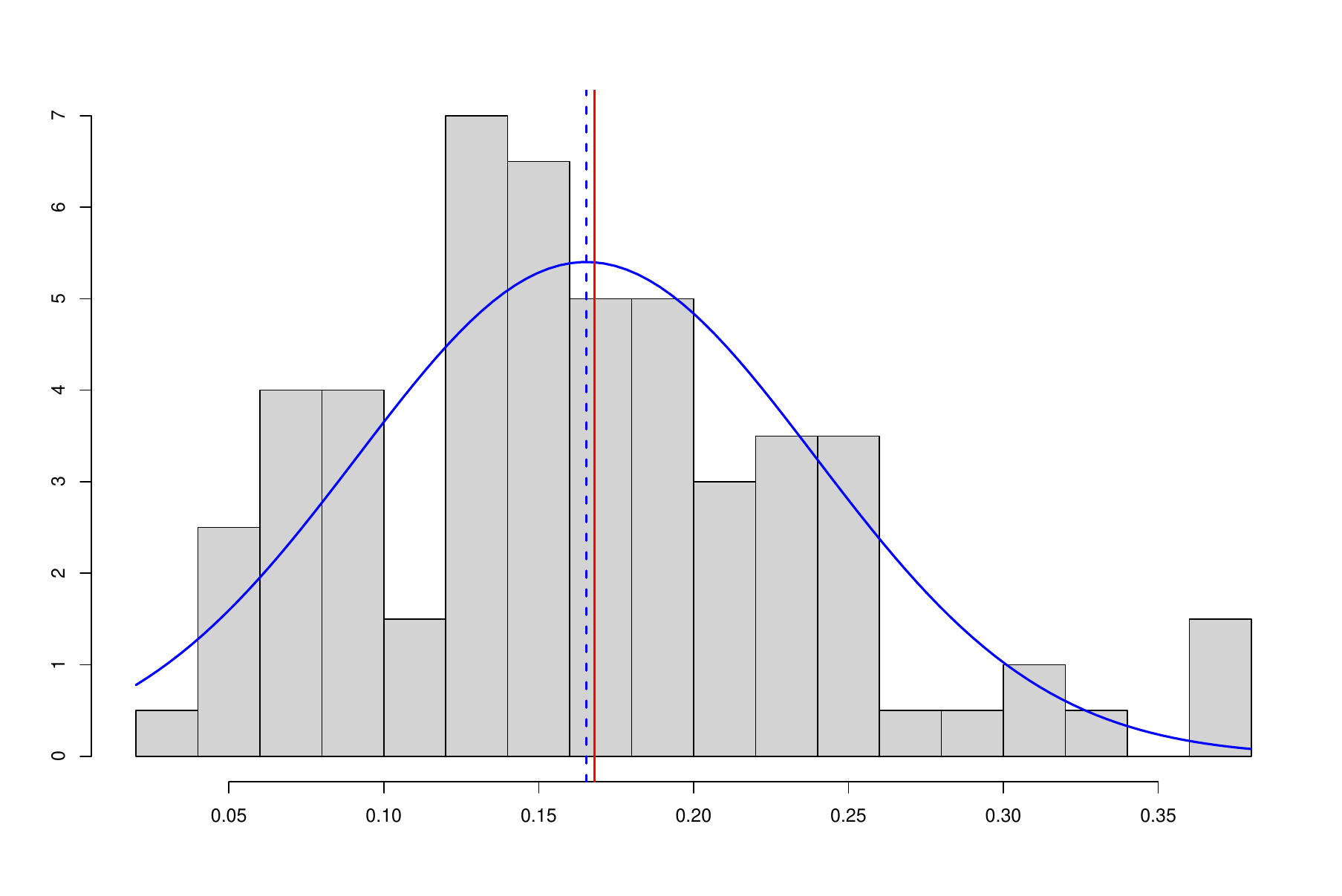}
         \caption{RF}
         \label{hist_large_sample_RF}
     \end{subfigure}
     \hfill
     \begin{subfigure}[b]{0.45\textwidth}
         \centering
         \includegraphics[width=1.3\textwidth]{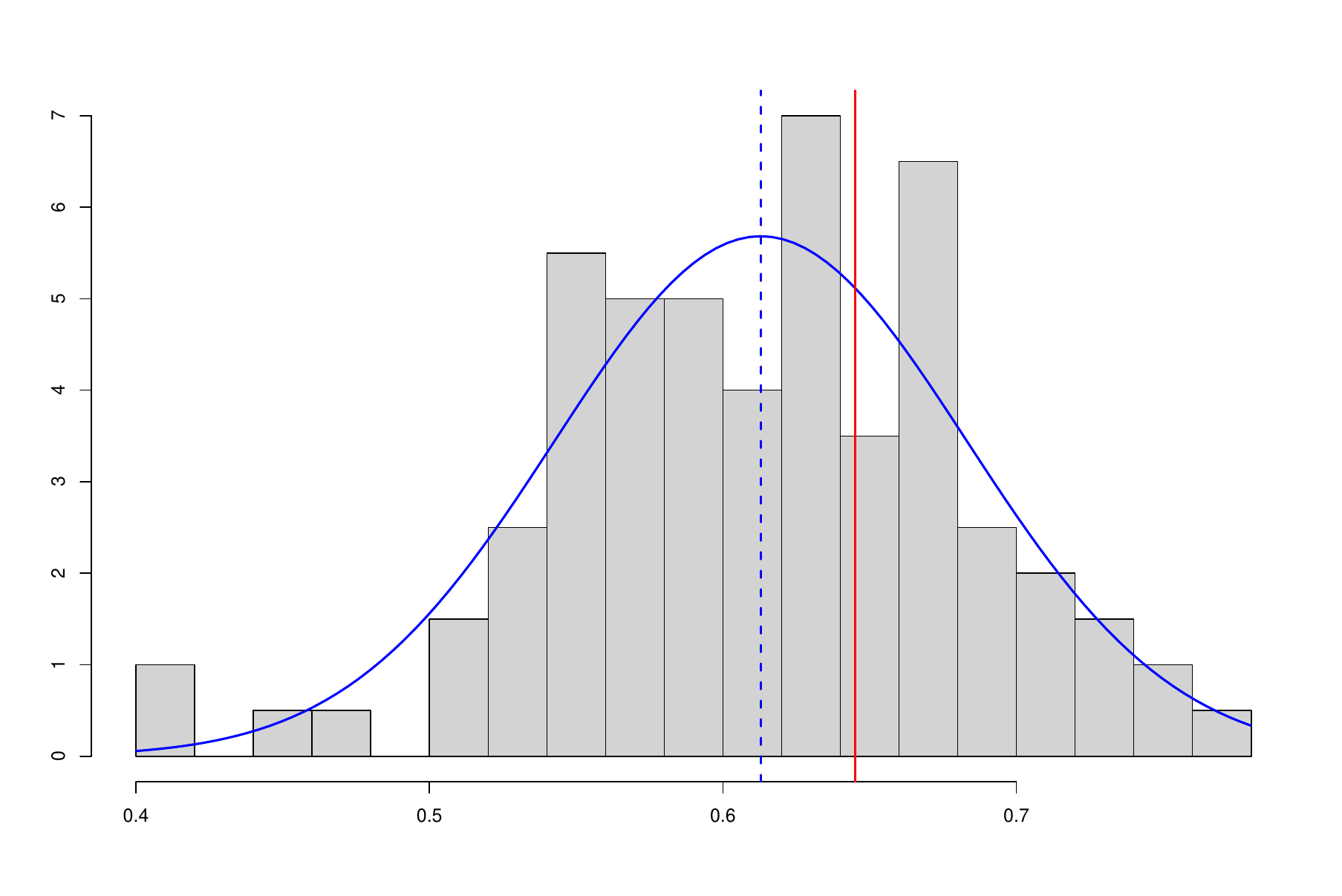}
         \caption{RB-RF}
         \label{hist_large_sample_RBRF}
     \end{subfigure}
      \hfill
     \begin{subfigure}[b]{0.45\textwidth}
         \centering
        \includegraphics[width=1.3\textwidth]{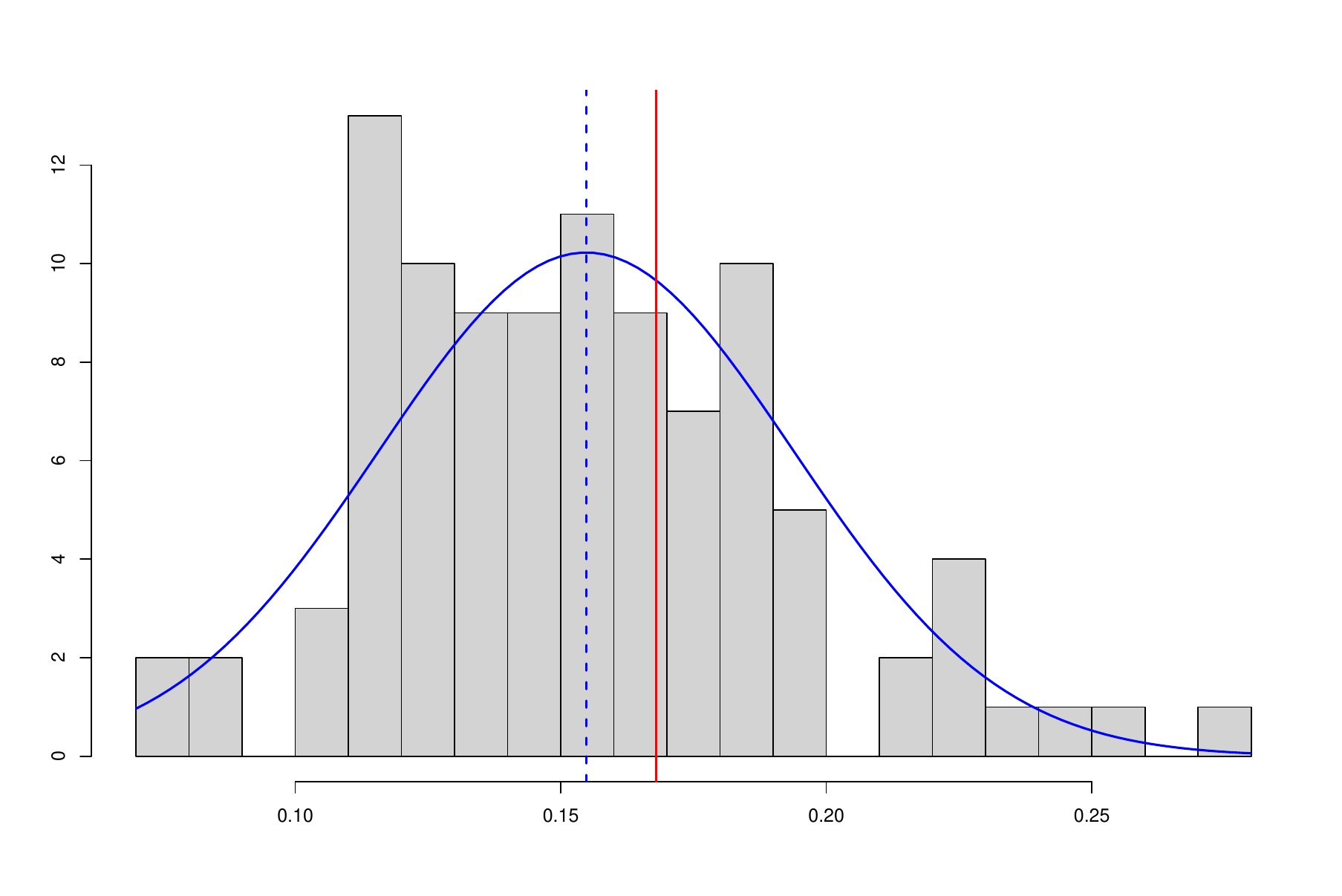}
        \caption{IS-RF}
         \label{hist_large_sample_ISRF}
    \end{subfigure}
        \caption{Histograms of predictions from different estimators in scenario \textbf{(ImB-Sc)} with $\alpha_1 = 0.9$, $\alpha_2 = 0.7$, $p'=0.5$, $n=10.000$ and $B=100$ replicates. The blue curve is that of a Gaussian density whose mean (blue dashed line) and variance are estimated from the data. The red line corresponds to the centering term provided by our theory ($\mu(x)$ for the first and third graph, $\mu'(x)$ for the second one). From left to right, empirical variances are: $0.074$, $0.070$, $0.039$.}
        \label{fig:hist_large_sample}}
\end{figure}

\bibliographystyle{plainnat}
\bibliography{mybiblio}	
\end{document}